\documentclass[12pt]{article}
\usepackage[margin=1in]{geometry}

\usepackage{makecell}

\usepackage{microtype}
\usepackage{graphicx}
\usepackage{subfigure}
\usepackage{booktabs} 
\usepackage[table]{xcolor}

\usepackage[draft]{hyperref}

\usepackage{upgreek,morenotations,rotating}

\usepackage[utf8]{inputenc} 
\usepackage[T1]{fontenc}    
\usepackage{url}            
\usepackage{amsfonts}       
\usepackage{nicefrac}       
\usepackage{microtype}      

\usepackage{bm} 
\usepackage{color}

\usepackage[singlelinecheck=false]{caption}

\usepackage{tcolorbox}

\usepackage[toc,page]{appendix}

\usepackage{tikz}
\usetikzlibrary{arrows,decorations.pathmorphing,backgrounds,positioning,fit,petri}

\usepackage{cleveref}

\title{Being Properly Improper}

\author{
Tyler Sypherd$^{\ddagger}$ \qquad Richard Nock$^{\dagger}$ \qquad Lalitha Sankar$^{\ddagger}$ \\\\
{\small $^\ddagger$School of Electrical, Computer, and Energy Engineering}\\{\small 
Arizona State University, Tempe, AZ 85281}\\{\small
\texttt{$\{$tsypherd,lsankar$\}$@asu.edu}}
\\
\\ {\small $^\dagger$Google Research} \\
{\small \texttt{richardnock@google.com}}
}
\date{}

\begin{document}
\thispagestyle{empty}
\maketitle

\begin{abstract}

Properness for supervised losses stipulates that the loss function shapes the learning algorithm towards the true posterior of the data generating distribution. 
Unfortunately, data in modern machine learning can be corrupted or \emph{twisted} in many ways. 
Hence, optimizing a proper loss function on twisted data could perilously lead the learning algorithm towards the twisted posterior, rather than to the desired clean posterior.
Many papers cope with specific twists (e.g., label/feature/adversarial noise), but there is a growing need for a unified and actionable understanding atop properness.
Our chief theoretical contribution is a generalization of the properness framework with a notion called \textit{twist-properness}, which delineates loss functions with the ability to ``untwist'' the twisted posterior into the clean posterior.
Notably, we show that a nontrivial extension of a loss function called $\alpha$-loss, which was first introduced in information theory, is twist-proper.
We study the twist-proper $\alpha$-loss under a novel boosting algorithm, called \pmboost, and provide formal and experimental results for this algorithm.
Our overarching practical conclusion is that the twist-proper $\alpha$-loss outperforms the proper $\log$-loss on several variants of twisted data.

\end{abstract}
\section{Introduction}

The loss function is a cornerstone of machine learning (ML).
The founding theory of properness for supervised losses stipulates that the loss function shapes the learning algorithm towards the true posterior~\cite{rwID}. 
Consequently, a model trained with a proper loss function will try to closely approximate the Bayes rule of the data generating distribution. 
%
Historically, properness draws its roots from classical work in normative economics for \textit{class probability estimation} (\cpe)~\cite{rwID,sEO,samAP}; some of the most famous losses in supervised learning are proper, e.g., log, square, Matusita, to name a few.
Unfortunately in many modern applications, data can be corrupted or \emph{twisted} in various ways (see Section~\ref{sec-related}); examples of twists include label noise, adversarial noise, and feature noise.
Thus, optimizing a proper loss function on twisted data could perilously lead the learning algorithm towards the Bayes rule of the twisted posterior, rather than to the desired clean posterior.
To ensure that a model trained with a proper loss function on twisted data properly generalizes to the clean distribution, a generalization of properness is clearly required. 
%
%

To this end, we propose the notion of \textit{twist-properness}.
In words, a loss function is twist-proper if and only if (iff), for any twist, there exist hyperparameter(s) of the loss which allow its minimizer to ``untwist'' the twisted posterior into the clean posterior.
Such a generalization would be less impactful without a solid contender loss, and we show that a nontrivial extension of $\alpha$-loss, which itself is an information-theoretic hyperparameterization of the $\log$-loss~\cite{aIT,lkspAT,sdskAT}, \textit{is} twist-proper and exhibits desirable properties for local and global (namely, fixed hyperparameter) twist corrections. 
Furthermore twist-properness is not vacuous, as we provide a counterexample that another (popular) generalization of the $\log$-loss, the focal loss~\cite{lgghdFL}, which was originally designed to solve specific twists, i.e., class imbalance, is \textit{not} twist-proper. 
In addition, we provide a proof that a loss which acts as a general ``wrapper'' of a loss, the Super Loss~\cite{cwrSA}, is also \textit{not} twist proper. 
One of our key takeaways is that twist-properness necessitates a certain nontrivial \textit{symmetry} of the loss, rather than merely a trivial extension of the hyperparameter(s).
%
%
%
%

Recently, $\alpha$-loss was practically implemented in logistic regression and in deep neural networks~\cite{sypherd2019journal}.
In both settings, it was shown to be more robust to symmetric label noise for fixed $\alpha > 1$ than the proper $\log$-loss ($\alpha = 1$), thereby providing a hint at the twist-properness of $\alpha$-loss.
In order to complement our theory of twist-properness and these recent results regarding the robustness of $\alpha$-loss, we also practically implement $\alpha$-loss in boosting. 
Boosting is imbued with the computational constraint that strong learning happens from ``weak updates'' in polynomial time, thus inducing substantial convergence rates~\cite{kvAI}.
Furthermore, boosting algorithms are known to suffer under label noise, particularly for convex losses~\cite{long2010random}.
Thus, boosting presents as an ideal choice to further practically investigate the twist-properness of $\alpha$-loss. 

In order to implement $\alpha$-loss in boosting, a popular route is to invert the canonical link of the loss which computes the weighting of the examples~\cite{fGFA,nnOT}. 
While this is feasible for the $\log$-loss (one gets the popular sigmoid function), it turns out to be nontrivial for $\alpha$-loss. 
We address this issue by providing the first (to the best of our knowledge) general boosting scheme (called \pmboost) for any loss which requires only an approximation of the inverse canonical link, depending on a parameter $\zeta \in [0,1]$ (the closer to $0$, the better the approximation), and gives boosting-compliant convergence, further meeting the general optimum number of calls to the weak learner. 
The cost of this approximation is only a factor $O(1/(1-\zeta)^2)$ in number of iterations. 
Since the implementation of boosting algorithms faces machine types approximations to store the weights, a downstream consequence of 
our analysis is that it also articulates how such practical implementations (using link approximations) can approach the idealized convergence rate~\cite{nwLO}. 
\section{Related work}\label{sec-related}

Studying data corruption in ML dates back to the 80s \cite[Section 4]{vLD}. Remarkably, the first twist models assumed very strong corruption, possibly coming from an adversary with unbounded computational resources, \textit{but} the data at hand was binary. 
Thus, because the feature space was as ``complex'' as the class space, the twist models lacked the unparalelled data complexity that we now face. 
Obtaining such twist models at scale with real world data has been a major problem in ML over the past decade for a number of reasons, not all of which are borne out of bad intent. 
Nevertheless, there have been several streams of recent research aimed at addressing specific twists. 

\textit{Label noise} is a twist which has recently drawn much attention and garnered many corrective attempts~\cite{prmnqMD,zsGC,znsLN,natarajan2013learning,long2010random,sypherd2019journal}.
Notably,~\cite{natarajan2013learning} theoretically study the presence of class conditional noise in binary classification. 
Their approach consists of augmenting proper loss functions with re-weighting coefficients, which is strictly dependent on the class conditional noise percentages, and hence requires knowledge of the noise proportions. 
As a byproduct of their analysis, they show that biased SVM and weighted logistic regression are provably noise-tolerant. 

Setting label noise aside, there exists a zoo of other twists and corrective attempts.
For instance, \textit{data augmentation} techniques, with vicinal risk minimization standing as a pioneer~\cite{cwbvVR}, seek to induce general robustness~\cite{zcdlMB}. 
In deep learning, \textit{adversarial robustness} attempts to address the brittleness of neural networks to targeted adversarial noise~\cite{szsbegfIP,mmstvTD,ahPR}. 
\textit{Data poisoning} twists in computer vision can be very sophisticated and require further investigation~\cite{tjhapjntSE}. 
\textit{Invariant risk minimization} aims at finding data representations yielding good classifiers but also invariant to ``environment changes'' ~\cite{abglIR}; relatedly, \textit{covariate shift} seeks to address changes between train and test, stemming from non-stationarity or bias in the data~\cite{zylsAO}. 
A recent trend has also emerged with correcting losses due to model \textit{confidence} issues~\cite{gpswOC,mksgtdCD,cwrSA}. 

Viewed more broadly, the abovementioned papers 
study arguably much different problems, but they tend to have a theme that goes substantially deeper than the superficial observation that they assume twisted data in some way: the core loss function  
is usually a \textit{proper} loss. 
Therefore, they tend to start from the premise of a loss that inevitably fits the (unwanted) twist, and correct it mostly with a regularizer informed with some prior knowledge of the twist, on a ``twist-by-twist'' basis. 
There has been some positive work in this ``loss + regularizer'' direction~\cite{awakRB,znsLN}, but we note that this does not fully address the underlying issue of properness perilously shaping the learning algorithm towards the twisted posterior. 

Lastly, our generalization of properness with twist-properness is partly inspired by recent work by~\cite{cvcsOF},
where they theoretically investigate the focal loss. 
Notably, they show that the focal loss is classification-calibrated, but not strictly proper. 
From their work, we also gather the implicit notion that hyperparameterized losses that generalize proper losses (e.g., focal loss or $\alpha$-loss generalizing $\log$-loss), which may represent a next step for loss functions in ML, need to be carefully understood from the standpoint of what their hyperparameterization trades-off from properness.

\section{Losses for class-probability estimation}

Our setting is that of losses for class probability estimation (\cpe) and our notations follow~\cite{rwCB,rwID}. 
Given a domain of observations $\mathcal{X}$, we wish to learn a classifier $h:\mathrm{dom}(h) = \mathcal{X}$ that predicts the label $\Y\in \mathcal{Y} \defeq \{-1,1\}$ (without loss of generality, we assume two classes or labels) associated with every instance of data drawn from $\mathcal{X}$. 
Traditionally, there are two kinds of outputs sought: one requires 
$\mathrm{Im}(h) = [0,1]$, in which case $h$ provides an estimate of $\pr[\Y =1 | \X]$, which is often called the Bayes posterior.
This is the framework of class probability estimation. The other kind of output requires $\mathrm{Im}(h) = \mathbb{R}$, but is usually completed by a mapping to $[0,1]$, e.g., via the softmax in deep learning. 
A loss for class probability estimation, $\loss : \mathcal{Y} \times [0,1]
\rightarrow \mathbb{R}$, has the general definition 
\begin{align} \label{eqpartialloss} 
\loss(y,u) & \defeq  \iver{y=1}\cdot \partialloss{1}(u) +
                     \iver{y=-1}\cdot \partialloss{-1}(u),
\end{align}
where $\iver{\cdot}$ is Iverson's bracket \cite{kTN}. Functions $\partialloss{1}, \partialloss{-1}$ are called \textit{partial} losses, minimally assumed to satisfy $\mathrm{dom}(\partialloss{1}) = \mathrm{dom}(\partialloss{-1}) = [0,1]$ and $|\partialloss{1}(u)| \ll \infty, |\partialloss{-1}(u)| \ll \infty, \forall u \in (0,1)$ to be useful for ML.
 Key additional properties of partial losses are:\\
    \noindent \textbf{(M)} Monotonicity: $\partialloss{1}, \partialloss{-1}$ are non-(increa/decrea)sing;\\
    \noindent \textbf{(D)} Differentiability: $\partialloss{1}$ and $\partialloss{-1}$ are  differentiable;\\
    \noindent \textbf{(S)} Symmetry: $\partialloss{1}(u) = \partialloss{-1}(1-u), \forall u\in [0,1]$.\\
  Commonly used proper losses such as log, square and Matusita all satisfy the above three assumptions. The pointwise conditional risk of local guess $u \in [0,1]$ with respect to a \textit{ground truth} $v \in [0,1]$ is  
\begin{eqnarray}
  \poirisk(u,v) & \defeq & \expect_{\Y \sim \bernoulli(v)}\left[\loss(\Y,u)\right] = v\cdot \partialloss{1}(u) + (1-v)\cdot \partialloss{-1}(u) \label{eqpoirisk},
\end{eqnarray}
where $\bernoulli(v)$ defines a Bernoulli distribution with $v$.\\
\noindent \textbf{Properness} $\poirisk(u,v)$ is the fundamental quantity that allows to distinguish proper losses: a loss is \textit{proper} iff for any ground truth $v \in [0,1]$, $\poirisk(v,v) = \inf_u \poirisk(u,v)$, and strictly proper iff $u=v$ is the sole minimiser~\cite{rwID}. The (pointwise) \textit{Bayes} risk is $\poibayesrisk(v) \defeq \inf_u \poirisk(u,v)$.\\
\noindent \textbf{Surrogate loss} Oftentimes, minimization occurs over the reals (e.g., boosting), hence it is useful to employ a surrogate to the $0$-$1$ loss~\cite{bjmCCj}.
~\cite{nnOT} showed that the outputs in $[0,1]$ and $\mathbb{R}$ can be related via convex duality of the losses. 
Let $g^{\star}(z) \defeq \sup_t \{ zt - g(t)\}$ denote the convex conjugate of $g$~\cite{bvCO}. The \textit{surrogate} $\surloss(\cdot)$ of $\poibayesrisk$ is thus given by
\begin{eqnarray}
\surloss(z) & \defeq & (-\poibayesrisk)^{\star}(-z), \forall z\in \mathbb{R}. \label{defsur}
\end{eqnarray}
For example, picking the log-loss as $\ell$ gives the binary entropy for $\poibayesrisk$ and the logistic loss for $F$.
Convex duality implies that predictions in $[0,1]$ and $\mathbb{R}$ are related via {the (canonical) \textit{link} of the loss, $(-\poibayesrisk)'$}~\cite{nwLO} where we use the notation $f'$ to denote the derivative of a function $f$ with respect to its argument.
In the sequel, we will see that boosting requires inverting the link of the loss, which we show is nontrivial for hyperparameterized losses, such as $\alpha$-loss.
Lastly, we provide summary properties of a \cpe~loss (not necessarily proper) and its surrogate, monotonicity being of primary importance. Some parts of the following Lemma are known in the literature (e.g., concavity in~\cite[Lemma 1]{aSR}), or are folklore. 
\begin{lemma}\label{lemUstar}
    $\forall \loss$ \cpe~loss, $\poibayesrisk$ is concave and continuous; $F$ is convex, continuous and non-increasing.
\end{lemma}
A proof is provided for completeness in Appendix~\ref{app-proof-lemUstar}.
 

    \section{Twist-proper losses}\label{sec-imp}

With the classical setting of properness in hand, we now provide fundamental definitions of \textit{twists} and \textit{twist-properness}, and study the twist-properness of several hyperparameterized loss functions.
When it comes to correcting (or untwisting) twists, one needs a loss with the property that its minimizer in~\eqref{eqpoirisk} is different from the now twisted value $\tilde{v}$ \textit{and} recovers the ``hidden'' ground truth $v$. \\
\noindent \textbf{Bayes tilted estimates} We first characterize the minimizers of \eqref{eqpoirisk} when the \cpe~loss is not necessarily proper. 
We define the set-valued (pointwise) Bayes \textit{tilted estimate} $\pseudob$ as
\negspace
\begin{eqnarray} 
\pseudob(\tilde{v}) & \defeq & \arg\inf\limits_{u\in[0,1]} L(u,\tilde{v}). \label{defBTE}
\end{eqnarray}
\negspace
Ideally, we would like for $v \in \pseudob(\tilde{v})$, i.e., the Bayes tilted estimate $\pseudob(\tilde{v})$ untwists (with hyperparameter(s)) the twisted value $\tilde{v}$ and recovers the ground truth $v$.
However, it follows that if the loss is proper, $\tilde{v} \in \pseudob(\tilde{v})$ and, if strictly proper, $\pseudob(\tilde{v}) = \{\tilde{v}\}$. 
This formally highlights the limitation of proper loss functions in twisted settings, namely, the inability of a proper loss to untwist the twisted value because the minimization of the loss is centered on what it ``perceives'' to be the ground truth.
The following result stipulates the cardinality of the Bayes tilted estimates. 
\begin{lemma} \label{lemma:btesingleton}
If the partial losses $\ell_{1}$ and $\ell_{-1}$ of $\loss : \mathcal{Y} \times [0,1] \rightarrow \mathbb{R}$ satisfy \textbf{(M)}, \textbf{(D)}, and \textbf{(S)} and are also strictly convex, then $|t_{\ell}(\tilde{v})| = 1$ for every $\tilde{v} \in [0,1]$. 
\end{lemma}
In Appendix~\ref{app-proof-bayestilt}, we provide an extended version of Lemma~\ref{lemma:btesingleton}, denoted Lemma~\ref{proptilt}, where we prove properties of Bayes tilted estimates for when $t_{\ell}$ is set-valued (e.g., set-valued monotonicity and symmetry, and analysis of extreme values).
As a consequence, we show that strict monotonicity of the partial losses is not sufficient to guarantee that $\pseudob$ is a singleton; in fact, strict convexity is required as in Lemma~\ref{lemma:btesingleton}. \\
 \noindent \textbf{An important class of twists} We now adopt more conventional ML notations and instead of a hidden ground truth $v$ and twisted ground truth $\tilde{v}$, we use $\posclean$ and $\posnoise$ to denote the ``clean'' and ``twisted'' posterior probabilities that $Y = 1$ given $X =x$, respectively.
Further, a ``\textbf{twist}'' refers to a general mapping $\posclean \mapsto \posnoise$, which could be a consequence of label/feature/adversarial noise.
 %
%
The following delineates a fundamental class of twists, important in the sequel.
  \begin{definition}\label{def:Bayesblunting}
A twist $\posclean \mapsto \posnoise$ is said to be Bayes blunting iff  $(\posclean \leq \posnoise \leq 1/2) \vee (\posclean \geq \posnoise \geq 1/2)$.
\end{definition}
The term ``blunting'' is inspired by adversarial training \cite{cmnoswMB}.
Intuitively, a Bayes blunting twist does not change the \textit{maximum a posteriori} guess for the label given the observation, but it does reduce algorithmic confidence in learned posterior estimates. 
Furthermore, Bayes blunting twists capture a very important twist (see Section~\ref{sec-related}): \textit{symmetric label noise} (SLN).
Under symmetric label noise with flip probability $p \in [0,1]$, the twisted posterior $\posnoise$ is given by $\posnoise = \posclean (1-p) + (1-\posclean) p$~\cite{rwCB}.
The following result readily follows via Definition~\ref{def:Bayesblunting} from consideration of $p$ for fixed $\posclean$.
\begin{lemma}\label{lemSymLabFlip}
    SLN is Bayes blunting for $p\leq 1/2$.
  \end{lemma}
  Historically,~\cite{rwCB} showed that proper loss functions are not robust to this twist which further motivates our consideration of twist-proper losses. \\
  \noindent \textbf{Twist-proper losses} To overcome these limitations of properness, we propose a generalized notion, called twist-properness, which utilizes hyperparameterization of the loss to untwist twisted posteriors into clean posteriors.
\begin{definition}\label{def:TwistProper}
A loss $\loss$ is twist-proper (respectively, strictly twist-proper) iff for any twist, there exists hyperparameter(s) such that $\posclean \in \pseudob(\posnoise)$ (respectively, $\{\posclean\} = \pseudob(\posnoise)$).
  \end{definition}
    Where a proper loss could perilously lead the learning algorithm to estimate $\posnoise$, a \textit{twist-proper} loss employs hyperparameters so that its Bayes tilted estimate recovers $\posclean$, hence guiding the algorithm to untwist the twisted posterior. 
  We emphasize the need for hyperparameters as otherwise, twist-properness would trivially enforce $\pseudob(\cdot) = [0,1]$.
  Recently, hyperparameterized loss functions have garnered much interest in ML, to name a few~\cite{barron2019general,lgghdFL,awakRB,li2021tilted,sypherd2019journal}, possibly because such losses allow practitioners to induce variegated models. 
  Indeed, hyperparameterized loss functions could be efficiently implemented via meta-algorithms, such as AutoML~\cite{he2021automl}, \textit{or} practically utilized in the burgeoning field of federated learning~\cite{kairouz2019advances}, where the hyperparameter(s) might yield more fine-grained ML model customization for edge devices. \\
Ostensibly, ``optimal'' hyperparameters requires explicit knowledge of the distribution and twist, \textit{and} each example in the training set requires a different hyperparameter to untwist its twisted posterior.
However, in the sequel, we show that a twist-proper loss, namely $\alpha$-loss, with a \textit{fixed} hyperparameter ($\alpha$) can untwist a large class of twists, i.e. Bayes blunting twists (such as SLN), better than $\log$-loss.
Thus, we posit through our experimental results in Section~\ref{sec:experiments} that the practitioner only needs peripheral, rather than explicit, knowledge of a Bayes blunting twist in the data. \\
    \noindent \textbf{Twist-(im)proper losses}
    \cite{lgghdFL} introduced the focal loss to improve class imbalance issues associated with dense object detection. 
    It generalizes the $\log$-loss and has become popular due to its success in such domains. 
    Recently, the focal loss has received increased scrutiny~\cite{cvcsOF}, where it was shown to be classification-calibrated but not strictly proper.
    Here, we determine the \textit{twist-properness} of the focal loss. 
    \begin{lemma}\label{lemFocalLoss}
 Define the \textbf{focal loss} via the following partial losses: $\partialloss{1}^{\text{FL}}(u) \defeq - (1-u)^\gamma \log u$ and $ \partialloss{-1}^{\text{FL}}(u) \defeq \partialloss{1}^{\text{FL}}(1-u)$, with $\gamma \geq 0$. Then the focal loss is \textbf{not} twist proper.
\end{lemma}
In the proof (see Appendix~\ref{app-proof-lemFocalLoss}), we also provide a proof that a loss which acts as a general ``wrapper'' of a loss, the Super Loss~\cite{cwrSA}, is \textit{not} twist proper. 
Concerning the focal loss, Lemma \ref{lemFocalLoss} is not necessarily an impediment for this loss function, which was designed to deal with a specific twist, class imbalance, and it does not prevent \textit{generalizations} of the focal loss that would be twist proper. 
However, our proof suggests that the Bayes tilted estimate~\eqref{defBTE} of such generalizations risks not being in a simple analytical form. 
Intuitively, twist-properness requires more than a trivial extension of the hyperparameter of the loss;
it also seems to require a certain \textit{symmetry}, which we observe with the following twist-proper loss, $\alpha$-loss.\\
    %
    %
\noindent \textbf{A twist-proper loss}
The $\alpha$-loss was first introduced in information theory in the early 70s~\cite{aIT} for $\alpha \in \mathbb{R}_{+}$ and recently received increased scrutiny in privacy and ML~\cite{lkspAT,sdskAT} for $\alpha \geq 1$.
Most recently,~\cite{sypherd2019journal} studied the calibration, optimization, and generalization characteristics of $\alpha$-loss in ML for $\alpha \in \mathbb{R}_{+}$. 
In particular, they experimentally found that $\alpha$-loss is robust to noisy labels under logistic regression and convolutional neural-networks.
We now provide our (extended) definition of the $\alpha$-loss in~\cpe.
\begin{definition}\label{defALPHALOSS}
For $\alpha \geq 0$, the $\alpha$-loss has the following partial losses: $\forall u \in [0,1]$, $ \ell_{1}^{\alpha}(u) \defeq \ell_{-1}^{\alpha}(1-u)$ where
\begin{eqnarray} \label{defALPHAlossEQ}
\ell_{1}^{\alpha}(u) \defeq \frac{\alpha}{\alpha-1}\cdot\left(1-u^{\frac{\alpha-1}{\alpha}}\right),
\end{eqnarray}
and by continuity we have $\ell_{1}^{0}(u) \defeq \infty$, $\ell_{1}^{1}(u) \defeq -\log u$, and $\ell_{1}^{\infty}(u) \defeq 1- u$.
For $\alpha < 0$, we let $\forall u \in [0,1]$,
\begin{align} \label{defalphalossnegalpha}
\ell_{1}^{\alpha}(u) \defeq \ell_{-1}^{-\alpha}(u) = \ell_{1}^{-\alpha}(1-u).
\end{align}
\end{definition}
For a plot of~\eqref{defALPHAlossEQ}, see Figure~\ref{fig:alphalossplot} in the appendix.
Note that the $\alpha$-loss is \textbf{(S)}ymmetric by construction, and that it continuously interpolates the $\log$-loss ($\alpha = 1$) which is proper~\cite{rwCB}.
Our definition extends the previous definitions with~\eqref{defalphalossnegalpha}, which induces a fundamental symmetry that is required for twist-properness and is utilized in the following result.
For any $u\in [0,1]$, we let $\logit(u) \defeq \log(u/(1-u))$ denote the logit of $u$.
  \begin{lemma}\label{lemPointwise}
  The following four properties, labeled \textbf{(a)-(d)}, all hold for $\alpha$-loss: 
  \textbf{(a)} \textbf{(M)}, \textbf{(D)}, \textbf{(S)} all hold, $\forall \alpha \in \mathbb{R} \setminus \{0\}$; \textbf{(b)} if $(\alpha = 0) \vee (\alpha = \pm \infty \wedge \posnoise = 1/2)$, then
  $t_{\ell^{\alpha}}(\posnoise) = \mathrm{[0,1]}$, if $\alpha \in \mathbb{R}\setminus \{0, \pm \infty\}$, then $t_{\ell^{\alpha}}(\posnoise) = \left\{\frac{\posnoise^\alpha}{\posnoise^\alpha+(1-\posnoise)^\alpha}\right\}$, 
  and if $\alpha \rightarrow \pm \infty$, then $t_{\ell^{\pm\infty}}(\posnoise) = \pm1$ or $\mp 1$, depending on the sign of $\posnoise-1/2$; 
  \textbf{(c)} hence, $\alpha$-loss is twist-proper with $\alpha^* = \logit(\posclean) / \logit(\posnoise)$;
  \textbf{(d)} for any Bayes blunting twist, $\alpha^* \geq 1$.
\end{lemma} 
    %
    The proof of Lemma~\ref{lemPointwise} can be found in Appendix~\ref{app-proof-lemPointwise}.
    Note that \textbf{(a)} readily follows from Definition~\ref{defALPHALOSS}.
    The Bayes tilted estimate in \textbf{(b)}, i.e. $t_{\ell^{\alpha}}$ for $\alpha \in \mathbb{R} \setminus \{0,\pm \infty\}$, is known in the literature as the $\alpha$-tilted distribution~\cite{aIT,lkspAT,sypherd2019journal}.
    We observe that the $\alpha$-tilted distribution is symmetric upon permuting $(\posnoise, \alpha)$ and $(1-\posnoise, -\alpha)$. 
    Hence, our nontrivial extension of the $\alpha$-loss induces a \textit{symmetry}, particularly useful for untwisting malevolent twists, which thereby yields twist-properness, \textbf{(c)}.
    Lastly, \textbf{(d)} indicates that $\alpha^{*} \geq 1$ for any Bayes blunting twist (e.g., SLN with $p \leq 1/2$); however, note that this holds merely for a given $x$, not over the whole domain $\mathcal{X}$. \\
    %
    %
    \noindent \textbf{Untwisting $\mathcal{X}$} Just as 
classification-calibration is a pointwise form of consistency \cite{bjmCCj},  twist-properness is a \emph{pointwise form of correction}. 
Extending twist-properness to the entire domain $\mathcal{X}$ seems to require learning a \textit{mapping} $\alpha : \mathcal{X} \rightarrow [-\infty, \infty]$, which is infeasible under standard ML assumptions, since one would need explicit knowledge of the distribution and twist.
Nevertheless, we show here that for a large class of twists, namely Bayes blunting twists, a fixed $\alpha_{0} > 1$ obtained \textit{non-constructively}, is strictly ``better'' than the proper choice, $\log$-loss ($\alpha = 1$).
We also provide a general \textit{constructive} formula for $\alpha^{*} \in \mathbb{R}$, calculated from distributional information. \\
%
%
In order to represent population quantities, we assume a marginal distribution $\meas{M}$ over $\mathcal{X}$ (following notation by~\cite{rwID}), from which the expected value of a \textit{loss} $\loss$ quantifies its true risk of a given classifier $h$. 
With a slight abuse of notation, we also let $\posclean, \posnoise: \mathcal{X} \rightarrow [0,1]$ denote the clean and twisted posterior \textit{mappings}, respectively. 
To evaluate the efficacy of the Bayes tilted estimate of $\alpha$-loss at untwisting the twisted posterior mapping and recovering the clean posterior mapping, we define the following averaged cross-entropy, given by
\negspace
\begin{align} \label{defCE}
  \nonumber \celoss(\posnoise, \posclean; \alpha) & \defeq \expect_{\X \sim \meas{M}} [\posclean(\X)\cdot - \log t_{\ell^{\alpha}}(\posnoise(\X)) \\
                                                     & \thickspace\thinspace + (1-\posclean(\X))\cdot - \log t_{\ell^{\alpha}}(1-\posnoise(\X))],
\end{align}
\negspace
where for convenience we used the symmetry property of $t_{\ell}$ from Appendix~\ref{app-proof-bayestilt}, i.e., $t_{\ell^{\alpha}}(1-\posnoise(\X)) = 1 - t_{\ell^{\alpha}}(\posnoise(\X))$.
Following~\cite{sfBF}, we denote the binary entropy as 
$H_{\text{b}}(u) \defeq -u\cdot \log(u) - (1-u)\cdot \log(1-u)$,
for $u\in [0,1]$.
We let $H(\posclean)$ represent an averaged binary entropy of the $\posclean$-mapping, given by
\begin{align} \label{eq:avgcleanentropy}
H(\posclean) &\defeq \expect_{\X \sim \meas{M}} [H_{\text{b}}(\posclean(X))].
\end{align}
With~\eqref{defCE} and~\eqref{eq:avgcleanentropy}, we readily obtain (cf.~\cite{cover1999elements}),
\begin{align} \label{eq:kldivalpha}
D_{\klloss}(\posnoise,\posclean;\alpha) \defeq \celoss(\posnoise, \posclean; \alpha) - H(\posclean),
\end{align}
that is, a KL-divergence between the $\alpha$-Bayes tilted estimate of the twisted posterior and the clean posterior mappings.
Intuitively, $D_{\klloss}(\posnoise,\posclean;\alpha)$ aggregates a series of information-trajectories, strictly dependent on $\alpha$ (either fixed or a mapping), each tracing a path on the probability \textit{simplex} between the two posterior mappings for every $x \in \mathcal{X}$.
%
Slightly more restrictive than Definition~\ref{def:Bayesblunting}, we define a \textit{strictly} Bayes blunting twist as a Bayes blunting twist where $(\eta_{c} < \eta_{t} \leq 1/2) \vee (\eta_{c} > \eta_{t} \geq 1/2)$; we state one of our main results whose proof is in Appendix~\ref{app-proof-thmALPHABLUNTING}.
%
%
%
%
%
\begin{theorem}\label{thmALPHABLUNTING}
For any strictly Bayes blunting twist $\posclean \mapsto \posnoise$, 
there exists a fixed $\alpha_{0} > 1$ and an optimal $\alpha^{\star}$-mapping, $\alpha^{\star}: \mathcal{X} \rightarrow \mathbb{R}_{>1}$, which induces
the following ordering 
\begin{align} \label{eq:bayesbluntingdesiredstatement}
D_{\klloss}(\posnoise, \posclean; 1) > D_{\klloss}(\posnoise, \posclean; \alpha_{0}) \geq D_{\klloss}(\posnoise, \posclean; \alpha^\star).
\end{align}
\end{theorem}
This result answers in the affirmative that untwisting $\mathcal{X}$ for a large class of twists with a fixed hyperparameter $\alpha_{0}$ is better than simply using the proper choice, i.e., $\alpha=1$ ($\log$-loss). 
Specifically, Theorem~\ref{thmALPHABLUNTING} holds for SLN, which is a strictly Bayes blunting twist for flip probability $0 < p \leq 1/2$ (Lemma~\ref{lemSymLabFlip}).
The result also states that there exists a mapping $\alpha^{\star}: \mathcal{X} \rightarrow \mathbb{R}_{>1}$ which optimally untwists the strictly Bayes blunting twist, however, finding this mapping in practice currently remains out of reach.
Regarding the search for a fixed $\alpha_{0} > 1$ in practice,~\cite{sypherd2019journal} showed via landscape analysis and experiments on SLN that the search space for $\alpha_{0}$ is bounded (due to saturation), typically $\alpha_{0} \in [1.1,8]$.
In Section~\ref{sec:experiments}, we report experimental results for several $\alpha$; we also incorporate a method inspired by~\cite{mrowLF} to estimate the amount of SLN in training data and thus estimate $\alpha_{0}$ using Lemma~\ref{lemPointwise}. \\
%
%
%
%
%
%
Theorem~\ref{thmALPHABLUNTING} gave a \textit{nonconstructive} indication for the optimal regime of $\alpha$ for strictly Bayes blunting twists. 
Our next result gives a \textit{constructive} formula for a fixed $\alpha$ for any twist. 
Given $B > 0$, let $\meas{M}(B)$ denote the distribution restricted to the support over $\mathcal{X}$ for which we have almost surely 
\negspace
\begin{eqnarray} \label{eqclipSUP}
(1+\exp (B))^{-1} \leq \posnoise(x) \leq (1+ \exp (-B))^{-1},
\end{eqnarray}
\negspace
and let $p(B) \in [0,1]$ be the weight of this support in $\meas{M}$. 
We let $\meas{D}(B)$ denote the product distribution on examples ($\mathcal{X} \times \mathcal{Y}$) induced by marginal $\meas{M}(B)$ and posterior $\posclean$ (see~\cite[Section 4]{rwID}). 
We define the logit-\textit{edge} as
\begin{eqnarray} \label{defETAB}
\edge \defeq (1/B)\cdot \expect_{(\X,\Y) \sim \meas{D}(B)} \left[\Y\cdot \logit(\posnoise(\X))\right],
\end{eqnarray}
where we note that $\edge \in [-1,1]$ due to the assumption in~\eqref{eqclipSUP}.
Finally, we let $q \defeq (1+\edge)/2 \in [0,1]$.
\begin{theorem}\label{thmALPHAGLOB}
Let $B > 0$.
If $p(B) = 1$ and we fix $\alpha = \alpha^*$ with $\alpha^*  \defeq \logit(q)/B$, then the following bound holds  
\begin{align} \label{eq:thmalphaUB}
D_{\klloss}(\posnoise,\posclean;\alpha^{*}) \leq H_{\text{b}}(q) - H(\posclean).
\end{align}
\end{theorem}
The proof of Theorem~\ref{thmALPHAGLOB} is in Appendix~\ref{app-proof-thmALPHAGLOB}, where we also prove an extended version of the result when $p(B) < 1$.
In addition, we provide a simple example where $\klloss(\posnoise, \posclean; \alpha^*)$ can vanish with respect to $\klloss(\posnoise, \posclean; 1)$ (the ``proper'' choice). 
Intuitively, the difference on the right-hand-size of~\eqref{eq:thmalphaUB} in Theorem~\ref{thmALPHAGLOB} is reminiscent of a Jensen's gap.
Also in the proof of Theorem~\ref{thmALPHAGLOB}, we find that if $|\alpha^{*}|$ is large, there is more ``flatness'' in the bounded terms near $\alpha^{*}$.
Hence, this suggests that a choice of $\alpha_{0}$ ``close-enough'' to $\alpha^{*}$ could yield similar performance.


\section{Sideways boosting a surrogate loss}\label{sec-boosting}

With the theory of twist-properness and the twist-proper $\alpha$-loss in hand, we now turn towards the algorithmic contribution of this work. 
As stated in the introduction, $\alpha$-loss was recently implemented in logistic regression and in deep neural networks~\cite{sypherd2019journal}, and was found to be more robust to symmetric label noise for fixed $\alpha > 1$ than the proper $\log$-loss ($\alpha = 1$).
Thus, in order to complement our theory of twist-properness and these recent results of $\alpha$-loss, we also practically implement $\alpha$-loss in \textit{boosting}. \\
Formally, we have a training sample $\mathcal{S} \defeq \{(\ve{x}_i, y_i), i\in [m]\} \subset \mathcal{X} \times \mathcal{Y}$ of $m$ examples, where $[m] \defeq \{1, 2, ..., m\}$ and note that $\mathcal{Y} = \{-1,+1\}$. 
We write $i\sim \mathcal{S}$ to indicate sampling example $(\ve{x}_i, y_i)$ according to $\mathcal{S}$. 
Following~\cite{ssIBj,cssLR,nnOT}, the boosting algorithm minimizes an expected surrogate loss with respect to $\mathcal{S}$ in order to learn a real-valued classifier $H: \mathcal{X} \rightarrow \mathbb{R}$ given by 
\begin{eqnarray}
H_{\ve{\beta}} & \defeq & \mbox{$\sum_j \beta_j h_j$},
\bignegspace
\end{eqnarray}
where $\{h_\cdot\ : \mathcal{X} \rightarrow \mathbb{R}\}$ are \weak~(weak learning) classifiers with slightly better than random classification accuracy.
The oracle \weak~returns an index $j \in \mathbb{N}$, and the task for the boosting algorithm is to learn the coordinates of $\ve{\beta}$, initialized to the null vector. 
%
%
In our general framework, the losses we consider are the surrogates $F$ in Lemma \ref{lemUstar}, essentially convex and non-increasing functions, adding the condition that they are twice differentiable. We compute weights using the blueprint of \cite{fGFA}, which uses the full $H_{\ve{\beta}}$,
\begin{eqnarray}
w_i & \defeq & -F' (y_i H_{\ve{\beta}}(\ve{x}_i)), \forall i \in [m]. \label{eqDIAMF}
\end{eqnarray}
Via Lemma \ref{lemUstar}, weights $w_{i}$ are non-negative and tend to increase for an example given the wrong class by the current weak classifier $h_j$, thus, weighting puts emphasis on ``hard'' examples. 
Assuming strict concavity of the pointwise Bayes risk and property \textbf{(D)} of Lemma~\ref{proptilt} in Appendix~\ref{app-proof-bayestilt}, we get from the definition of $F$ in~\eqref{defsur} that
\begin{align}\label{defweights}
  F'(z) = {\poibayesrisk'}^{-1}(-z) = (\partialloss{-1}\circ \pseudob - \partialloss{1}\circ \pseudob)^{-1}(-z). 
\end{align}
%
%
%
%
%
%
We thus need to invert the \textit{difference} of the partial losses to recover $F'$. 
The inversion is easy for the log-loss because of properties of the $\log$ function, and for the square loss because partial losses are quadratic functions. 
However, for hyperparameterized losses, such as the $\alpha$-loss, the inversion in~\eqref{defweights} is nontrivial.  
We circumvent this difficulty by proposing a novel boosting algorithm, \pmboost, given in Algorithm~\ref{cboost}.
Rather than using $F'$ as in~\eqref{eqDIAMF} for the weight update in Step 2.1, \pmboost~uses an approximation function $\pseudoinvlink$ non-negative and increasing, that we call \textit{pseudo-inverse link} (\pil), which is studied in general in Appendix~\ref{app-PIL}.
%
\begin{algorithm}[t]
\caption{\pmboost}\label{cboost}
\begin{algorithmic}
  \STATE  \textbf{Input} sample ${\mathcal{S}}$, number of iterations $T$, $a_f>0$, \pil~$\pseudoinvlink$;
\STATE  Step 1 : let $\bm{\beta} \leftarrow \bm{0}$; // first classifier, $H_{\ve{0}} = 0$
\STATE  Step 2 : \textbf{for} $t = 1, 2, ..., T$
\STATE  \hspace{1.1cm} Step 2.1 : let $w_{i} \leftarrow \pseudoinvlink(-y_i H_{\ve{\beta}}(x_i)), \forall i \in [m]$\;
\STATE  \hspace{1.1cm} Step 2.2 : let $j \leftarrow
\weak({\mathcal{S}}, \bm{w})$\; 
\STATE  \hspace{1.1cm} Step 2.3 : let $\edge_j   \leftarrow
(1/m) \cdot \sum_{i}
{w_{i} y_{i} h_j(\bm{x}_i)}$\;
\STATE  \hspace{1.1cm} Step 2.4 : let $\beta_j \leftarrow \beta_j + a_f \edge_j$
\; 
\STATE \textbf{Output} $H_{\ve{\beta}}$.
\end{algorithmic}
\end{algorithm}
Specifically, in Lemma~\ref{lemPILalphaloss}, we provide $\pseudomirrorinv$ for $\alpha$-loss, given in~\eqref{eq:tildefupdate}.
Furthermore in Lemma~\ref{edgeALPHA}, we show that there exists $K > 0$ such that, for almost all $z \in \mathbb{R}$, $|(\pseudomirrorinv - (-\poibayesrisk')^{-1})(z)| \lesssim K/\alpha$. 
We now theoretically analyze \pmboost, and we make two classical assumptions on \weak~\cite{ssIBj,nwLO}.
\begin{assumption} \textbf{(R)}\label{boundH}
 The weak classifiers have bounded range: $\exists M>0$ such that $|h_j(x_i)| \leq M, \forall j$.
\end{assumption}
    Let $\tilde{\edge}_j  \defeq   m \cdot \edge_j / (\ve{1}^\top \ve{w}_j) \in [-M, M]$ be the normalized edge of the $j$-th weak classifier, where with a slight abuse of notation of~\eqref{defETAB}, $\edge_j$ is the (unnormalized) edge (Step 2.3). 
\begin{assumption} \textbf{(WLA)}\label{defWLA}
  The weak classifiers are not random: $\exists \upgamma > 0$ such that $|\tilde{\edge}_j| \geq \upgamma \cdot M, \forall j$.
\end{assumption}
Since we employ $\pseudoinvlink$ instead of $F'$ in~\pmboost, we need two more functional assumptions on the first- and second-order derivatives of $F$. 
The \textit{edge discrepancy} of a function $F$ on weak classifier $h_j$ at iteration $t$ is given by
\begin{eqnarray}
  \Delta_j(F) \defeq \left|\expect_{i\sim \mathcal{S}} \left[ y_i h_j(\ve{x}_i) F'(y_i H_{\ve{\beta}}(\ve{x}_i))\right] - \edge_j\right|,\negspace
\end{eqnarray}
which is the absolute difference of the edge using (the derivative of) $\surloss$ vs. using \pmboost's $\pseudoinvlink$ (implicit in $\edge_j$).
\begin{assumption}\textbf{(O1, O2)}\label{defLT}
$\exists \zeta , \pi \in [0,1)$ such that:
  \begin{itemize}[left= -0.2cm]
  \item [] \textbf{(O1)} the edge discrepancy is bounded $\forall t$: $\Delta_j(\surloss) \leq \zeta \cdot \edge_j$, where $j$ is returned by \weak~at iteration $t$;
    \item [] \textbf{(O2)} the curvature of $\surloss$ is bounded: $\surlossprime^* \defeq \sup_{z} \surloss''(z) \leq (1-\zeta)(1+\pi)/(a_f  M^2)$.
  \end{itemize}
\end{assumption}
Note that \textbf{(O2)} is quite mild for specific sets of functions, e.g., proper canonical losses are Lipschitz~\cite{rwCB}, so \textbf{(O2)} can in general be ensured by a simple renormalization of the loss. 
On the other hand, \textbf{(O1)} becomes progressively harder to ensure as the number of iterations increases because the choices of the \weak~will tend to collapse.
Fortunately in practice (see Section~\ref{sec:experiments}), this phenomenon is not limiting to effective boosting.
Let $\tilde{w}_t \defeq \ve{1}^\top \ve{w}_t$, the total weight at iteration $t$ in \pmboost.
  \begin{theorem}\label{genBOOST}
    Suppose \textbf{(R, WLA)} hold on \weak~and \textbf{(O1, O2)} hold on $\surloss$, for each iteration of \pmboost. Denote $Q(\surloss) \defeq  2 \surlossprime^*/(\upgamma^2 (1-\zeta)^2(1-\pi^2))$. The following holds:
    \begin{itemize}[left= -0.2cm]
        \item on the risk defined by $\surloss$: $\forall z^* \in \mathbb{R}, \forall T>0$, if we observe $\sum_{t=0}^{T} \tilde{w}^2_t \geq Q(\surloss)\cdot (\surlossprime(0) - \surlossprime(z^*))$, then 
     \begin{eqnarray} \label{eq:boostsurrogate}
 \expect_{i\sim \mathcal{S}}\left[\surlossprime(y_i H_{\ve{\beta}}(x_i))\right] & \leq & \surlossprime(z^*).
      \end{eqnarray}
      \item on edge distribution: $\forall \theta \geq 0, \forall \varepsilon \in [0,1], \forall T > 0$, letting $F_{\varepsilon,\theta} \defeq (1-\varepsilon) \inf \surlossprime + \varepsilon \surlossprime(\theta)$, if the number of iterations satisfiees $T \geq \frac{1}{\varepsilon^2} \cdot \frac{Q(\surloss)\cdot (\surlossprime(0) - F_{\varepsilon,\theta})}{\pseudoinvlink^2(-\theta)}$, then
\begin{eqnarray} \label{eq:boostmargin}
    \pr_{i\sim \mathcal{S}}\left[y_i H_{\ve{\beta}}(x_i) \leq \theta\right] & \leq & \varepsilon.
  \end{eqnarray}
\end{itemize}
\end{theorem}
  Thus, Theorem \ref{genBOOST} gives boosting compliant convergence on training, and the synthesis of~\eqref{eq:boostsurrogate} and~\eqref{eq:boostmargin} provides a very strong convergence guarantee. 
  When classical assumptions about the loss of interest are satisfied, such as it being Lipschitz (ensured for proper canonical losses~\cite{rwCB}), there is a natural extension to generalization following standard approaches~\cite{bmRA,sfblBT}. 
  See Appendix~\ref{app-proof-genBOOST} for the proof of Theorem~\ref{genBOOST}, and for additional remarks regarding its \textit{optimality} and further application to addressing discrepancies due to \textit{machine type approximations}.
     
     \begin{table}[t]
  \begin{center}
    \resizebox{\columnwidth}{!}{
\begin{tabular}{cccccc}
                        &               &               &           &                   & \\
\multirow{1}{*}{Dataset} &  \multirow{1}{*}{Algorithm}             &                \multicolumn{4}{c}{Feature Noise Twister}                               \\ \cline{3-6} 
            &      & $p=$ $0$           & $0.15$            & $0.25$                    & $0.5$   \\ \hhline{======}
            & AdaBoost & $1.00 \pm 0.00$  & $0.99 \pm 0.01$ & $0.97 \pm 0.01$  & $0.88 \pm 0.02$  \\ \cline{2-6} 
            & us ($\alpha = 1.1$) & $1.00 \pm 0.00$  & $1.00 \pm 0.00$ & $0.99 \pm 0.01$ & $0.91 \pm 0.02$         \\ \cline{2-6} 
            & us ($\alpha = 2.0$) & $1.00 \pm 0.00$  & $1.00 \pm 0.00$ & $\mathbf{1.00 \pm 0.00}$ & $0.91 \pm 0.03$           \\ \cline{2-6} 
\multirow{-3}{*}{xd6}  & us ($\alpha = 4.0$) & $1.00 \pm 0.00$  & $1.00 \pm 0.00$ & $1.00 \pm 0.00$ & $\mathbf{0.96 \pm 0.02}$ \\ \cline{2-6} 
            & XGBoost & $1.00 \pm 0.00$  & $0.97 \pm 0.02$ & $0.96 \pm 0.01$ & $0.83 \pm 0.03$           \\ 
\end{tabular}
}
\caption{Accuracies of AdaBoost, \pmboost~(for $\alpha \in \{1.1,2,4\}$), and XGBoost on the xd6 dataset affected by the feature noise twister with the flipping probability $p = \{0,0.15,0.25,0.5\}$. All three algorithms were trained with depth $3$ regression trees. For each value of $\alpha$, we set $a_{f} = 8$. Note that the xd6 dataset is perfectly classified (when there is no twist) by a Boolean formula on the features, given in~\cite{buntine1992further}, which explains the performance when $p = 0$. 
}
\label{table:xd6featurenoise}
\end{center}
     \bignegspace
     \bignegspace
\end{table}

\section{Experiments} \label{sec:experiments}

\begin{figure*}[h]
  \begin{center}
\includegraphics[scale=.5]{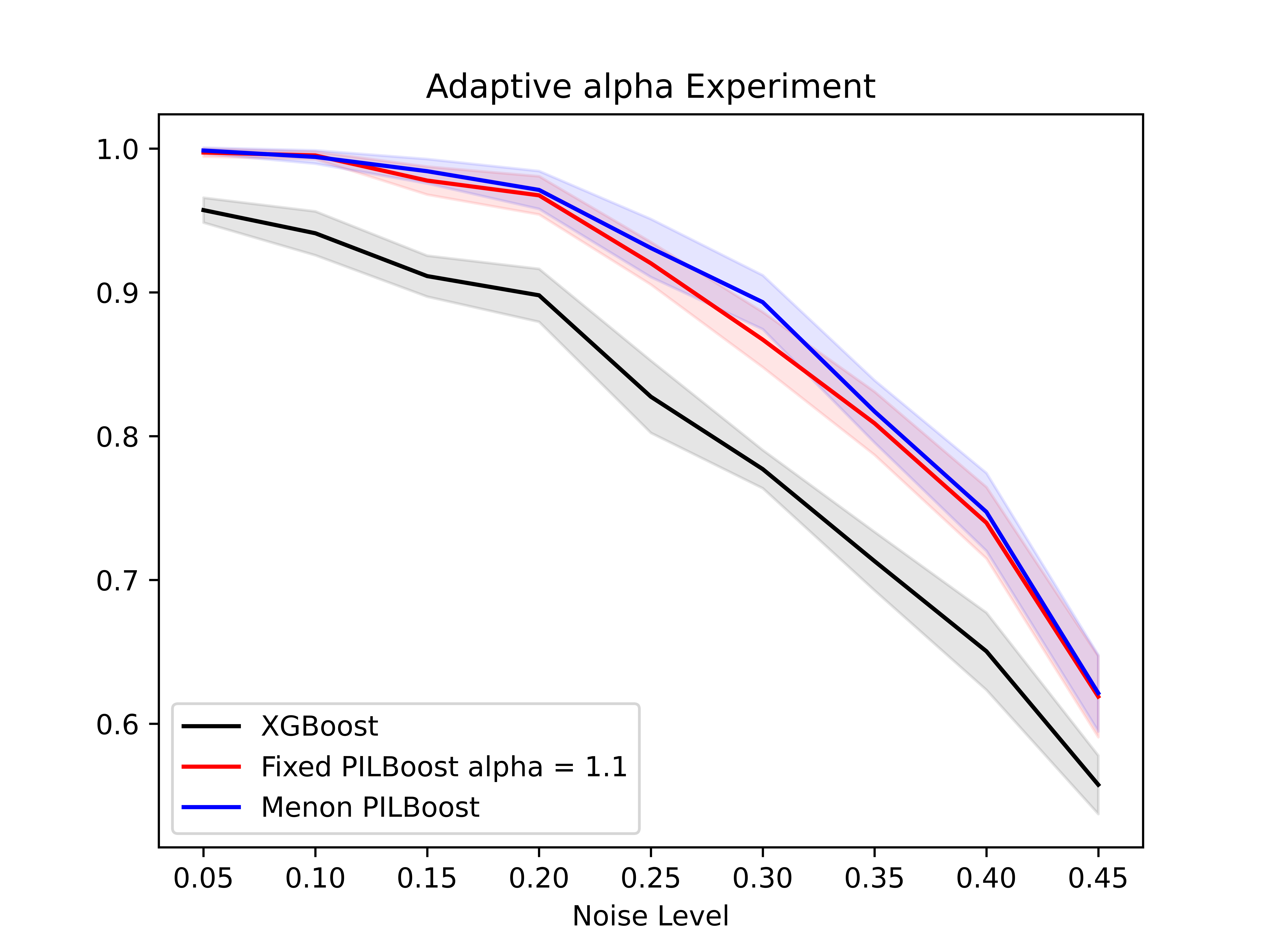}
\end{center}
\caption{Adaptive $\alpha$ experiment on the xd6 dataset with depth 3 regression trees. 
Solid curves correspond to mean classification accuracy and shaded areas are the associated $95\%$ confidence intervals obtained from a t-test.
For each label noise value, we train three algorithms: 1) vanilla XGBoost; 2) \pmboost~with fixed $\alpha = 1.1$; 3) and, an adaptive $\alpha$ \pmboost~(we refer to as Menon \pmboost). 
For details regarding Menon \pmboost, refer to Class Noise Twister in the main body. 
The result suggests that a fixed value of $\alpha = 1.1$ in \pmboost~is good, but approximating $\alpha_{0}$ does induce slightly better model performance. 
For general twists, we suggest this heuristic (or some variant) as inspired by~\cite{mrowLF} could be used to learn $\alpha_{0}$.
Further experimental consideration is given in Appendix~\ref{sec:adaptivealphaexperiments}.
}
  \label{fig:ConfusionMatrix1} 
\vspace{-0.5cm}
\end{figure*}

\begin{figure*}[h]
\begin{center}
\begin{tabular}{ccc}
\hspace{-0.3cm} \includegraphics[trim=20bp 15bp 30bp 15bp,clip,width=0.32\linewidth]{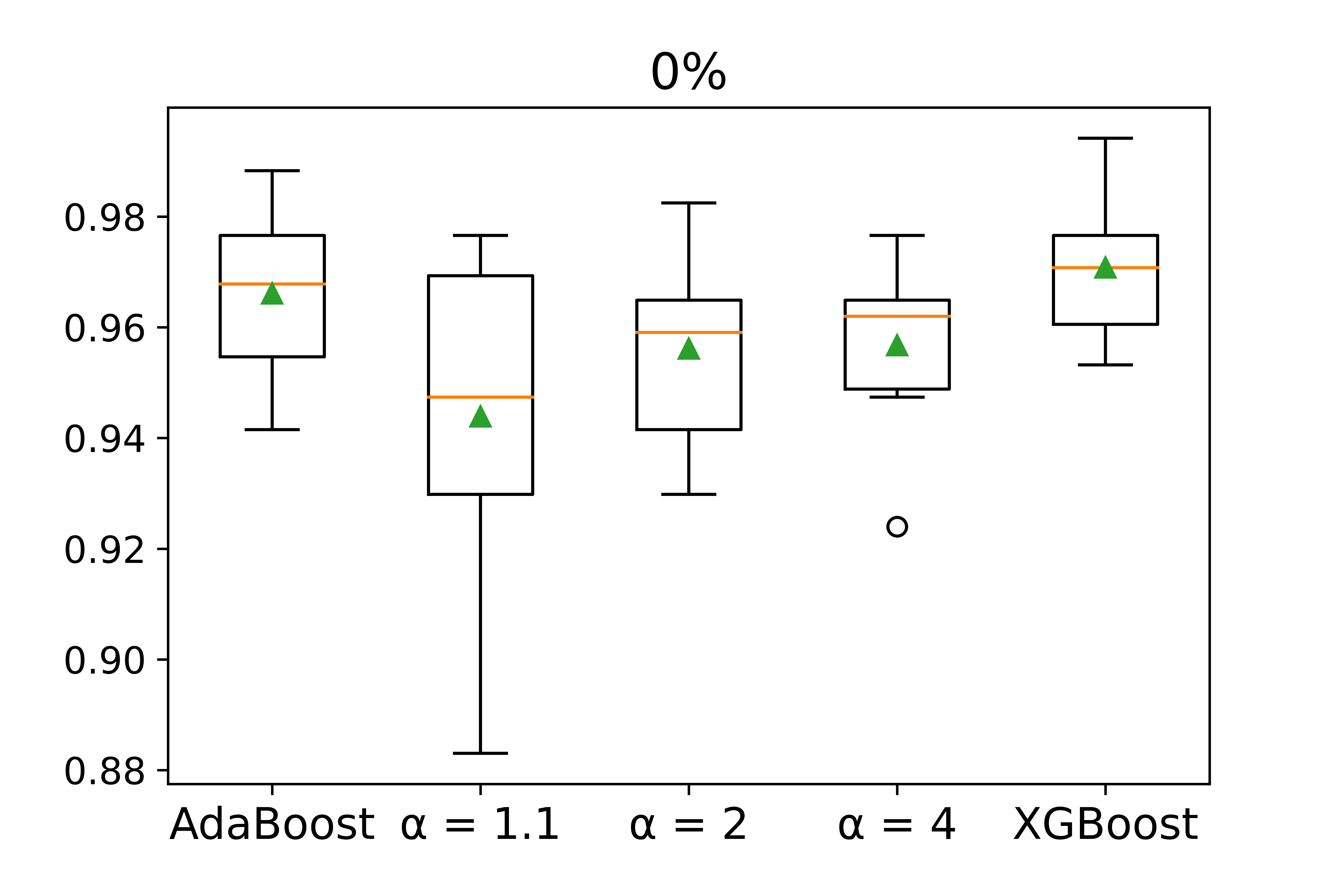}
  \hspace{-0.3cm} & \hspace{-0.2cm} \includegraphics[trim=20bp 15bp 30bp 15bp,clip,width=0.32\linewidth]{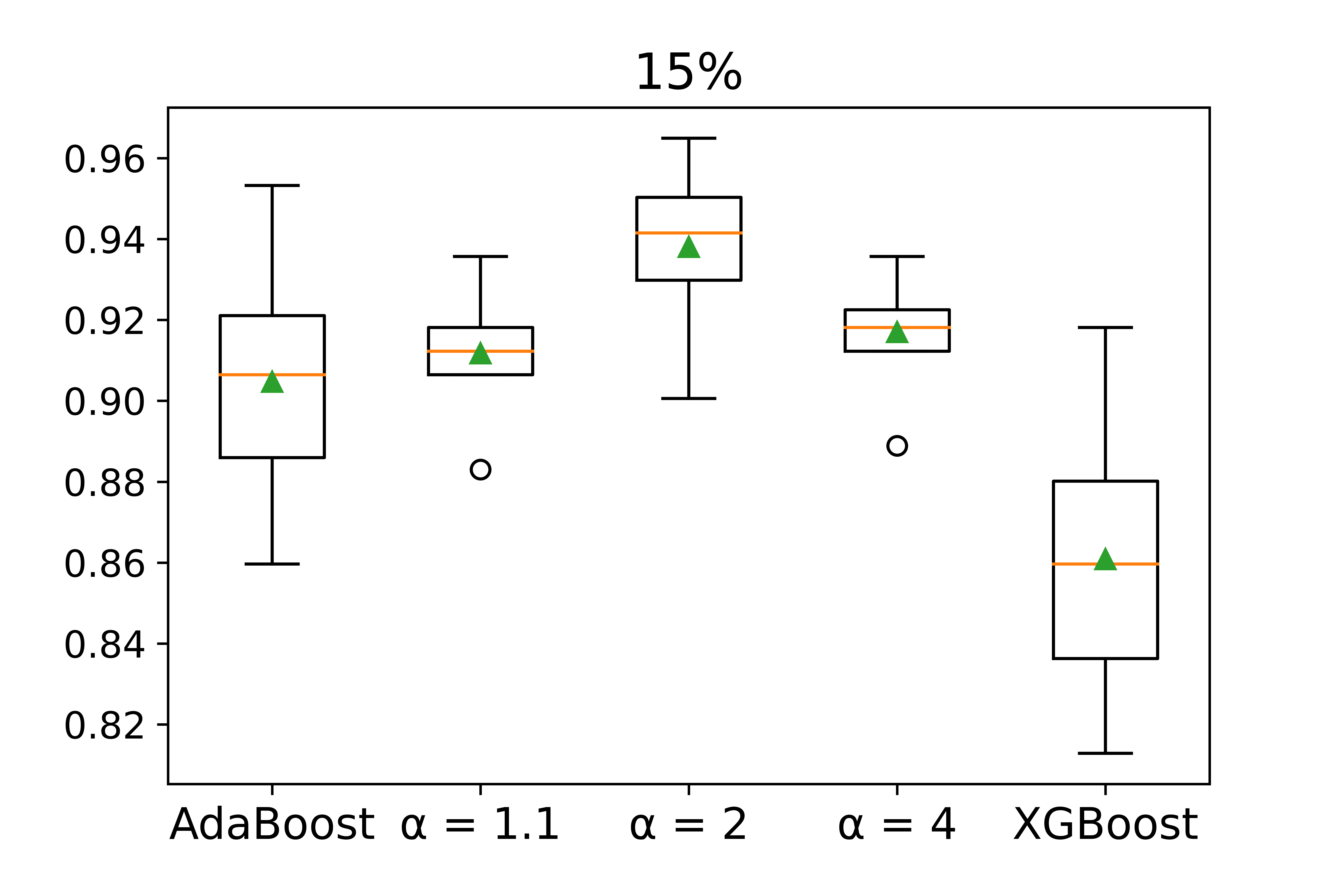}
  & \hspace{-0.3cm} \includegraphics[trim=20bp 15bp 30bp 15bp,clip,width=0.32\linewidth]{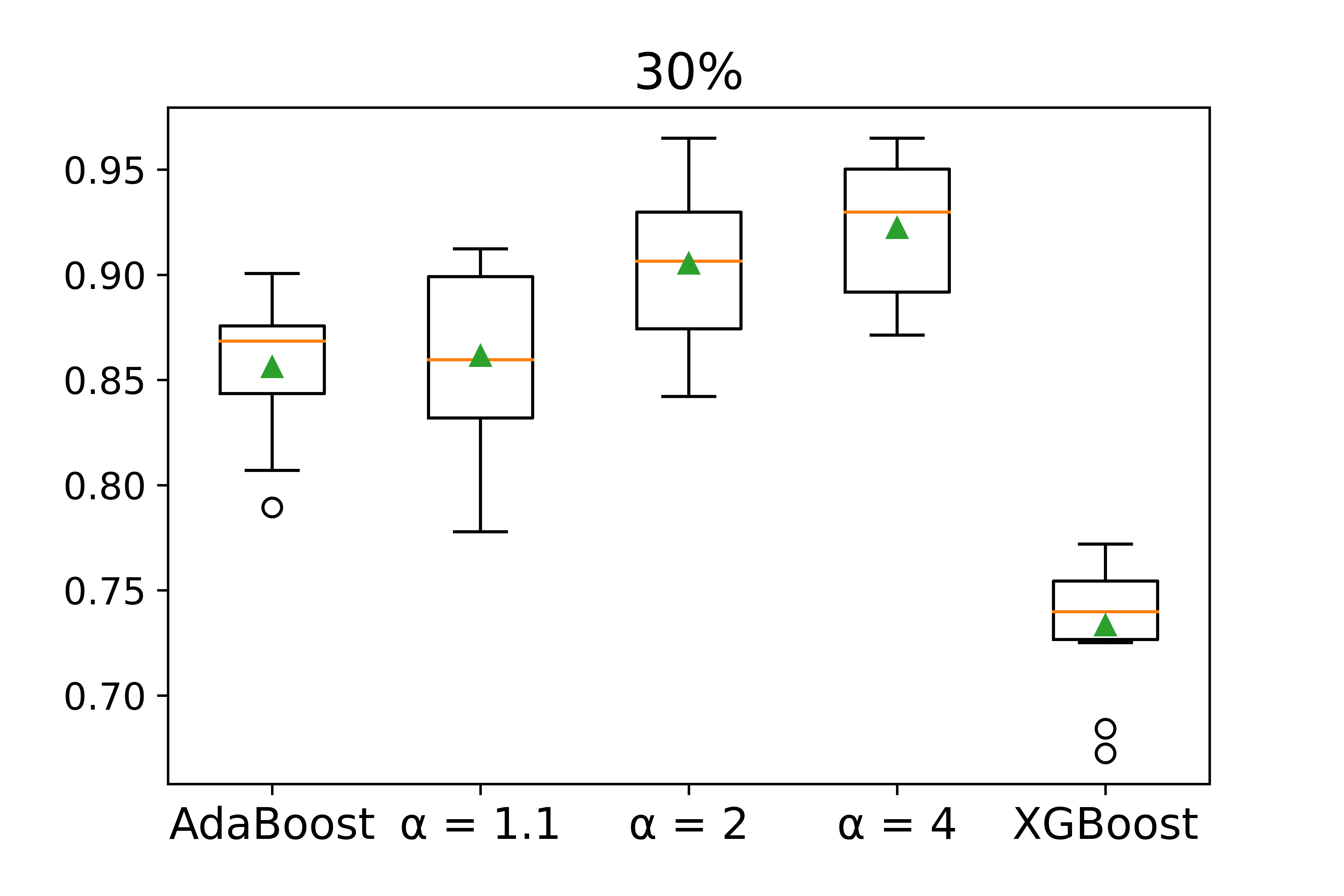}
\end{tabular}
\end{center}
\vspace{-0.4cm}
\caption{Box and whisker plots reporting the accuracy of AdaBoost, \pmboost~(for $\alpha \in \{1.1,2,4\}$), and XGBoost on the cancer dataset affected by the class noise twister with $0\%$, $15\%$, and $30\%$ twist. 
Note that the orange line is the median, the green triangle is the mean, the box is the interquartile range, and the circles outside of the whiskers are outliers.
All three algorithms were trained with decision stumps (depth $1$ regression trees). For $\alpha = 1.1, 2,$ and $4$, we set $a_{f} = 7, 2,$ and $4$, 
respectively. Numeric values corresponding to the box and whisker plots are provided in Table~\ref{table:cancerclassnoise} in Section~\ref{sec:additionalclassnoise}. We find that \pmboost~has gains over AdaBoost and XGBoost when there is twist present, and $\alpha^{*}$ (of our set) increases as the amount of twist increases, which follows theoretical intuition (Lemma~\ref{lemPointwise}).}
  \label{fig:breastcancerlabelnoise}
\vspace{-0.5cm}
\end{figure*}

\begin{figure}[h]
\begin{center}
\begin{tabular}{cc}
\hspace{-0.3cm} 
\includegraphics[trim=0bp 0bp 30bp 15bp,clip,width=0.35\linewidth]{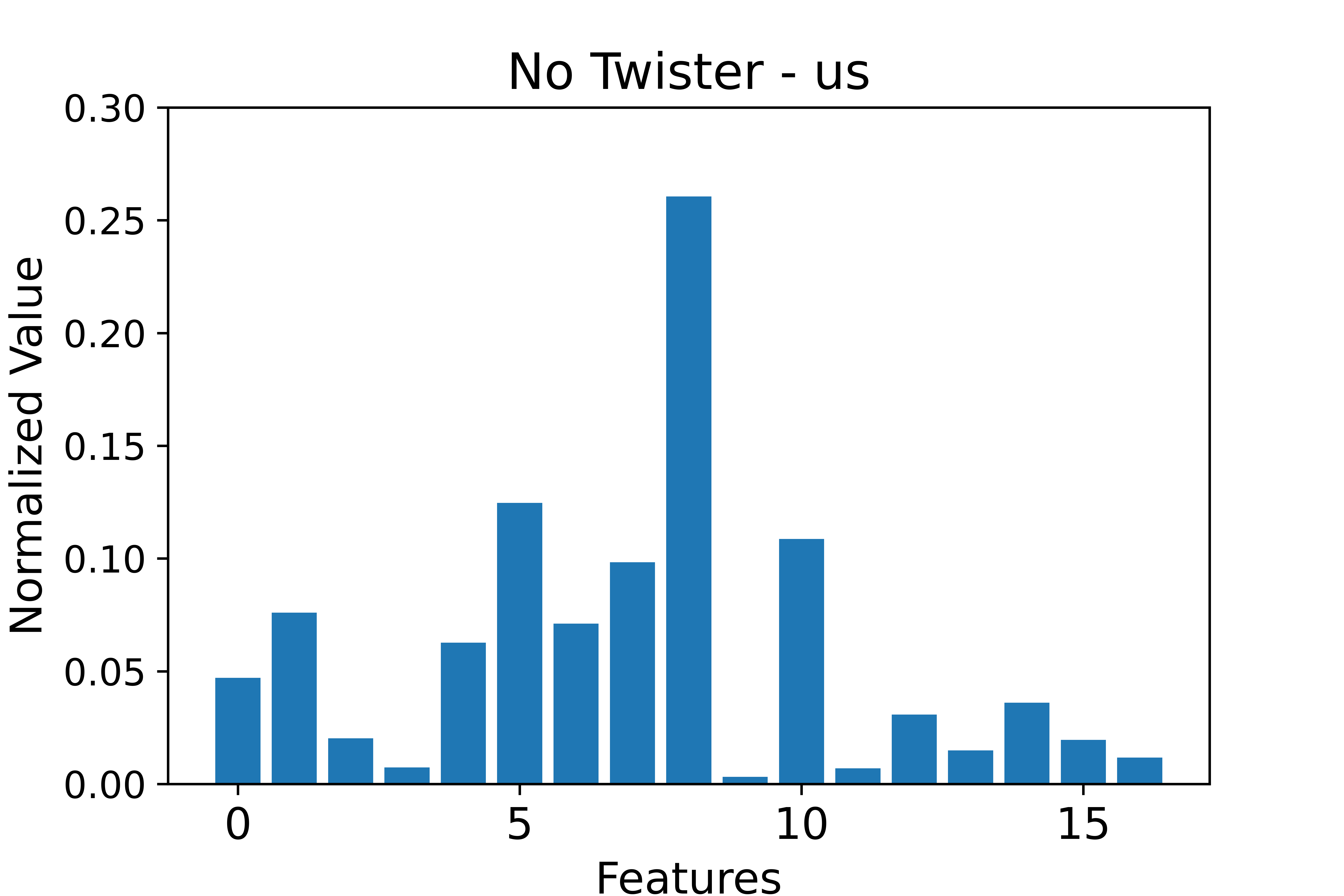}
\hspace{-0.3cm} & \hspace{-0.3cm} \includegraphics[trim=0bp 0bp 30bp 15bp,clip,width=0.35\linewidth]{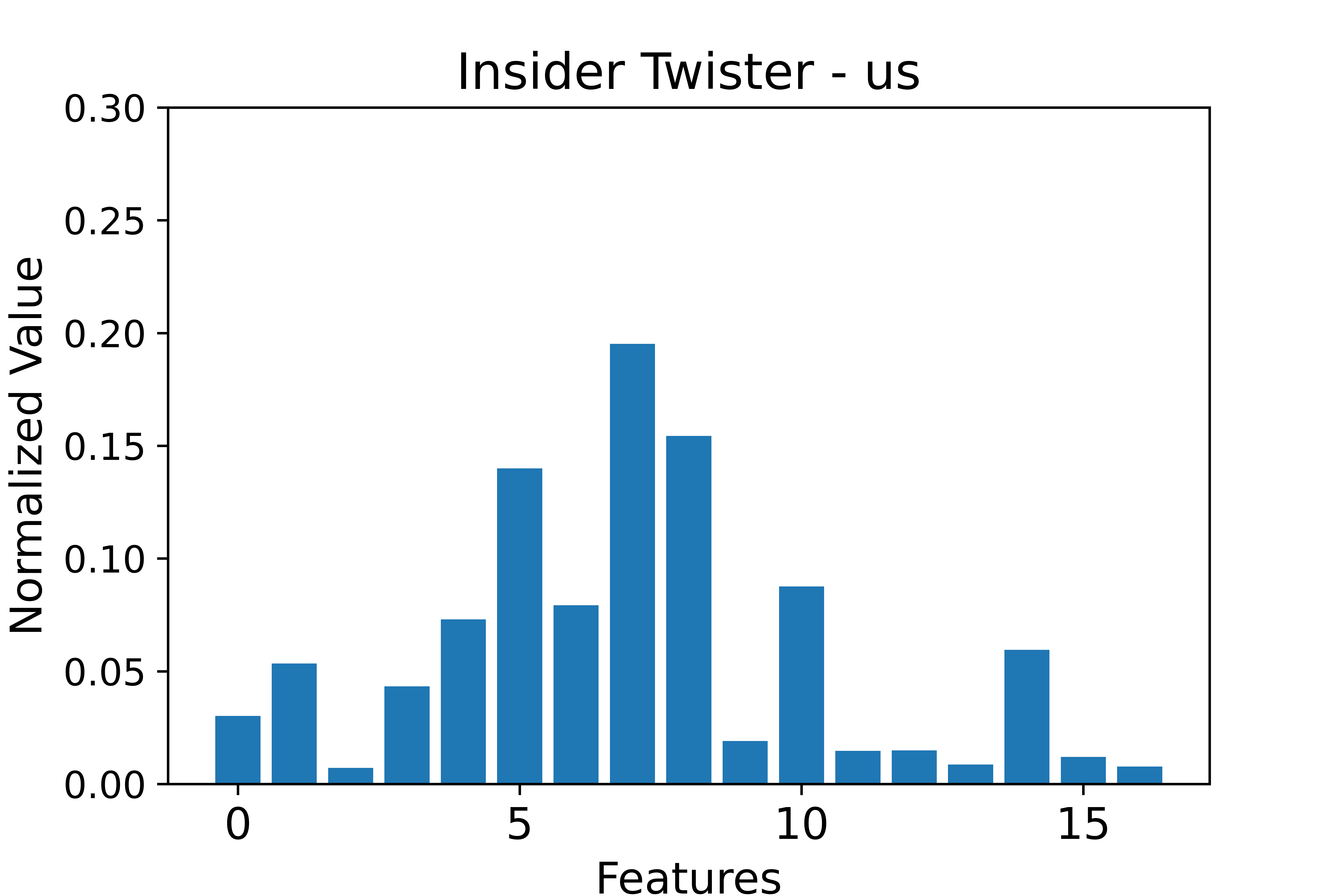}
                  \hspace{-0.3cm} \\
  \hspace{-0.2cm} \includegraphics[trim=0bp 0bp 30bp 15bp,clip,width=0.35\linewidth]{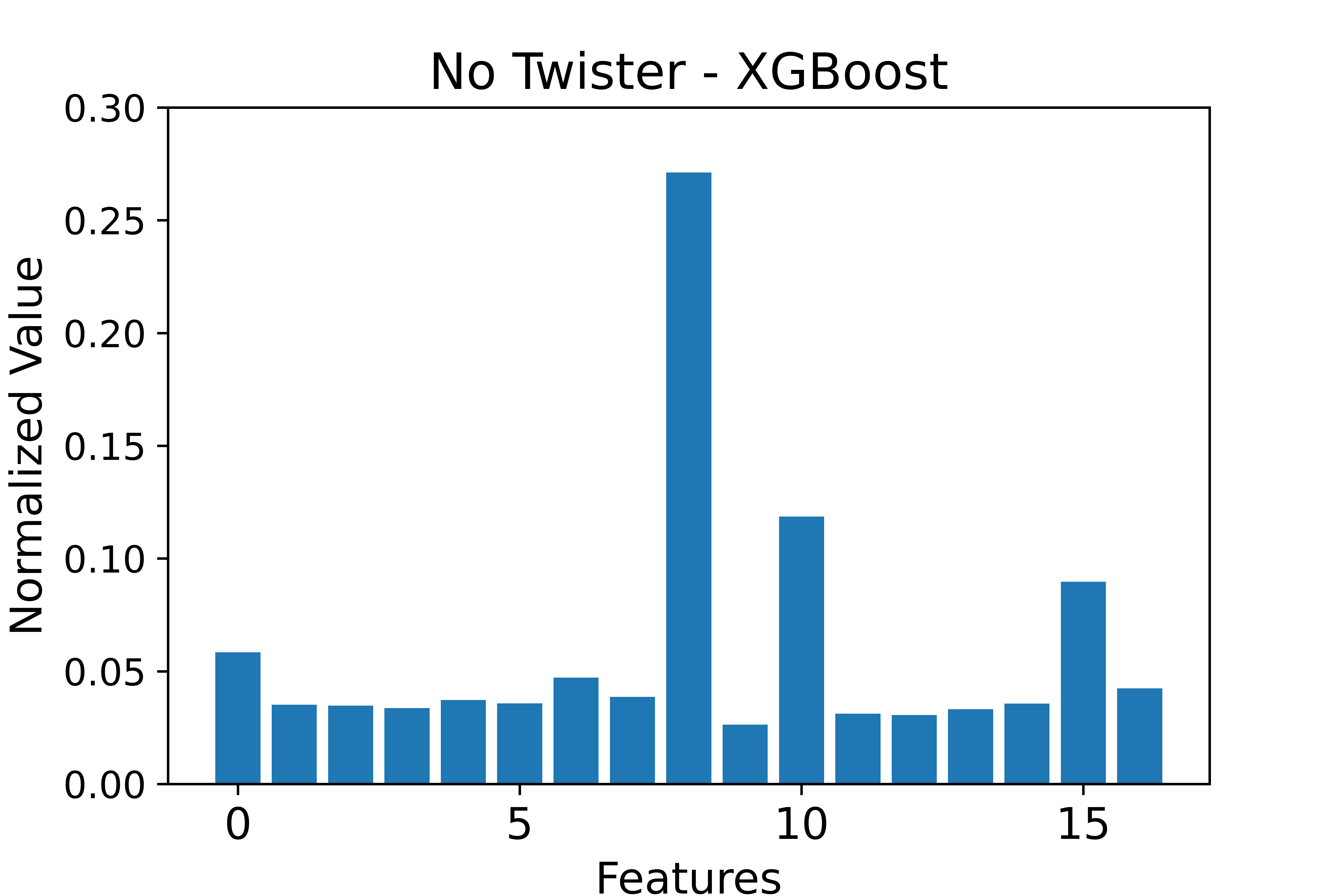}
\hspace{-0.3cm} & \hspace{-0.3cm} \includegraphics[trim=0bp 0bp 30bp 15bp,clip,width=0.35\linewidth]{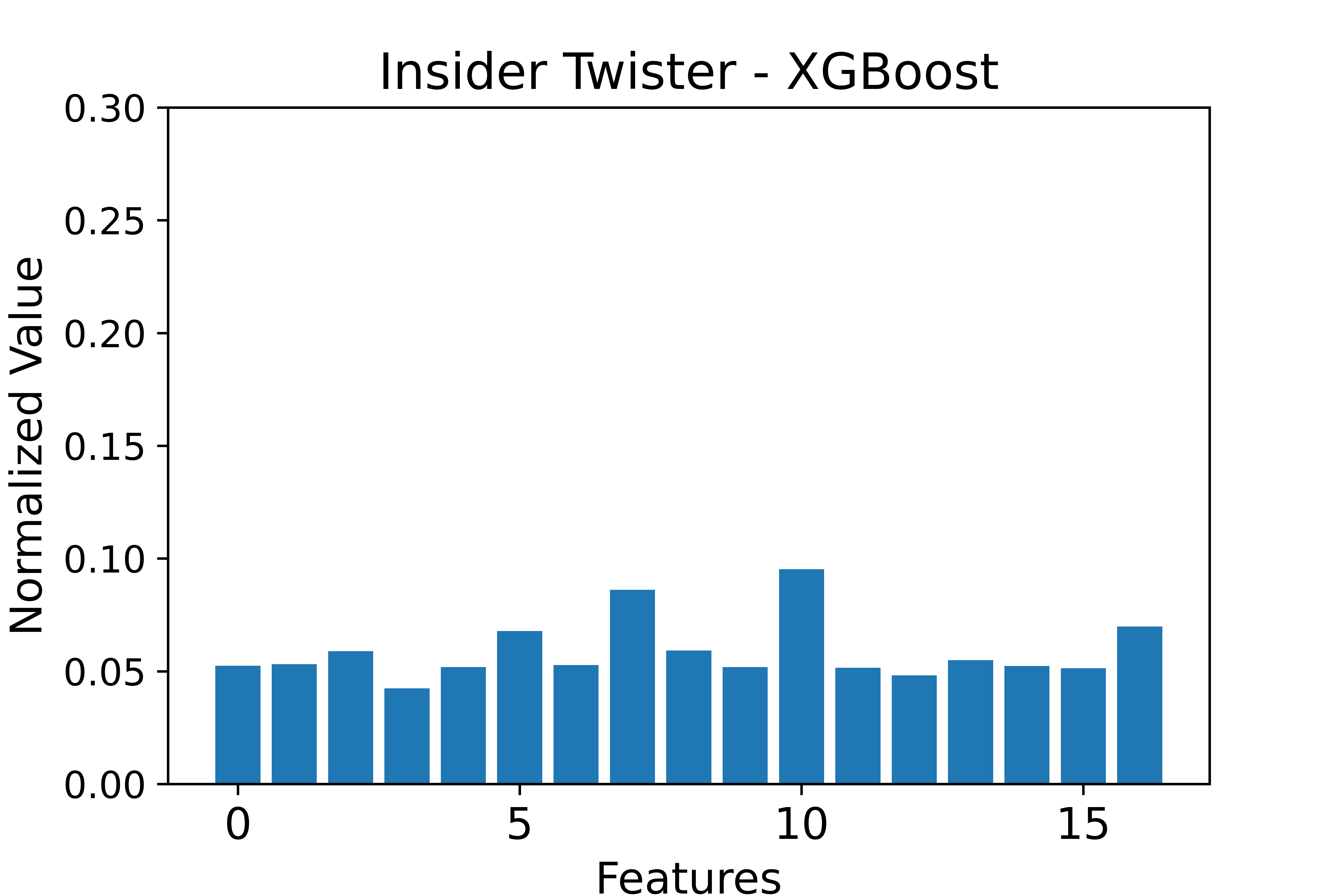}
%
\end{tabular}
\end{center}
\caption{Normalized feature importance profiles for \pmboost~with $\alpha = 1.1$ and $a_{f} = 7$ (\textit{top}) and for XGBoost (\textit{bottom}) on the online shoppers intention dataset (both for depth $3$ trees) with and without the insider twister. 
We find that the insider twister significantly perturbs the feature importance of XGBoost as evidenced in the plot (\textit{far right}), and hence significantly reduces the inferential capacity of the learned model.
More details can be found in Insider Twister (main body) and Appendix~\ref{sec:additionalinsidertwister}.
}
  \label{fig:onlineshoppersinsidertwisterfeatures}
\end{figure}

We provide experimental results on \pmboost~(for $\alpha \in \{1.1, 2, 4\}$) and compare with AdaBoost~\cite{freund1997decision} and XGBoost~\cite{chen2016xgboost} on four binary classification datasets, namely, cancer~\cite{breastcancerdataset}, xd6~\cite{buntine1992further}, diabetes~\cite{smith1988using}, and online shoppers intention~\cite{sakar2019real}.
We performed $10$ runs per algorithm with randomization over the train/test split and the twisters.
All experiments use regression decision trees (of varying depths $1$-$3$) in order to align with XGBoost.
Hyperparameters of XGBoost were kept to default to maintain the fairest comparison between the three algorithms; for more of these experimental details, please refer to Appendix~\ref{sec:parametersxgboost}. 
In order to demonstrate the \textit{twist-properness} of $\alpha$-loss as implemented in \pmboost, we corrupt the training examples of these datasets according to three different (malicious) twisters.\\ 
\noindent \textbf{Class Noise Twister} (all datasets): This twister is equivalent to SLN in the training sample. 
Results on this twister for the cancer dataset are presented in Figure~\ref{fig:breastcancerlabelnoise} and see Appendix~\ref{sec:additionalclassnoise} for further results. 
In general, we find that \pmboost~is more robust to the Class Noise Twister than AdaBoost and XGBoost, and we find that $\alpha^{*}$ increases as the amount of twist increases, which complies with our theory (Lemma~\ref{lemPointwise} and Theorem~\ref{thmALPHABLUNTING}).
We also present an \textit{adaptive $\alpha$ experiment} in Figure~\ref{fig:ConfusionMatrix1}.
We denote the adaptive method Menon \pmboost, since we take inspiration from~\cite{mrowLF}, where they show that one can estimate the level of label noise (see their Appendix D.1) from the minimum and maximum posterior values. 
Using a single decision tree classifier with $O(\log(m))$ leaves and $O(\sqrt{m})$ samples per leaf ($m \approx 681$ examples for xd6 dataset with $70/30$ train/test-split), and information gain as the splitting criterion, we estimate the minimum and maximum posterior values directly from the training data with local counts of number of samples classified such that $Y = 1$ at each leaf. 
Once we obtain $\posmin$ and $\posmax$ in this way, we estimate the symmetric noise value $p \in [0,1]$ with the geometric mean $p = \sqrt{\posmin(1-\posmax)}$.
Finally, to estimate $\alpha_{0}$ for each noise level, we apply the formula in Lemma~\ref{lemPointwise}\textbf{(c)} and the SLN formula given just before Lemma~\ref{lemSymLabFlip} where we estimate $\posclean$ with the average posterior from the decision tree classifier. 

\noindent \textbf{Feature Noise Twister} (xd6 dataset): This twister perturbs the training sample by randomly flipping features. More precisely, for each training example, the example is selected if $\text{Ber}(p_{1})$ returns $1$. Then, for each selected training example, and for each feature independently, the feature is flipped (the features of xd6 are Booleans) to the other symbol if $\text{Ber}(p_{2})$ also returns $1$. 
Results on this twister are presented in Table~\ref{table:xd6featurenoise} where $p_{1} = p_{2} = p$.
In general, we find that \pmboost~is more robust to the Feature Noise Twister than AdaBoost and XGBoost, and we find that $\alpha^{*}$ increases as the amount of twist increases.

\noindent \textbf{Insider Twister} (online shoppers intention dataset): This twister assumes more knowledge about the model than the previous two twisters. 
In essence, the insider twister adds noise to a few of the most informative features for predicting the class.
Specifically for the online shoppers intention dataset, the insider twister adds noise to feature $8$ (\textit{page values} - numeric type with range in $[-250,435]$), feature $10$ (\textit{month}), and feature $15$ (\textit{visitor type} - ternary alphabet).
For \textit{page values}, the insider twister adds i.i.d. $\mathcal{N}(0,60)$ to the entries; for both \textit{month} and \textit{visitor type}, the insider twister independently increments (with probability $1/2$) the symbol according to their respective alphabets such that about $50\%$ of each of these features are perturbed.
Results on this twister are presented in Figure~\ref{fig:onlineshoppersinsidertwisterfeatures} and further discussion in Appendix~\ref{sec:additionalinsidertwister} (Figure~\ref{fig:averagedresultsinsidertwister}); post-twister, the feature importance profile of XGBoost is almost uniform, displaying damages to the algorithm's discriminative abilities (Figure~\ref{fig:onlineshoppersinsidertwisterfeatures}, right), while the feature importance profile of \pmboost~is much less perturbed overall.

\section{Conclusion}\label{sec-conc}

In this work, we have introduced a generalization of the properness framework via the notion of \textit{twist-properness}, which allows delineating loss functions with the ability to ``untwist'' the twisted posterior into the clean posterior.
Notably, we have shown in Lemma~\ref{lemPointwise} that a nontrivial extension of a loss function called $\alpha$-loss, first introduced in information theory, is twist-proper.
For a large class of twists (Definition~\ref{def:Bayesblunting}), we have shown in Theorem~\ref{thmALPHABLUNTING} that while a pointwise estimation of $\alpha$ is optimal, a fixed choice of $\alpha_{0}>1$ readily outperforms the proper choice ($\log$-loss).
We then studied the twist-proper $\alpha$-loss under a novel boosting algorithm, called \pmboost, which uses the pseudo-inverse link for weight updates, for which we proved convergence guarantees in Theorem~\ref{genBOOST}.
Lastly, we have presented experimental results for \pmboost~optimizing $\alpha$-loss on several twists, and
also a method inspired by~\cite{mrowLF} to estimate $\alpha_{0}$ using \textit{only} training data in Figure~\ref{fig:ConfusionMatrix1}. 
A possible next step for this line of work would be to design losses that are both twist-proper \textit{and} algorithmically convenient.

\bibliographystyle{plain}
\bibliography{bibgen}

\begin{appendices}

\section{Proofs and Further Theoretical Results}

\subsection{Proof of Lemma \ref{lemUstar}}\label{app-proof-lemUstar}

We study $U \defeq (-\poibayesrisk)^\star$, which is convex by definition, and show that it is non-decreasing.
Monotonicity follows from the non-negativity of the argument of the partial losses and the definition of the convex conjugate: suppose $z'\geq z$
 and let $u^* \in \arg\sup_u z u +\poibayesrisk(u)$. We have
\begin{eqnarray}
  U(z') & \defeq & \sup_{u\in [0,1]} z'u +\poibayesrisk(u) \\
  & = & \sup_{u\in [0,1]} (z'-z)u + zu +\poibayesrisk(u) \\
        & \geq & (z'-z)u^* + zu^* +\poibayesrisk(u^*) \\
        & = & (z'-z)u^* + U(z)\\
  & \geq & U(z),
\end{eqnarray}
which completes the proof that $U$ is non-decreasing and therefore $F(z)\defeq U(-z)$ non-increasing.

Concavity of $\poibayesrisk$ follows from definition. We show continuity of $\poibayesrisk$, the continuity of $F$ then following from the definition of the convex conjugate $F$ \cite{bvCO}. Let $a, u \in (0,1)$, let $u^* \in \pseudob(u), a^*\in \pseudob(a)$. We get:
    \begin{eqnarray}
      \poibayesrisk(u) & \defeq & u \partialloss{1}(u^*) + (1-u) \partialloss{-1}(u^*)\\
                      & \leq & u \partialloss{1}(a^*) + (1-u) \partialloss{-1}(a^*)\\
      & & = \poibayesrisk(a) + (u-a) ( \partialloss{1}(a^*) - \partialloss{-1}(a^*)),
    \end{eqnarray}
    (the inequality holds since otherwise $u^* \not\in \pseudob(u)$) Permuting the roles of $u$ and $a$, we also get
    \begin{eqnarray}
      \poibayesrisk(a) & \leq & \poibayesrisk(u) + (a-u) ( \partialloss{1}(u^*) - \partialloss{-1}(u^*)),
      \end{eqnarray}
      from which we get 
      \begin{eqnarray}
        |\poibayesrisk(a)-\poibayesrisk(u)| & \leq & Z \cdot |a-u|,\label{contua}
      \end{eqnarray}
      with $Z\defeq \max_{v\in \{a,u\}} \sup |\partialloss{1}(\pseudob(v))-\partialloss{-1}(\pseudob(v))|$ (where we use set differences if $\pseudob$s are not singletons). Since $Z\ll \infty$, \eqref{contua} is enough to show the continuity of $\poibayesrisk$ (we have by assumption $\mathrm{dom}(\poibayesrisk) = [0,1]$). 
      
      \subsection{Bayes Tilted Estimates}\label{app-proof-bayestilt}

The proof of Lemma~\ref{lemma:btesingleton} readily follows from Definition~\ref{defBTE} and standard properties of convex functions, see e.g.,~\cite{bvCO}.

Below, we also provide analysis of the properties of Bayes tilted estimates for more general losses which induce set-valued functions. 
Following convention, we denote the set valued inequality $A \leq B$, such that, $\forall a\in A, \exists b \in B, a\leq b$ and the set-valued (Minkowski) difference $A-B \defeq \{a-b: a\in A, b\in B\}$.

  \begin{lemma}\label{proptilt}
    The following properties of $\pseudob$ follow from assumptions M, D or S on partial losses:\\
    \noindent \textbf{(M)} implies set-valued monotonicity: $\forall u_1<u_3\in [0,1]$, we have $\pseudob(u_1) \leq  \pseudob(u_3)$ and $\pseudob(u_1) \cap \pseudob(u_3) \subseteq \pseudob(u_2), \forall u_2 \in (u_1, u_3)$;\\
  \noindent \textbf{(D)} and $\pseudob$ differentiable imply monotonicity: $\forall u \in [0,1]$, $\partialloss{1}'(\pseudob(u)) \leq \partialloss{-1}'(\pseudob(u)) \Leftrightarrow  \pseudob'(u) \geq 0$;\\
    \noindent  \textbf{(S)} implies set-valued symmetry: $\pseudob(1-u) = \{1\} - \pseudob(u), \forall u\in [0,1]$; \\
    \textbf{(E)} Extreme values: $\partialloss{1}(1) = \partialloss{-1}(0) = 0$, $\partialloss{1}([0,1]) \subseteq \mathbb{R}_+, \partialloss{-1}([0,1]) \subseteq \mathbb{R}_+$. Further, this implies properness on extreme values, as $0 \in \pseudob(0), 1 \in \pseudob(1)$.
\end{lemma} 


\noindent \textbf{Case (M)} -- Suppose $\pseudob(a) \cap \pseudob(a') \neq \emptyset$ for some $a\neq a'$ and let $v^* \in \pseudob(a) \cap \pseudob(a')$. 
It means $\forall v\in [0,1]$,
    \begin{eqnarray}
a \partialloss{1}(v^*) + (1-a) \partialloss{-1}(v^*) & \leq & a \partialloss{1}(v) + (1-a) \partialloss{-1}(v),\label{eq000}\\
a' \partialloss{1}(v^*) + (1-a') \partialloss{-1}(v^*) & \leq & a' \partialloss{1}(v) + (1-a') \partialloss{-1}(v),\label{eq001}
    \end{eqnarray}
    and so $\forall \delta \in [0,1]$, if we let $a_\delta \defeq a + \delta (a'-a)$, a $1-\delta, \delta$ convex combination of both inequalities yields $\forall v\in [0,1]$,
    \begin{eqnarray}
a_\delta \partialloss{1}(v^*) + (1-a_\delta) \partialloss{-1}(v^*) & \leq &  a_\delta  \partialloss{1}(v) + (1-a_\delta ) \partialloss{-1}(v), \forall v\in [0,1],
    \end{eqnarray}
    which implies $v^* \in \pseudob(a_\delta)$ and shows the right part of \textbf{Case (M)}.\\
To show show the left part of \textbf{Case (M)}; we add to \eqref{eq000} and \eqref{eq001} we now add the inequality:
    \begin{eqnarray}
a \partialloss{1}(v^\circ) + (1-a) \partialloss{-1}(v^\circ) & \leq & a \partialloss{1}(v) + (1-a) \partialloss{-1}(v),\label{eq002}
    \end{eqnarray}
    with therefore $v^\circ\in \pseudob(a)$, implying $a \partialloss{1}(v^\circ) + (1-a) \partialloss{-1}(v^\circ) = a \partialloss{1}(v^*) + (1-a) \partialloss{-1}(v^*)$ as otherwise one of $v^\circ, v^*$ would not be in $\pseudob(a)$. We then get
    \begin{eqnarray}
      a' \partialloss{1}(v^\circ) + (1-a') \partialloss{-1}(v^\circ) & = & a \partialloss{1}(v^\circ) + (1-a) \partialloss{-1}(v^\circ) + (a'-a)\cdot (\partialloss{1}(v^\circ) -\partialloss{-1}(v^\circ) )\nonumber\\
                                                                     & = & a \partialloss{1}(v^*) + (1-a) \partialloss{-1}(v^*) + (a'-a)\cdot (\partialloss{1}(v^\circ) -\partialloss{-1}(v^\circ) ) \nonumber\\
        & = & a' \partialloss{1}(v^*) + (1-a') \partialloss{-1}(v^*) + (a'-a)\cdot \Delta,\label{eq003}
    \end{eqnarray}
    with $\Delta \defeq \partialloss{1}(v^\circ) -\partialloss{-1}(v^\circ) - (\partialloss{1}(v^*) - \partialloss{-1}(v^*))$. Considering \eqref{eq003}, we deduce from \eqref{eq001} that to have $v^\circ \in \pseudob(a')$, we equivalently need $(a'-a)\cdot \Delta \leq 0$. We also know by assumption that $\partialloss{1}$ is non-increasing and $\partialloss{-1}$ is non-decreasing, so $g(u) \defeq \partialloss{1}(u) -\partialloss{-1}(u)$ is non-increasing. We thus have $(a'-a)\cdot \Delta \leq 0$ iff one of the two possibilities hold:
    \begin{itemize}
\item $a' \geq a$ and $v^\circ \geq v^*$, or
\item $a' \leq a$ and $v^\circ \leq v^*$,
\end{itemize}
which shows the right part of \textbf{Case (M)}.\\

\noindent \textbf{Case (E)} -- we have $\poibayesrisk(0) = \inf_{v\in[0,1]} \partialloss{-1}(v) = 0$ for $v=0$, hence $0 \in \pseudob(0)$. Similarly, $\poibayesrisk(1) = \inf_{v\in[0,1]} \partialloss{1}(v) = 0$ for $v=1$, hence $1 \in \pseudob(1)$.\\

\noindent \textbf{Case (D)} -- we have
    \begin{eqnarray}
      \frac{\mathrm{d} }{\mathrm{d} u}  \poibayesrisk(u) & = & \partialloss{1}(\pseudob(u)) + u \partialloss{1}'(\pseudob(u)) \pseudob'(u) - \partialloss{-1}(\pseudob(u)) +  (1-u) \partialloss{-1}'(\pseudob(u)) \pseudob'(u) \nonumber\\
      & = & \partialloss{1}(\pseudob(u)) - \partialloss{-1}(\pseudob(u)) + \pseudob'(u)\cdot ( u \partialloss{1}'(\pseudob(u)) +  (1-u) \partialloss{-1}'(\pseudob(u)))  \label{eqDERL},
    \end{eqnarray}
    but since $v = \pseudob(u)$ is the solution to \eqref{eq000} it satisfies $u \partialloss{1}'(\pseudob(u)) + (1-u) \partialloss{-1}'(\pseudob(u)) = 0$ 
    , so that \eqref{eqDERL} simplifies to
    \begin{eqnarray}
\frac{\mathrm{d} }{\mathrm{d} u}  \poibayesrisk(u) & = & \partialloss{1}(\pseudob(u)) - \partialloss{-1}(\pseudob(u)),
    \end{eqnarray}
    and since $\poibayesrisk$ is concave and the partial losses are differentiable,
    \begin{eqnarray}
\frac{\mathrm{d}^2 }{\mathrm{d} u^2}  \poibayesrisk(u) & = & \pseudob'(u) \cdot (\partialloss{1}'(\pseudob(u)) - \partialloss{-1}'(\pseudob(u))) \leq 0, \forall u \label{eqDERL2},
    \end{eqnarray}
    which proves the statement of the Lemma.\\
    
\noindent \textbf{Case (S)} -- Suppose $v^* \in \pseudob(a)$, which implies
    \begin{eqnarray}
      a \partialloss{1}(v^*) + (1-a) \partialloss{-1}(v^*) & \leq & a \partialloss{1}(v) + (1-a) \partialloss{-1}(v), \forall v\in[0,1].\label{eq00l}
    \end{eqnarray}
    We also note that since symmetry holds, $a \partialloss{1}(v^*) + (1-a) \partialloss{-1}(v^*) = (1-a) \partialloss{1}(1-v^*) + a \partialloss{-1}(1-v^*)$, which implies because of \eqref{eq00l} $1-v^* \in \pseudob(1-a)$. \\

\noindent \textbf{Remark}: even if we assume the partial losses to be strictly monotonic, the tilted estimate can still be set valued. To see this, craft the partial losses such that $v \in \pseudob(u)$ and then for some $w> v$, replace the partial losses in the interval
  $[v, w]$ by affine parts w/ slope $-a < 0$ for $\partialloss{1}$, $b>0$ for $\partialloss{-1}$ and such that $b/a = u/(1-u)$ which guarantees $ L(u,v) =  L(u,w)$ and thus $w \in \pseudob(u)$;

\subsection{Proof of Lemma \ref{lemFocalLoss}}\label{app-proof-lemFocalLoss}

  We recall the focal loss' corresponding pointwise conditional risk in lieu of \eqref{eqpoirisk}:
  \begin{eqnarray}
\poirisk_\gamma(u,v) & \defeq & - v\cdot (1-u)^\gamma \log u - (1-v) \cdot u^\gamma \log(1-u),
  \end{eqnarray}
  and if it is twist proper, then for any $\posnoise, \posclean \in [0,1]$, there
  exists $\gamma \geq 0$ such that
  \begin{eqnarray}
\left. \frac{\partial}{\partial u} \poirisk_\gamma(u,\posnoise)\right|_{u = \posclean} & = & 0.
    \end{eqnarray}
    Equivalently, we must find $\gamma \geq 0$ such that (keeping notations $ u \defeq \posclean, v \defeq \posnoise$ for clarity):
    \begin{eqnarray}
(1-v) u^\gamma \cdot\left(\gamma (1-u)\log(1-u) -u\right) & = & v (1-u)^\gamma \cdot\left(\gamma u \log u + u - 1\right),
    \end{eqnarray}
    and we see that twist properness implies the statement that for any $K\geq 0$ (note that $K = \frac{v}{1-v}$) and any $u\in [0,1)$, there exists $\gamma \geq 0$ such that
    \begin{eqnarray}
      \frac{f(u, \gamma)}{f(1-u, \gamma)} & = & K,\label{eqPropF}\\
      f(u, \gamma) & \defeq & u^\gamma \cdot\left(\gamma (1-u)\log(1-u) -u\right).
    \end{eqnarray}
We study the ratio for $u\in [0,1/2]$. We have $f(u, \gamma) \leq 0, \forall u \in [0,1], \forall \gamma \geq 0$ and
    \begin{eqnarray}
      \frac{\partial}{\partial \gamma} f(u, \gamma) & = & u^{\gamma}(\gamma \cdot a(u) - b(u)),
    \end{eqnarray}
    with $a(u) \defeq (1-u) \log(1-u)\log(u) \geq 0$ and $b(u) \defeq u \log u - (1-u)\log(1-u)$, satisfying $b(1-u) = -b(u)$ and $u a(1-u) = (1-u)a(u)$. Hence, we arrive after some derivations to
    \begin{eqnarray}
      \frac{\partial}{\partial \gamma} \frac{ f(u, \gamma)}{f(1-u, \gamma)} & = & \frac{(u(1-u))^\gamma}{f^2(1-u, \gamma)}\cdot \left(A(u) \gamma^2 + B(u)\gamma + C(u)\right),\\
      A(u) & \defeq & u(1-u) \log(u) \log(1-u) \log(u(1-u)),\\
      B(u) & \defeq & -(u^2\log^2 u + (1-u)^2\log^2 (1-u) + (1-2u)^2\log u \log(1-u)),\\
      C(u) & \defeq & (1-2u) b(u).
    \end{eqnarray}
    All functions $A, B, C$ are non positive for any fixed $u \in [0,1/2]$, so the ratio in \eqref{eqPropF} is non-increasing over $\gamma \geq 0$ and as a consequence, for any fixed $u \in [0,1/2]$,
    \begin{eqnarray}
      \frac{f(u, \gamma)}{f(1-u, \gamma)} & \leq & \frac{f(u, 0)}{f(1-u, 0)}\\
      & & = \frac{u}{1-u}, \forall \gamma \geq 0,
    \end{eqnarray}
so we see that \eqref{eqPropF} cannot be satisfied when $K > u/(1-u)$ and as a consequence, the focal loss is not twist-proper.

\noindent \textbf{Twist-improperness of the Super Loss}
The Super Loss \cite{cwrSA} works as a ``wrapper'' of a loss, its
partial losses being defined as 
\begin{align}
L_{b,\lambda}(\ell,\sigma_{i}) = (\partialloss{b} - \tau)\sigma_{i} + \lambda(\log{\sigma_{i}})^{2},
\end{align}
where $b\in \{-1,1\}$ indicates the partial loss of a loss of interest, $\tau \in \mathrm{Im} \partialloss{b}, \lambda > 0$ are user-defined parameters. $\sigma_{i}$ is a functional computed to minimize the partial losses, and we get the optimal expression:
\begin{align}
\sigma^{*}(\partialloss{b}) = \exp\left({-W(1/2 \max{(-2/e,\beta)})}\right),
\end{align}
with $\beta = \frac{\partialloss{b} - \tau}{\lambda}$ (notice this is also a function via the partial loss). $W$ is called Lambert's function. It does not have an analytical form. 
\begin{lemma}
Suppose loss $\ell$ in the Super Loss is such that its partial loss $\partialloss{1}$ is strictly decreasing and $\partialloss{-1}$ is strictly increasing. Then the corresponding Super Loss with partial losses $L_{b,\lambda}(\ell,\sigma^*(\partialloss{b}))$ ($b\in \{-1,1\}$) is not twist proper.
\end{lemma}
\noindent\textbf{Remark}: the assumptions about the partial losses are very weak and would be satisfied by all popular losses (e.g. log, square, Matusita, etc.).
\begin{proof}
The notable facts about $W$, useful for our proof are:
\begin{align}
e^{W(z)} = \frac{z}{W(z)} \quad \text{and} \quad \dfrac{d}{dz} W(z) = \frac{1}{z + e^{W(z)}}  \quad \text{and} \quad  \sup \exp (-W(z)) = e.\label{eq-w-prop}
\end{align}
Simplifying notations above, we end up studing a loss with partial losses defined as
\begin{align}
L^*_{\lambda}(\partialloss{b}) = (\partialloss{b} - \tau)\sigma_{i} + \lambda(\log{\sigma^*(\partialloss{b})})^{2}.
\end{align}
Recall that a loss $\ell$ is twist-proper iff for any twist, there exists hyperparameters such that $\posclean \in t_{\ell}(\posnoise)$.
Examining this for the Super Loss, we obtain
\begin{align}
t_{L}(v) &= \text{arginf}_{u \in [0,1]} L(u,v) \\
&= \text{arginf}_{u \in [0,1]} v \cdot L_{\lambda}^{*}(\ell_{1}(u)) + (1-v)\cdot L_{\lambda}^{*}(\ell_{-1}(u)) \\
  \label{eq:superloss0} &= \text{arginf}_{u \in [0,1]}
                          \left\{\begin{array}{l}
                                   v \cdot \left[(\ell_{1}(u) - \tau)\sigma^*(\partialloss{1}) + \lambda (\log{\sigma^*(\partialloss{1})})^{2}\right] \\
                                   + (1-v)\cdot \left[(\ell_{-1}(u) - \tau)\sigma^*(\partialloss{-1}) + \lambda (\log{\sigma^*(\partialloss{-1})})^{2} \right]
                                   \end{array}\right.
\end{align}
We note that if $\ell$ is proper, then $v \in t_{\ell}(v)$.
Computing the minimum in~\eqref{eq:superloss0}, we obtain
\begin{align}
0 = \dfrac{d}{du} v \cdot \left[(\ell_{1}(u) - \tau)\sigma^*(\partialloss{1}) + \lambda (\log{\sigma^*(\partialloss{1})})^{2}\right] + (1-v)\cdot \left[(\ell_{-1}(u) - \tau)\sigma^*(\partialloss{-1}) + \lambda (\log{\sigma^*(\partialloss{-1})})^{2} \right].
\end{align}
Similar to the computation of the focal loss, we need
\begin{align}
\frac{\dfrac{d}{du} (\ell_{1}(u) - \tau)\sigma^*(\partialloss{1}) + \lambda (\log{\sigma^*(\partialloss{1})})^{2}}{\dfrac{d}{du}(\ell_{-1}(u) - \tau)\sigma^*(\partialloss{-1}) + \lambda (\log{\sigma^*(\partialloss{-1})})^{2} } = - \frac{(1-v)}{v} = -K,\label{eq-pfl}
\end{align}
and this needs to hold (via the choice of parameters $\tau, \lambda$) for any $u \in [0,1)$ and $K>0$. To save notations, define
\begin{eqnarray*}
\beta_b(u) & = & \frac{\partialloss{b}(u) - \tau}{2\lambda}.
\end{eqnarray*}
Remark that \textit{if} $\partialloss{b}(u) > \tau - (2\lambda)/e$, we have $\sigma^*(\partialloss{b}(u)) = \exp(-W(\beta_b(u)))$ and so
\begin{eqnarray}
  \lefteqn{\dfrac{d}{du} (\partialloss{b}(u) - \tau)\sigma^*(\partialloss{b}(u)) + \lambda (\log{\sigma^*(\partialloss{b}(u))})^{2}}\nonumber\\
  & = & 2\lambda \cdot \dfrac{d}{du}  \left[\beta_b(u) \exp(-W(\beta_b(u))) + \frac{W^2(\beta_b(u))}{2}\right]\nonumber\\
  & = &  2\lambda \cdot \left[\beta'_b(u) \exp(-W(\beta_b(u))) - \beta_b(u) \beta'_b(u) \cdot \frac{\exp(-W(\beta_b(u)))}{\beta_b(u)+\exp(W(\beta_b(u)))} + \frac{\beta'_b(u)W(\beta_b(u))}{\beta_b(u)+\exp(W(\beta_b(u)))}\right]\nonumber\\
  & = & 2\lambda \beta'_b(u) \cdot \frac{1 + \beta_b(u)\exp(-W(\beta_b(u)))-\beta_b(u)\exp(-W(\beta_b(u)))+W(\beta_b(u))}{\beta_b(u)+\exp(W(\beta_b(u)))}\nonumber\\
  & = & \partialloss{b}'(u)\cdot \frac{1 + W(\beta_b(u))}{\beta_b(u)+\exp(W(\beta_b(u)))}\\
  & = & \partialloss{b}'(u)\cdot \exp(-W(\beta_b(u))),\label{simpl-eq-sl}
\end{eqnarray}
since indeed it comes from \eqref{eq-w-prop},
\begin{eqnarray}
\frac{1 + W(z))}{z+\exp(W(z))} & = & \exp(-W(z));
  \end{eqnarray}
also, \textit{if} $\partialloss{b}(u) \leq \tau - (2\lambda)/e$, we have $\sigma^*(\partialloss{b}(u)) = \exp(-W(-1/e)) = e$ and so
  \begin{eqnarray}
  \dfrac{d}{du} (\partialloss{b}(u) - \tau)\sigma^*(\partialloss{b}(u)) + \lambda (\log{\sigma^*(\partialloss{b}(u))})^{2} & = & \partialloss{b}'(u)\cdot e. \label{simpl-eq-sl2}
  \end{eqnarray}
Since $\lim_{z \rightarrow -1/e^+} \exp(-W(z)) = e$, we can summarize both \eqref{simpl-eq-sl} and \eqref{simpl-eq-sl2} as
  \begin{eqnarray}
\dfrac{d}{du} (\partialloss{b}(u) - \tau)\sigma^*(\partialloss{b}(u)) + \lambda (\log{\sigma^*(\partialloss{b}(u))})^{2} & = & \partialloss{b}'(u)\cdot \exp(-W(\max\{-1/e, \beta_b(u)\})).
  \end{eqnarray}
Now consider a loss $\loss$ satisfying $\partialloss{-1}$ strictly increasing and $\partialloss{1}$ strictly decreasing. Pick $u$ so that we have simultaneously
  \begin{eqnarray}
    \partialloss{1}(u) & > & \tau - \frac{2\lambda}{e},\\
    \partialloss{-1}(u) & \leq & \tau - \frac{2\lambda}{e},
  \end{eqnarray}
  which, assuming both inequalities fit in the range of the respective partial loss, that $u \in [0,\gamma]$ for some $\gamma > 0$. Rewriting \eqref{eq-pfl}, we need to show that for any such $\gamma > 0$ and $u \in [0,\gamma]$ and $K>0$, there exists a choice of $\tau, \lambda$ such that
  \begin{eqnarray}
\frac{\partialloss{1}'(u)\cdot \exp(-W(\beta_1(u)))}{\partialloss{-1}'(u)\cdot e} & = & -K,
  \end{eqnarray}
  which rewrites conveniently as
  \begin{eqnarray}
    \exp(-W(\beta_1(u))) & = & - K e \cdot \frac{\partialloss{-1}'(u)}{\partialloss{1}'(u)},\label{eq-b-kprime}
  \end{eqnarray}
  or,
  \begin{eqnarray}
    \exp(-W(\beta_1(u))) & = & K',\label{eqexpW}
  \end{eqnarray}
for any $K' > 0$ (the RHS of \eqref{eq-b-kprime} is indeed always strictly positive). But $\sup \exp (-W(z)) = e$ \eqref{eq-w-prop}, so \eqref{eqexpW} cannot hold and the Super Loss is not twist proper.
  \end{proof}

\subsection{Proof of Lemma \ref{lemPointwise}}\label{app-proof-lemPointwise}

\begin{figure}[h] 
    \centerline{\includegraphics[scale=.5]{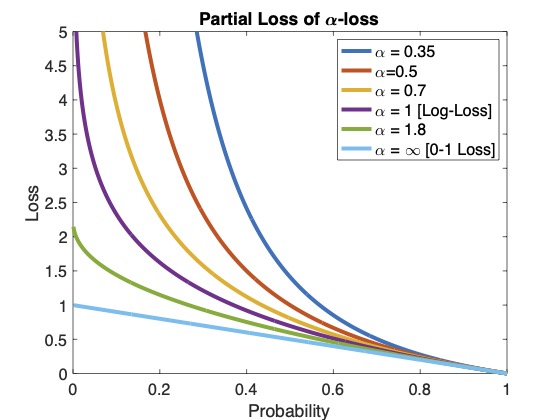}}
    \caption{A plot of $\ell_{1}^{\alpha}(u)$ for $\alpha > 0$ as given in Definition~\ref{defALPHALOSS}.}
    \label{fig:alphalossplot}
\end{figure}

For part \textbf{(a)}: we cite~\cite{sypherd2019journal} which demonstrates \textbf{(M)}, \textbf{(D)}, \textbf{(S)} for $\alpha > 0$. 
With our extension of $\alpha$-loss, these can also be readily shown for $\alpha < 0$, since they are mapped back to the $\alpha > 0$ losses.

For part \textbf{(b)}: we know from Lemma~\ref{lemma:btesingleton} that $\alpha$-loss, for $\alpha \in \mathbb{R} \setminus \{0,\pm \infty\}$, due to strict convexity, returns a singleton, i.e., $|t_{\ell^{\alpha}}(\posnoise)| = 1$.
With regards to that singleton, we know from~\cite{sypherd2019journal,lkspAT} for $\alpha > 0$ that $t_{\ell^{\alpha}}(\posnoise) = \frac{\posnoise^{\alpha}}{\posnoise^{\alpha} + (1-\posnoise)^{\alpha}}$. 
A very similar calculation recovers $t_{\ell^{\alpha}}(\posnoise) = \frac{\posnoise^{-\alpha}}{\posnoise^{-\alpha} + (1-\posnoise)^{-\alpha}}$ for $\alpha < 0$.
Multiplying the numerator and denominator of this expression by $(1-\posnoise)^{\alpha}$, we can simply write both expressions using $\frac{\posnoise^{\alpha}}{\posnoise^{\alpha} + (1-\posnoise)^{\alpha}}$.
Regarding the limit as $\alpha \rightarrow \pm \infty$ yielding $t_{\ell^{\pm \infty}}(\posnoise) = \pm 1$ or $\mp 1$, this was also already shown by~\cite{sypherd2019journal} for $+\infty$ and is similarly (readily) extended for the $\alpha \rightarrow -\infty$ case.

For part \textbf{(c)}: here, we break entirely new ground. Let $\alpha > 0$. To obtain twist-properness as stipulated in Definition~\ref{def:TwistProper}, we seek to know for what $\alpha$ the following holds 
\begin{align}
\posclean = \frac{\posnoise^{\alpha}}{\posnoise^{\alpha} + (1-\posnoise)^{\alpha}}. 
\end{align}
Solving for $\alpha$, we obtain
\begin{align}
\posclean &= \frac{\posnoise^{\alpha}}{\posnoise^{\alpha} + (1-\posnoise)^{\alpha}} \\
\frac{1}{\posclean} &= 1 + \left(\frac{1}{\posnoise} - 1\right)^{\alpha} \\
\frac{1}{\posclean} - 1 &= \left(\frac{1}{\posnoise} - 1\right)^{\alpha} \\
\log{\left(\frac{1}{\posclean} - 1 \right)} &= \alpha \log{\left(\frac{1}{\posnoise} - 1 \right)} \\
\alpha^{*} &= \frac{\log{\left(\frac{1 - \posclean}{\posclean} \right)}}{\log{\left(\frac{1- \posnoise}{\posnoise}\right)}}. 
\end{align}
After multiplying the numerator and denominator by $-1$, we obtain the desired result. 
Namely, $\alpha$-loss is twist-proper for 
\begin{align} \label{eq:properalpha}
\alpha^{*} = \frac{\logit(\posclean)}{\logit{(\posnoise)}}.
\end{align}

\begin{figure}[h] 
    \centerline{\includegraphics[scale=.25]{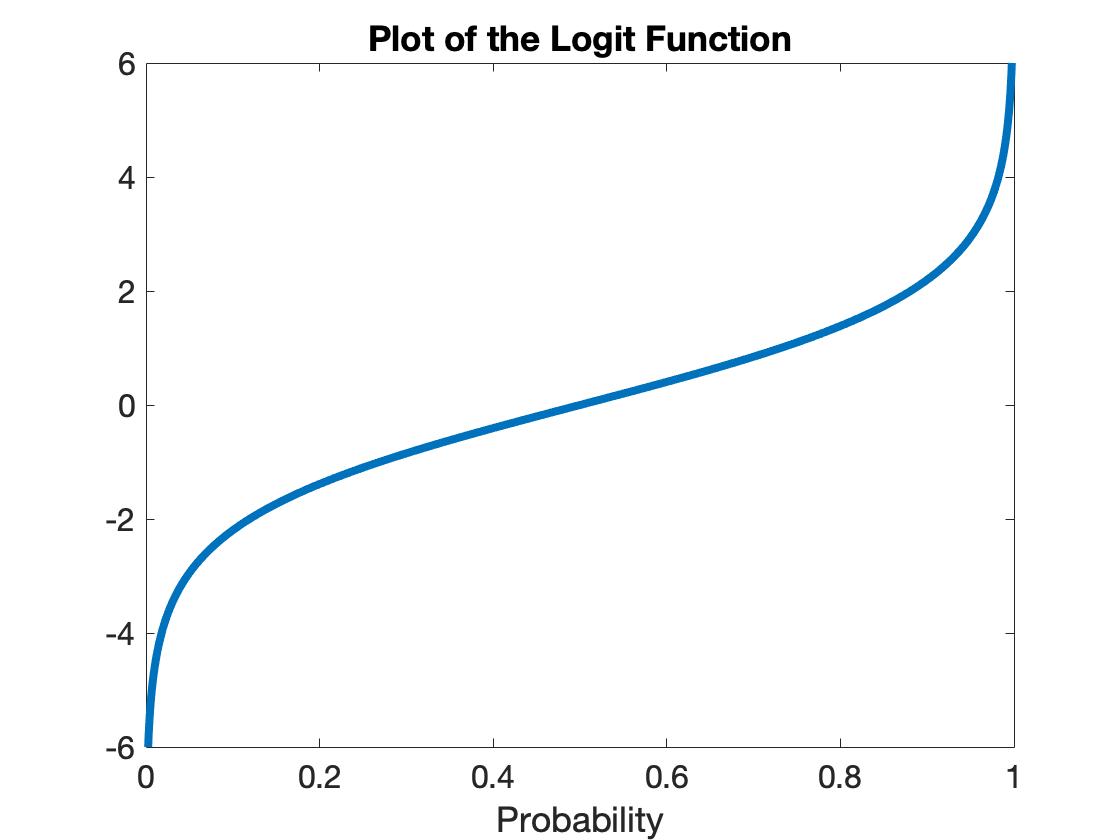}}
    \caption{A plot of the logit $\logit(u) \defeq \log(u/(1-u))$.}
    \label{fig:logitplot}
\end{figure}

For part \textbf{(d)}: we recall the definition of Bayes blunting twist from Definition~\ref{def:Bayesblunting}: 
a twist $\posclean \mapsto \posnoise$ is Bayes blunting iff  $(\posclean \leq \posnoise \leq 1/2) \vee (\posclean \geq \posnoise \geq 1/2)$.
Also, recall that $\alpha^{*}$ is given in~\eqref{eq:properalpha}, and see Figure~\ref{fig:logitplot} for a plot of the logit function.
Let $\posclean \geq 1/2$. 
The Bayes blunting twist can take $\posnoise$ from $\posclean \geq \posnoise \geq 1/2$. 
If $\posnoise = \posclean$, then $\alpha^{*} = 1$.
If $\posnoise \rightarrow 1/2$, as can be seen in the figure, the sign crossover point is $1/2$, so $\alpha^{*} \rightarrow \infty$.
Thus, by continuity, we have that $\alpha^{*} \geq 1$.
Finally, the case where $\posclean < 1/2$ follows, \textit{mutatis mutandis}.

\subsection{Proof of Theorem \ref{thmALPHABLUNTING}}\label{app-proof-thmALPHABLUNTING}

\begin{figure}[h]
    \centering
    \centerline{\includegraphics[width=.75\linewidth]{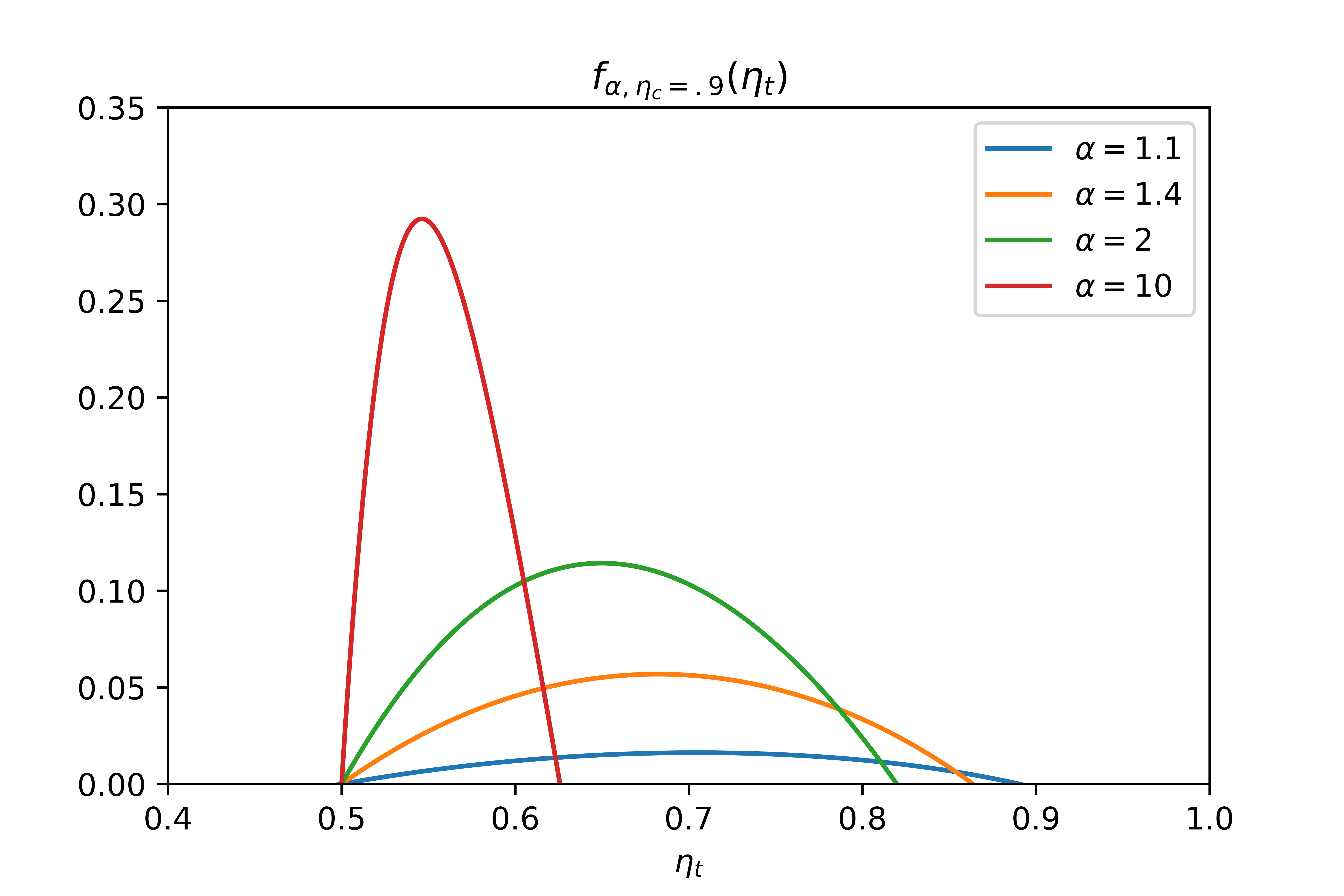}}
    \caption{Characteristic plot of the non-negative part of $f_{\alpha,\posclean}(\posnoise)$, where $\posclean = 0.9$, as a function of $\posnoise$ for several values of $\alpha$.
    Recall that the non-negative region of $f$ indicates where using Bayes tilted $\alpha$-estimate, as measured with the cross entropy for $\alpha$ given in~\eqref{defCE}, is strictly less than the $\alpha = 1$ cross entropy.
    Also recall that a Bayes blunting twist has the capability to shift $\posnoise$ anywhere in $[.5,\posclean = 0.9]$. 
    We see that for small $\alpha$, more twisted probabilities are ``covered'', whereas for large $\alpha$, less twisted probabilities are ``covered'', however, the large $\alpha$'s induce a large positive magnitude (ultimately measured by the KL-divergence) increase over the proper $\alpha = 1$.
    A key takeaway is that a fixed $\alpha$ (small enough) can correct a Bayes blunting twist for almost all $x \in \mathcal{X}$. However, it is not necessarily optimal as a perfectly tuned $\alpha$-correction mapping will use larger $\alpha$'s to optimally correct strongly twisted posteriors, inducing more gains over the $\alpha=1$ cross entropy.
    }
     \label{fig:BayesBlunting}
\end{figure}

\begin{figure}[h]
    \centering
    \centerline{\includegraphics[width=.75\linewidth]{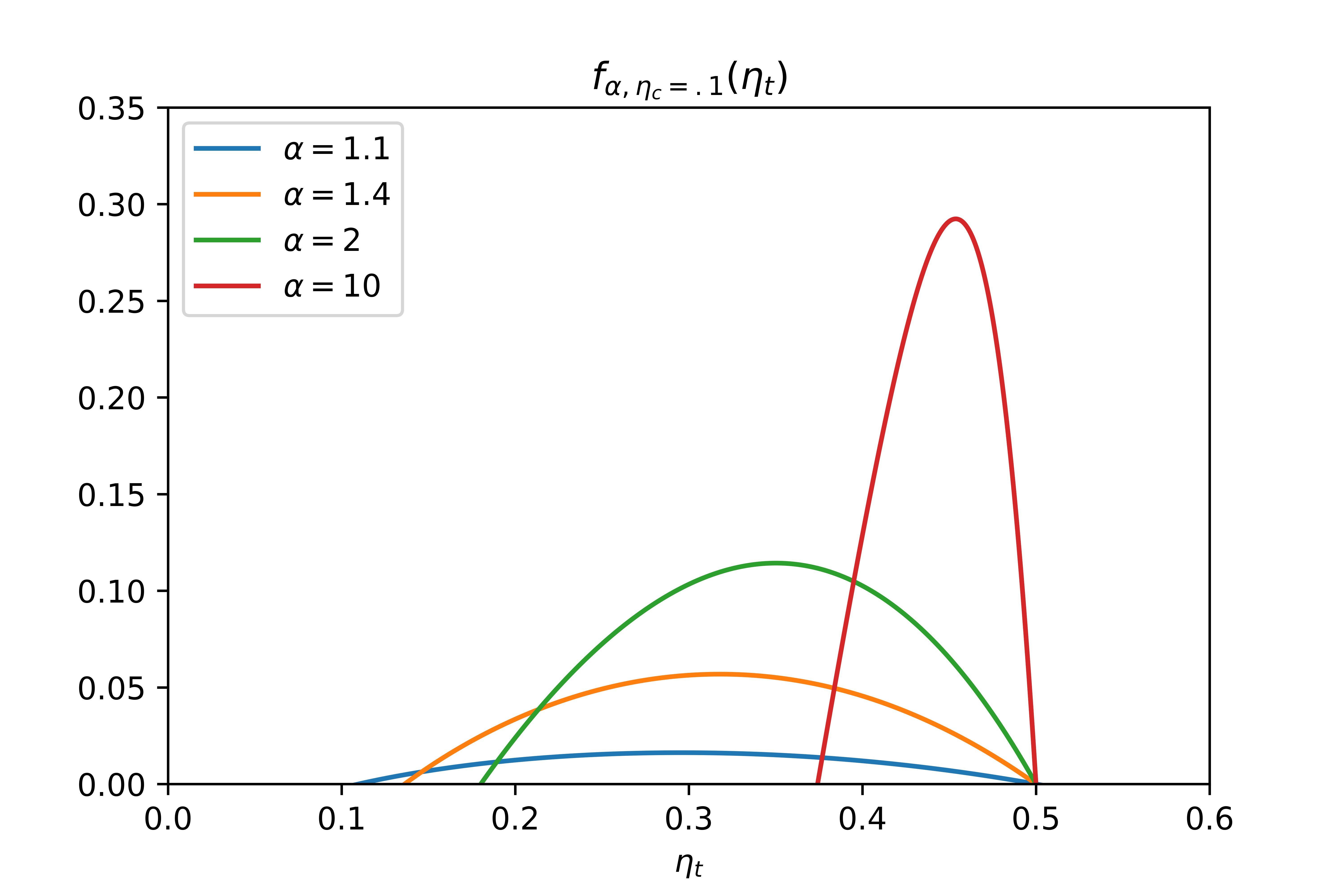}}
    \caption{Symmetric image of Figure~\ref{fig:BayesBlunting} for $\posclean = 0.1$.
    }
     \label{fig:symBayesBlunting}
\end{figure}

We want to show that for any strictly Bayes blunting twist $\posclean \mapsto \posnoise$, there exists a fixed $\alpha_{0} > 1$ and an optimal $\alpha^{\star}$-mapping, $\alpha^{\star}: \mathcal{X} \rightarrow \mathbb{R}_{>1}$, which induces
the following ordering 
\begin{align} \label{eq:bayesbluntingwts}
D_{\klloss}(\posnoise, \posclean; 1) > D_{\klloss}(\posnoise, \posclean; \alpha_{0}) \geq D_{\klloss}(\posnoise, \posclean; \alpha^\star).
\end{align}
Recalling~\eqref{eq:kldivalpha}, which is the identity 
\begin{align} 
D_{\klloss}(\posnoise,\posclean;\alpha) = \celoss(\posnoise, \posclean; \alpha) - H(\posclean),
\end{align}
by subtracting $H(\posclean)$ from both sides, we rewrite the desired statement~\eqref{eq:bayesbluntingdesiredstatement} (also given here in~\eqref{eq:bayesbluntingwts}) as 
\begin{align}
\celoss(\posnoise, \posclean; 1) > \celoss(\posnoise, \posclean; \alpha_{0}) \geq \celoss(\posnoise, \posclean; \alpha^{\star}).
\end{align}
In essence, we want to show that 
\begin{align} \label{eq:bayesbluntingsetup}
\celoss(\posnoise,\posclean;\alpha) < \celoss(\posnoise,\posclean;1) \quad \text{OR} \quad 0 < \celoss(\posnoise,\posclean;1) - \celoss(\posnoise,\posclean;\alpha),
\end{align}
for some $\alpha_{0} > 1$.
Continuing with the right-hand-side of~\eqref{eq:bayesbluntingsetup}, we have
\begin{align}
&\celoss(\posnoise,\posclean;1) - \celoss(\posnoise,\posclean;\alpha) \\
\nonumber &= \expect_{\X \sim \meas{M}} [\posclean(X) \cdot - \log{\posnoise(X)} + (1-\posclean(X)) \cdot -\log{(1-\posnoise(X))}] \\
&\quad\quad\quad\quad - \expect_{\X \sim \meas{M}} [\posclean(X) \cdot - \log{t_{\ell^{\alpha}}(\posnoise(X))} + (1-\posclean(X)) \cdot -\log{(1-t_{\ell^{\alpha}}(\posnoise(X))}] \\
&= \expect_{\X \sim \meas{M}}\left[\posclean(X) \log{\left(\frac{t_{\ell^{\alpha}}(\posnoise(X))}{\posnoise(X)} \right)} + (1 - \posclean(X)) \log{\left(\frac{1-t_{\ell^{\alpha}}(\posnoise(X))}{1-\posnoise(X)} \right)} \right] \\
\label{eq:bayesbluntingint0} &= \expect_{\X \sim \meas{M}}\left[\posclean(X) \log{\left(\frac{\posnoise(X)^{\alpha - 1}}{\posnoise(X)^{\alpha} + (1-\posnoise(X))^{\alpha}}\right)} + (1-\posclean(X)) \log{\left(\frac{(1-\posnoise(X))^{\alpha - 1}}{(1-\posnoise(X))^{\alpha} + \posnoise(X)^{\alpha}} \right)} \right],
\end{align}
where we used the linearity of the expectation and some algebra to combine the expressions. 
We want to show that the expression in brackets in~\eqref{eq:bayesbluntingint0} is strictly positive as this implies that $0 < \celoss(\posnoise,\posclean;1) - \celoss(\posnoise,\posclean;\alpha)$, in words, that the $\alpha$-Bayes tilted estimate untwists the Bayes blunting twist.
Continuing, we examine the expression in brackets in~\eqref{eq:bayesbluntingint0}
\begin{align}
\label{eq:compexpressionf} f_{\alpha, \posclean}(\posnoise) &= \posclean \log{\frac{\posnoise^{\alpha-1}}{\posnoise^\alpha + (1-\posnoise)^\alpha}} + (1-\posclean) \log{\frac{(1-\posnoise)^{\alpha-1}}{\posnoise^\alpha + (1-\posnoise)^\alpha}} \\
\label{eq:simpleexpressionf} &= (\alpha - 1) \posclean \log{\posnoise} + (\alpha - 1)(1-\posclean)\log{(1-\posnoise)} - \log{(\posnoise^{\alpha} + (1-\posnoise)^{\alpha})},
\end{align}
where we implicitly fix $X = x$ and consider scalar-valued $\posclean, \posnoise \in [0,1]$ and $\alpha \in \mathbb{R}_{+} /\{1\}$.
We note that $\alpha < 0$ does not need to be considered in this analysis since that regime of $\alpha$ is primarily useful for very strong twists due to its symmetry property (recall in Lemma~\ref{lemPointwise} that $t_{\ell^{\alpha}}$ is symmetric upon permuting $(\posnoise,\alpha)$ and $(1-\posnoise,-\alpha)$), i.e., \textit{not} useful for Bayes blunting twists which reduce confidence in the posterior but do not flip its sign across $\posnoise - 1/2$.
To build intuition of $f$, see Figure~\ref{fig:BayesBlunting} for a plot of this function.
Formally, we take note of the following observations/properties of $f$:
\begin{enumerate}
    \item \textbf{SYMMETRY.} From~\eqref{eq:compexpressionf}, it can be readily shown that for $\alpha > 0$ and for any fixed $\posclean \in [0,1]$, $f_{\alpha, \posclean}(\posnoise) = f_{\alpha, 1-\posclean}(1-\posnoise)$ for all $\posnoise \in [0,1]$. Also note that, by construction, for any $\posclean \in [0,1]$, $f_{1, \posclean}(\posnoise) = 0$, for all $\posnoise \in [0,1]$.
    \item \textbf{CONTINUITY.} From~\eqref{eq:simpleexpressionf}, it can be readily shown that for any fixed $\posclean, \posnoise \in [0,1]$, $f_{\alpha, \posclean}(\posnoise)$ is continuous in $\alpha \geq 1$.
    \item \textbf{CONCAVITY.} For arbitrarily fixed $\posclean$ and for any $\alpha > 1$, $f_{\alpha, \posclean}(\cdot)$ is concave in $\posnoise$, since (from~\eqref{eq:compexpressionf}) the composition of a concave function with a non-decreasing concave function yields a concave function. 
    As a side note, observe that $f_{\alpha, \posclean}(\cdot)$ is convex for $0 < \alpha < 1$, thus this regime of $\alpha$ does not untwist Bayes blunting twists. 
    Regarding (increa/decrea)sing concavity of $f_{\alpha, \posclean}(\cdot)$ for any fixed $\posclean \in [0,1]$ as a function of $\alpha$, traditionally a second derivative argument could indicate whether concavity is increasing or decreasing as a function of $\alpha$. 
    Unfortunately, $\frac{d^{2}}{d\posnoise^{2}}f_{\alpha, \posclean}(\posnoise)$ is an unwieldy analytical expression. 
    However, using a Taylor series approximation of $\frac{d^{2}}{d\posnoise^{2}}f_{\alpha, \posclean}(\posnoise)$ near $\posnoise = 1/2$, we find that the dominating term is $\approx -\alpha^{2}$.
    Thus, while not a proof, this \textit{indicates} that concavity of $f_{\alpha, \posclean}(\cdot)$ increases as $\alpha$ increases greater than $1$, which is sufficient for our purposes in the sequel.
    \item \textbf{ZEROES.} 
    It can be readily shown that for every $\posclean \in [0,1]$, $f_{\alpha, \posclean}(1/2) = 0$ for any $\alpha > 1$. Further, it can be shown that for any $\posclean \in [0,1]$, $\lim\limits_{\alpha \rightarrow 1^{+}} f_{\alpha, \posclean}(\posclean) \rightarrow 0^{-}$. 
    Thus, the exact values of $\posnoise$ for the other zero of $f_{\alpha,\posclean}(\cdot)$ (not $\posnoise = 1/2$), for each $\alpha > 1$, are given by the solution to the following transcendental equation:
    \begin{align} \label{eq:transcendentalbb}
        \posnoise = \left(\left((1-\posnoise)^{\alpha - 1} \right)^{1-\frac{1}{\posclean}} \left(\posnoise^{\alpha} + (1-\posnoise)^{\alpha} \right)^{\frac{1}{\posclean}} \right)^{\frac{1}{\alpha - 1}},
    \end{align}
    which can be rewritten as 
    \begin{align} \label{eq:transbbcont}
    \log{\left(\frac{\posnoise}{(1-\posnoise)^{1-\frac{1}{\posclean}}} \right)} = \frac{\alpha}{\alpha - 1} \log{\left(\posnoise^{\frac{1}{\posclean}}\right)} + \frac{1}{\posclean (\alpha - 1)} \log{\left(1 + \left(\frac{1}{\eta_{t}} - 1\right)^{\alpha} \right)}.
    \end{align}
    Suppose $\posclean > 1/2$, then since we have a Bayes blunting twist, $\posclean \geq \posnoise \geq 1/2$. 
    Letting $\alpha \rightarrow \infty$, note that the second term on the right-hand-side is $0$.
    Thus, we can solve for the zeroes when $\alpha = \infty$ by examining
    \begin{align}
    \log{\left(\frac{\posnoise}{(1-\posnoise)^{1-1/\posclean}} \right)} = \log{\left(\posnoise^{\frac{1}{\posclean}} \right)}.
    \end{align}
    After some manipulations, we obtain $\log{\left(\frac{1}{\posnoise} - 1 \right)} = 0$, which is only satisfied when $\posnoise = 1/2$. 
    Thus, for $\posclean > 1/2$ and $\alpha \rightarrow \infty$, both zeroes of~\eqref{eq:compexpressionf} converge at $\posnoise = 1/2$. 
    For $\posclean < 1/2$, the same argument holds, \textit{mutatis mutandis}.
    \item \textbf{MAXIMUM.} It can also be shown that the maximum of $f_{\alpha, \posclean}(\cdot)$ for each $\alpha > 1$ as a function of $\posnoise$ is given by the following transcendental equation
    \begin{align} \label{eq:transmax}
    \frac{\alpha(1-\posclean) + \posclean - \posnoise}{\alpha \posclean - \posclean + \posnoise} = \left(\frac{1}{\posnoise} - 1 \right)^{\alpha}.
    \end{align}
    One key observation we can make from~\eqref{eq:transmax} is that as $\alpha$ increases, the term on the right-hand-side grows (or decays) exponentially with $\alpha$, whereas the term on the left-hand-side is linear in $\alpha$. 
    With case-by-case analysis, i.e., for $\posclean > 1/2$ or $\posclean < 1/2$, it can be reasoned that as $\alpha$ increases, the solution to~\eqref{eq:transmax}, $\posnoise$, approaches $1/2$. 
    A second key observation we can make from~\eqref{eq:transmax} is that as $\alpha \rightarrow 1^{+}$, the solution to~\eqref{eq:transmax}, $\posnoise$, approaches $\posclean/2 + 1/4$. This is readily observed by setting $\alpha = 1+\epsilon$, for some $\epsilon > 0$, and $\posnoise = \frac{\posclean-1/2}{2} + 1/2 =  \posclean/2 + 1/4$, along with a Taylor series approximation of $\left(1/\posnoise - 1 \right)^{1+\epsilon}$ for $\epsilon$ near $0$.
    
\end{enumerate}


Intuitively, the remainder of the proof consists of a clipping argument. In words, we choose $\alpha_{0} > 1$ which ensures the least twisted $\posclean$ induces nonnegativity of $f_{\alpha, \posclean}(\cdot)$ given in~\eqref{eq:compexpressionf}, then, we argue that this implies that all possible $\posclean$ are ``covered'' - in the sense of inducing nonnegativity of~\eqref{eq:compexpressionf} - by the chosen $\alpha_{0} > 1$.

Formally, let $\posclean \rightarrow \posnoise$ be a \textit{strictly} Bayes blunting twist.
Thus, we have that either $(\posclean < \posnoise \leq 1/2)$ or $(\posclean > \posnoise \geq 1/2)$ for all $\posclean$.
By ZEROES of $f$, we have that for every $\posclean \in [0,1]$, $f_{\alpha, \posclean}(1/2) = 0$ for any $\alpha > 1$ and $\lim_{\alpha \rightarrow 1^{+}} f_{\alpha, \posclean}(\posclean) \rightarrow 0^{-}$.
We also have that as $\alpha \rightarrow \infty$, both zeroes of~\eqref{eq:compexpressionf} converge at $\posnoise = 1/2$. 
Thus, for every $\posclean$, the second zero (the first one is at $\posnoise = 1/2$) continuously shifts (CONTINUITY) from being located at $\posnoise = \posclean$ (as $\alpha \rightarrow 1$) to being located at $\posnoise = 1/2$ (as $\alpha \rightarrow \infty$).
Pick $\posclean^{*}$ such that $|\posclean - \posnoise|$ is minimized under $\posclean \rightarrow \posnoise$.
Note that $|\posclean - \posnoise|>0$ for all $\posclean$ by the definition of strictly Bayes blunting twist, hence there must exist an $\posclean$ minimizer (which we denote $\posclean^{*}$) under $\posclean \rightarrow \posnoise$.
By CONTINUITY, CONCAVITY, and ZEROES of $f$ above, we have that there exists some $\alpha_{0} > 1$ such that $f_{\alpha_{0},\posclean^{*}}(\posnoise) > 0$ for $\posnoise$ under $\posclean \rightarrow \posnoise$.
Note that by picking $\alpha_{0}>1$ in this way, we ensure a bijection of the relative zeroes of $f$ (as a function of $\alpha$) for every $\posclean$.
Thus, by CONCAVITY and SYMMETRY of $f$ above and this choice of $\alpha_{0} > 1$, 
we have that for all $\posclean$, both $f_{\alpha_{0},\posclean}(\posnoise) > 0$ and $f_{\alpha_{0},1-\posclean}(1-\posnoise) > 0$ under $\posclean \rightarrow \posnoise$ for the chosen $\alpha_{0} > 1$.
Thus, we obtain from~\eqref{eq:bayesbluntingint0} that for the chosen $\alpha_{0} > 1$, $0 < \celoss(\posnoise,\posclean;1) - \celoss(\posnoise,\posclean;\alpha_{0})$, i.e.,
\begin{align} \label{eq:bbdesired1}
\celoss(\posnoise,\posclean;1) > \celoss(\posnoise,\posclean;\alpha_{0}),
\end{align}
as desired.

We now show that $\celoss(\posnoise,\posclean;\alpha_{0}) \geq \celoss(\posnoise,\posclean ; \alpha^{\star})$, where the optimal $\alpha^{\star}$-mapping is given by some function $\alpha^{\star}: \mathcal{X} \rightarrow \mathbb{R}_{> 1}$.
By MAXIMUM above, we note that for a given $\posclean$, the maximum of $f_{\alpha, \posclean}(\cdot)$
moves from being achieved at $\posnoise =  \posclean/2 + 1/4$ to being achieved at $\posnoise = 1/2$, in the limit as $\alpha$ increases greater than $1$. 
By CONCAVITY above, we also observe that $f_{\alpha, \posclean}(\cdot)$ for a fixed $\posclean$ \textit{appears} (which is sufficient for the inequality) to become more strongly convex in general as $\alpha$ increases greater than $1$. 
We also note by CONTINUITY of $f$ above that the maximums are continuous in $\alpha > 1$.
Thus, under the \textit{strictly} Bayes blunting twist $\posclean \rightarrow \posnoise$, for every $\posclean$, 
there may exist an $\alpha>1$ which induces a larger (positive) magnitude in $f_{\alpha, \posclean}(\posnoise)$ than for the fixed $\alpha_{0}>1$ we found previously (for~\eqref{eq:bbdesired1}). 
Thus, there exists some optimal mapping $\alpha^{\star}: \mathcal{X} \rightarrow \mathbb{R}_{>1}$, such that
\begin{align} \label{eq:bbdesired2}
\celoss(\posnoise,\posclean;\alpha_{0}) \geq \celoss(\posnoise,\posclean;\alpha^{\star}).
\end{align}
Note that in the degenerate case, $\alpha^{\star} = \alpha_{0}$ for every $x \in \mathcal{X}$.
Therefore, combining~\eqref{eq:bbdesired1} and~\eqref{eq:bbdesired2}, we obtain 
\begin{align}
\celoss(\posnoise,\posclean;1) > \celoss(\posnoise,\posclean;\alpha_{0}) \geq \celoss(\posnoise,\posclean;\alpha^{\star}),
\end{align}
which is the desired result.

\subsection{Proof of Theorem \ref{thmALPHAGLOB}}\label{app-proof-thmALPHAGLOB}

As explained in the main body, we prove a result more general than Theorem \ref{thmALPHAGLOB}. 
However, briefly note that~\eqref{eq:thmalphaUB}, like the statement provided in Theorem~\ref{thmALPHABLUNTING} in~\eqref{eq:bayesbluntingdesiredstatement}, is proved for $\celoss$, but the statement provided in the main body as $\klloss$ is readily obtained from subtracting $H(\posclean)$ from both sides of the inequality.

First, we need a simple technical Lemma.
\begin{lemma}\label{lemBSUR}
  For any $B>0$, $\forall |z|\leq B, \forall \alpha \in \mathbb{R}$,
  \begin{eqnarray}
  \log(1+\exp(\alpha z)) & \leq & \log(1+\exp(\alpha B)) - \frac{B-z}{2}\cdot \alpha.
\end{eqnarray}
  \end{lemma}
  \begin{proof}
    We first note that $\forall |z|\leq 1, \forall \alpha \in \mathbb{R}$,
\begin{eqnarray}
\label{eq:thm10techlem1} \log(1+\exp(\alpha z)) & \leq & \frac{1+z}{2}\cdot
                                \log(1+\exp(\alpha)) +  \frac{1-z}{2}\cdot
                                \log(1+\exp(-\alpha)) \\
\label{eq:thm10techlem2}  & & = \log(1+\exp(\alpha)) - \frac{1-z}{2}\cdot \alpha,
\end{eqnarray}
which indeed holds as the LHS of~\eqref{eq:thm10techlem1} is convex and the RHS is the equation of a line passing through the points $(-1, \log(1+\exp(-\alpha)))$ and $(1, \log(1+\exp(\alpha)))$.
In~\eqref{eq:thm10techlem2}, we use $\log{(1+\exp(-\alpha)) = \log{(\exp(-\alpha)\cdot (1+\exp(\alpha)))}}$ on the second term in the RHS.
Hence if instead $|z|\leq B$, then
\begin{eqnarray}
  \log(1+\exp(\alpha z)) & = &\log\left(1+\exp\left(\alpha B  \cdot \frac{z}{B}\right)\right)\\
  & \leq & \log(1+\exp(\alpha B)) - \frac{B-z}{2}\cdot \alpha,
\end{eqnarray}
as claimed.
\end{proof}
We now show another Lemma which bounds the $\log$ quantities appearing in the cross-entropy in \eqref{defCE}, recalling that $\logit(u) \defeq \log(u/(1-u))$.
\begin{lemma}\label{lemEXPB}
  Fix $B>0$. For any $\ve{x} \in \mathcal{X}$ such that
  \begin{eqnarray}
\frac{1}{1+\exp B} \leq \posnoise(\ve{x}) \leq \frac{\exp B}{1+ \exp B}, \label{eqBPOST}
  \end{eqnarray}
  the following properties hold for the Bayes tiltes estimate $\pseudob$ of $\alpha$-loss:
  \begin{eqnarray*}
    - \log t_{\ell^{\alpha}}(\posnoise(\ve{x})) & \leq & \log(1+\exp(\alpha B)) -\alpha \cdot \frac{B + \logit(\posnoise(\ve{x}))}{2},\\
    - \log (1-t_{\ell^{\alpha}}(\posnoise(\ve{x}))) & \leq & \log(1+\exp(\alpha B)) - \alpha \cdot \frac{B - \logit(\posnoise(\ve{x}))}{2}.
  \end{eqnarray*}
  \end{lemma}
\begin{proof}
We note, using $z \defeq -\log \left(\frac{1-\posnoise(\ve{x})}{\posnoise(\ve{x})}\right)$, which satisfies $|z| \leq B$ from \eqref{eqBPOST} and Lemma \ref{lemBSUR},
\begin{eqnarray}
  - \log t_{\ell^{\alpha}}(\posnoise(\ve{x})) & = & -\log \left(\frac{\posnoise(\ve{x})^\alpha}{\posnoise(\ve{x})^\alpha+(1-\posnoise(\ve{x}))^\alpha}\right)\nonumber\\
  & = & -\log \left(\frac{1}{1+\left(\frac{1-\posnoise(\ve{x})}{\posnoise(\ve{x})}\right)^\alpha}\right)\nonumber\\
                                               & = & \log \left(1+\left(\frac{1-\posnoise(\ve{x})}{\posnoise(\ve{x})}\right)^\alpha\right)\nonumber\\
                                               & = & \log\left(1 + \exp\left(\alpha \log\left(\frac{1-\posnoise(\ve{x})}{\posnoise(\ve{x})}\right)\right)\right)\label{eqPSEUDOB0}\\ 
                                               & \leq & \log(1+\exp(\alpha B)) -\alpha \cdot \frac{B - \log\left(\frac{1-\posnoise(\ve{x})}{\posnoise(\ve{x})}\right)}{2}\label{bound1}\\
  & & = \log(1+\exp(\alpha B)) -\alpha \cdot \frac{B + \logit(\posnoise(\ve{x}))}{2},\label{blog1}
\end{eqnarray}
and similarly,
\begin{eqnarray}
 - \log (1-t_{\ell^{\alpha}}(\posnoise(\ve{x}))) & = & \log \left(1 + \exp\left(-\alpha \log\left(\frac{1-\posnoise(\ve{x})}{\posnoise(\ve{x})}\right)\right)\right)\label{eqPSEUDOB1}\\
                                & \leq &  \log(1+\exp(-\alpha B)) +\alpha \cdot \frac{B + \logit(\posnoise(\ve{x}))}{2}\nonumber\\
  & & = \log(1+\exp(\alpha B)) - \alpha B +\alpha \cdot \frac{B + \logit(\posnoise(\ve{x}))}{2}\nonumber\\
  & & = \log(1+\exp(\alpha B)) - \alpha \cdot \frac{B - \logit(\posnoise(\ve{x}))}{2},\label{blog2}
\end{eqnarray}
as claimed.
\end{proof}
Denote $\meas{M}(B)$ the distribution restricted to the support for which we have a.s. \eqref{eqBPOST} and let $p(B)$ be the weight of this support in $\meas{M}$. Let $\meas{M}(\overline{B})$ denote the restriction of $\meas{M}$ to the complement of this support. We let $\meas{D}(B)$ is the product distribution on examples ($\mathcal{X} \times \mathcal{Y}$) over the support of $\meas{M}(B)$ induced by marginal $\meas{M}(B)$ and posterior $\posclean$ (see \cite[Section 4]{rwID}). We are now in a position to show our generalization to Theorem \ref{thmALPHAGLOB}.

\begin{theorem}\label{thmALPHAGLOBsup}
  For any fixed $B>0$, let
\begin{eqnarray}
\edge(B) & \defeq & \frac{\expect_{(\X,\Y) \sim \meas{D}(B)} \left[\Y\cdot \logit(\posnoise(\X))\right]}{B} \quad \in [-1,1].
  \end{eqnarray}
and suppose we fix the scalar $\alpha \defeq \alpha^*$ with
\begin{eqnarray}
 \alpha^* & \defeq & \frac{\logit\left(\frac{1+\edge(B)}{2}\right)}{B}.
\end{eqnarray}
then the following bound holds on the cross-entropy of the Bayes tilted estimate of the $\alpha$-loss:
\begin{eqnarray*}
  \celoss(\posnoise, \posclean; \alpha)& \leq &  p(B) \cdot H\left(\frac{1+\edge(B)}{2}\right) \nonumber\\
  & & + (1-p(B)) \cdot \left(\edge(\overline{B}) \cdot \log\left(\frac{1+|\edge(B)|}{1-|\edge(B)|}\right) + \frac{1-|\edge(B)|}{1+|\edge(B)|}\right),
\end{eqnarray*}
where $\edge(\overline{B}) \defeq \expect_{(\X,\Y) \sim \meas{D}(\overline{B})} \left[\max\left\{0, -\mathrm{sign}(\alpha^*) \cdot \Y \cdot \logit(\posnoise(\X))\right\}\right] / B$ and $\meas{D}(\overline{B})$ is is defined analogously to $\meas{D}(B)$ with respect to $\meas{M}(\overline{B})$.
\end{theorem}
\noindent \textbf{Remark:} we notice this is indeed a generalization of Theorem \ref{thmALPHAGLOB}, which corresponds to case $p(B) = 1$. We also note $|\edge(\overline{B})| \geq 1$.
\begin{proof}
We remark that the cross-entropy \eqref{defCE} can be split as: 
\begin{eqnarray}
\celoss(\posnoise, \posclean; \alpha) & \defeq & \expect_{\X \sim \meas{M}} \left[
                                                   \begin{array}{c}
                                                     \posclean(\X)\cdot - \log t_{\ell^{\alpha}}(\posnoise(\X)) \\
                                                     + (1-\posclean(\X))\cdot - \log (1-t_{\ell^{\alpha}}(\posnoise(\X)))
                                                   \end{array}\right] \nonumber\\
                                      & = & p(B) \cdot K(\alpha) + (1-p(B)) \cdot L(B),\label{ceSPLIT}
\end{eqnarray}
with
\begin{eqnarray}
  K(\alpha) & \defeq & \expect_{\X \sim \meas{M}(B)} \left[\begin{array}{c}
                                                     \posclean(\X)\cdot - \log t_{\ell^{\alpha}}(\posnoise(\X)) \\
                                                     + (1-\posclean(\X))\cdot - \log (1-t_{\ell^{\alpha}}(\posnoise(\X)))
                                                     \end{array}\right],\\
  J(B) & \defeq & \expect_{\X \sim \meas{M}(\overline{B})} \left[\begin{array}{c}
                                                     \posclean(\X)\cdot - \log t_{\ell^{\alpha}}(\posnoise(\X)) \\
                                                     + (1-\posclean(\X))\cdot - \log (1-t_{\ell^{\alpha}}(\posnoise(\X)))
                                                     \end{array}\right].
\end{eqnarray}
We now focus on a bound on $K(\alpha)$, which we achieve via Lemma \ref{lemEXPB}:
\begin{eqnarray}
 K(\alpha) & \leq & \expect_{\X \sim \meas{M}(B)} \left[\begin{array}{c}
                                                     \posclean(\X)\cdot \left(\log(1+\exp(\alpha B)) -\alpha \cdot \frac{B + \logit(\posnoise(\X))}{2}\right) \label{eq:upperboundCEUB}\\
                                                     + (1-\posclean(\X))\cdot \left(\log(1+\exp(\alpha B)) - \alpha \cdot \frac{B - \logit(\posnoise(\X))}{2}\right)
                                                        \end{array}\right]\nonumber\\
  & = & \log(1+\exp(\alpha B)) - \alpha \cdot \frac{B + \expect_{\X \sim \meas{M}(B)} \left[\posclean(\X)\logit(\posnoise(\X)) + (1-\posclean(\X))\cdot - \logit(\posnoise(\X))\right]}{2}\nonumber\\
           & = & \log(1+\exp(\alpha B)) - \alpha \cdot \frac{B + \expect_{(\X,\Y) \sim \meas{D}(B)} \left[\Y\cdot \logit(\posnoise(\X))\right]}{2}\label{boundWD}\\
  & \defeq & \underbrace{\log(1+\exp(\alpha B)) - \alpha \cdot \frac{B + B\cdot \edge(B)}{2}}_{\defeq L(\alpha)},\nonumber
\end{eqnarray}
where we recall
\begin{eqnarray}
\edge(B) & \defeq & \frac{\expect_{\X \sim \meas{D}(B)} \left[\Y\cdot \logit(\posnoise(\X))\right]}{B} \quad \in [-1,1].
  \end{eqnarray}
  
  \begin{figure}[h]
    \centering
    \centerline{\includegraphics[width=.75\linewidth]{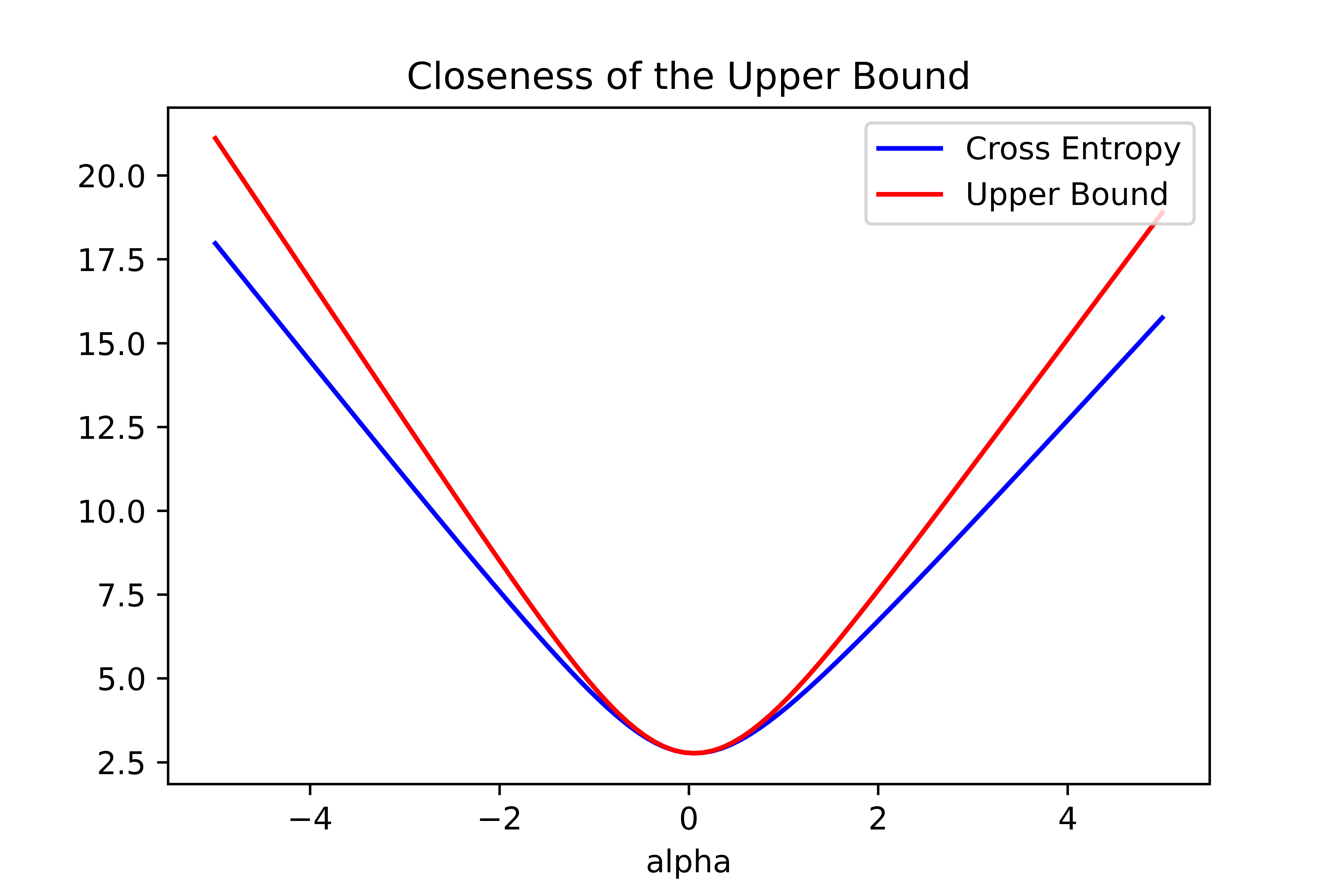}}
    \caption{A plot illustrating the closeness of~\eqref{eq:upperboundCEUB} where $K(\alpha)$ is given in blue and $L(\alpha)$ is given in red for a toy distribution: $x \sim \text{Uniform}([-B,B])$ (recall $B>0$ is the clipping threshold and we set $B = 2$ here) where $\posclean(x) = (1+\exp{(-x/a)})^{-1}$ and $\posnoise(x) = (1+\exp{(-x/b)})^{-1}$ such that $a = 10$ and $b = 0.6$.}
     \label{fig:CEUB}
\end{figure}
  
  Notice the change in distribution in \eqref{boundWD}, where $\meas{D}(B)$ is the product distribution on examples ($\mathcal{X} \times \mathcal{Y}$) over the support of $\meas{M}(B)$ induced by marginal $\meas{M}(B)$ and posterior $\posclean$ (see \cite[Section 4]{rwID}). We have
\begin{eqnarray}
L'(\alpha) & = & B\cdot\left(\frac{\exp(B\alpha)}{1+\exp(B\alpha)} - \frac{1+\edge(B)}{2}\right),
\end{eqnarray}
which zeroes
for
\begin{eqnarray} \label{eq:thmalphastar}
  \alpha^* & = & \frac{1}{B} \cdot \log\left(\frac{1+\edge(B)}{1-\edge(B)}\right) = \frac{\logit(q(B))}{B}.
\end{eqnarray}
Further, we have that 
\begin{align} \label{eq:secondderivativeLalpha}
L''(\alpha) = \dfrac{d}{d\alpha} L'(\alpha) = \frac{B^2 \exp{(\alpha B)}}{(\exp{(\alpha B)} + 1)^{2}},
\end{align}
and plugging in~\eqref{eq:thmalphastar} yields
\begin{align} \label{eq:secondderivativeLalphastar}
L''(\alpha^{*}) = B^{2} \frac{1 - \edge^{2}}{4}.
\end{align}
Note that for fixed $B>0$, as $|\alpha|$ increases in~\eqref{eq:secondderivativeLalpha}, $L''(\alpha)$ decreases. Thus, when the magnitude of $\alpha^{*}$ is large (due to the distribution and twist), this implies that there is more ``flatness'' near the choice of $\alpha^{*}$. 
Hence, in these regimes, a choice of $\alpha_{0}$ ``close-enough'' to $\alpha^{*}$ should have similar performance in practice.
Continuing with the main line, plugging in~\eqref{eq:thmalphastar} into~\eqref{boundWD} yields
\begin{eqnarray}
  K(\alpha^*) & \leq & \log(1+\exp(B\alpha^*))  - B\cdot \frac{1+\edge(B)}{2}\cdot \alpha^*\\
              & & =  -\log\left(\frac{1-\edge(B)}{2}\right)  - \frac{1+\edge(B)}{2} \cdot \log\left(\frac{1+\edge(B)}{1-\edge(B)}\right)\\
              & = &  - \frac{1+\edge(B)}{2} \log\left(\frac{1+\edge(B)}{2} \right) - \frac{1-\edge(B)}{2} \log\left(\frac{1-\edge(B)}{2} \right) \\
  & = & H\left(\frac{1+\edge(B)}{2}\right),\label{modALPHA}
\end{eqnarray}
which is the statement of Theorem \ref{thmALPHAGLOB}. We now focus on $J(B)$. Since $\log(1+\exp(-z))\leq \exp(-z), \forall z$ via an order-1 Taylor expansion, it follows that if $z\geq C$ for some $C>0$, then $\log(1+\exp(-z)) \leq \exp(-C)$. Equivalently, we get
\begin{eqnarray}
z \geq C & \Rightarrow & \log(1+\exp(z)) \leq z + \exp(-C).
\end{eqnarray}
By symmetry, we have
\begin{eqnarray}
  z \leq -C & \Rightarrow & \log(1+\exp(z)) \leq \exp(-C),
\end{eqnarray}
so we get
\begin{eqnarray}
|z| \geq C & \Rightarrow & \log(1+\exp(z)) \leq \max\{0,z\} + \exp(-C).
\end{eqnarray}
By definition, we have for any $\ve{x}$ in the support of $\meas{M}(\overline{B})$,
\begin{eqnarray}
\left|\log\left(\frac{1-\posnoise(\ve{x})}{\posnoise(\ve{x})}\right)\right| & \geq & B,
  \end{eqnarray}
  so have, considering $C \defeq B \cdot |\alpha^*|$, from \eqref{eqPSEUDOB0} and \eqref{eqPSEUDOB1},
\begin{eqnarray}
J(B) & \defeq & \expect_{\X \sim \meas{M}(\overline{B})} \left[\begin{array}{c}
                                                     \posclean(\X)\cdot - \log t_{\ell^{\alpha}}(\posnoise(\X)) \\
                                                     + (1-\posclean(\X))\cdot - \log (1-t_{\ell^{\alpha}}(\posnoise(\X)))
                                                               \end{array}\right]\\
     & = & \expect_{(\X,\Y) \sim \meas{D}(\overline{B})} \left[\log \left(1 + \exp\left(\Y \alpha^* \log\left(\frac{1-\posnoise(\X)}{\posnoise(\X)}\right)\right)\right)\right]\nonumber\\
     & \leq & \expect_{(\X,\Y) \sim \meas{D}(\overline{B})} \left[\max\left\{0, \Y \alpha^* \log\left(\frac{1-\posnoise(\X)}{\posnoise(\X)}\right)\right\}\right] + \exp\left(-B \cdot |\alpha^*|\right)\nonumber\\
  & & = |\alpha^*| \cdot \expect_{(\X,\Y) \sim \meas{D}(\overline{B})} \left[\max\left\{0, -\mathrm{sign}(\alpha^*) \cdot \Y \cdot \logit(\posnoise(\X))\right\}\right] + \frac{1-|\edge(B)|}{1+|\edge(B)|}\nonumber\\
  & = & \edge(\overline{B}) \log\left(\frac{1+|\eta(B)|}{1-|\edge(B)|}\right) + \frac{1-|\edge(B)|}{1+|\edge(B)|}, 
  \end{eqnarray}
which completes the proof of Theorem \ref{thmALPHAGLOBsup} after replacing the expression of $\alpha^*$.
\end{proof}
\noindent \textbf{Remarks}: Theorem \ref{thmALPHAGLOBsup} calls for several remarks:\\

\noindent \textbf{Gains with respect to the ``proper" choice $\alpha = 1$}: the case we develop is simplistic but allows a graphical comparison of the gains that Theorem allow to get compared to the choice $\alpha=1$, which we recall corresponds to the (proper) logistic loss. Suppose $p(B) = 1$ so the cross-entropy $\celoss(\posnoise, \posclean; \alpha)$ in \eqref{ceSPLIT} reduces to $K(.)$, $B=1$ and all logits take $\pm 1$ value a.e.,
\begin{eqnarray}
z(\ve{x}) \defeq \log\left(\frac{\posnoise(\ve{x})}{1-\posnoise(\ve{x})}\right) & = & \pm 1,
\end{eqnarray}
which can be achieved by clamping, and $p \defeq \pr_{(\X,\Y) \sim \meas{M}(1)}[\Y z(\X) = 1]$, which gives $\edge(1) = 2p - 1$, 
\begin{eqnarray*}
\alpha^* & = & \log\left(\frac{p}{1-p}\right),\label{defalphastarsup}
\end{eqnarray*}
and
\begin{eqnarray}
  \celoss(\posnoise, \posclean; \alpha^*) & \leq & H(p) \label{cerstar}
\end{eqnarray}
from Theorem \ref{thmALPHAGLOBsup}. The properness choice $\alpha^* = 1$ however gives
\begin{eqnarray}
  \celoss(\posnoise, \posclean; 1) = K(1) & = & \expect_{(\X,\Y) \sim \meas{M}(1)} \left[\log \left(1 + \exp\left(-\Y z(\X)\right)\right)\right]\\
       & = & p\log(1+\exp(-1)) + (1-p) \log(1+\exp(1)).\\
  & = & \log(1+e) - p \label{cerzero}.
\end{eqnarray}
Figure \ref{f-celoss-comp} plots $\celoss(\posnoise, \posclean; \alpha^*)$ \eqref{cerstar} vs $\celoss(\posnoise, \posclean; 1)$ \eqref{cerzero}. We remark that $\celoss(\posnoise, \posclean; \alpha^*) \leq \celoss(\posnoise, \posclean; 1)$, and the difference is especially large as $p \rightarrow \{0,1\}$, for which $\celoss(\posnoise, \posclean; \alpha^*)\rightarrow 0$ while we always have $\celoss(\posnoise, \posclean; 1) > 0.3, \forall p$.

\begin{figure}[t]
\begin{center}
\begin{tabular}{c}
\includegraphics[trim=130bp 650bp 1380bp
60bp,clip,width=0.30\linewidth]{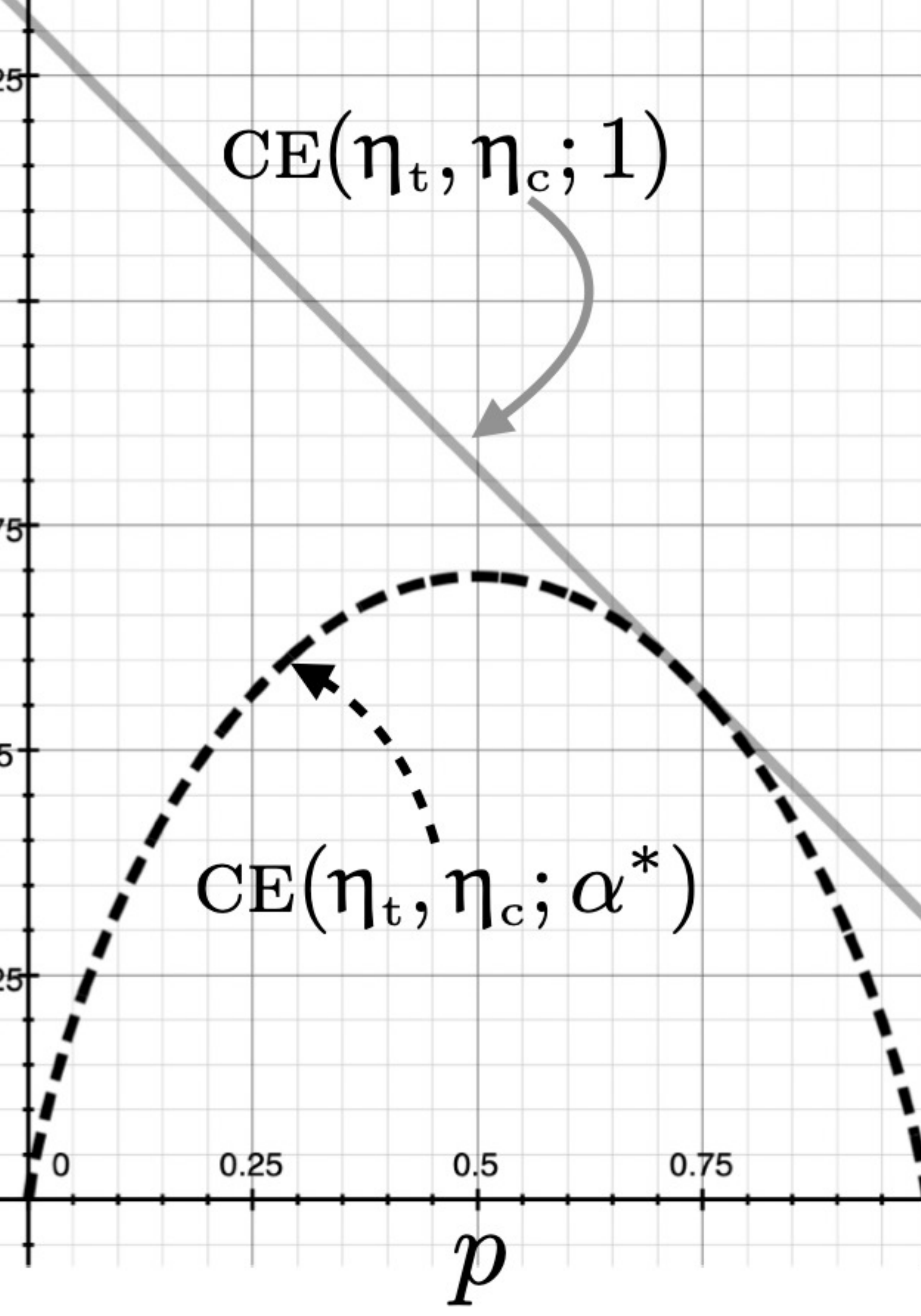}
\end{tabular}
\end{center}
\caption{Comparison between the cross-entropy of the logistic loss ($\alpha=1$) and that of the $\alpha$-loss for the scalar correction in \eqref{defalphastarsup} in Theorem \ref{thmALPHAGLOBsup}.}
  \label{f-celoss-comp}
\end{figure}

\noindent \textbf{Incidence of computing $\alpha^*$ on an \textit{estimate} of $\edge(B)$}: Theorem \ref{thmALPHAGLOBsup} can be refined if, instead of the true value $\edge(B)$ we have access to an estimate $\hat{\edge}(B)$. In this case, we can refine the proof of the Theorem from the series of eqs in \eqref{modALPHA}. We remark that
\begin{eqnarray}
H'\left(\frac{1+z}{2}\right) & = & \frac{1}{2} \cdot \log\left(\frac{1-z}{1+z}\right),
\end{eqnarray}
so since $H$ is concave, we have for any $\edge(B), \hat{\edge}(B)$,
\begin{eqnarray}
  H\left(\frac{1+\edge(B)}{2}\right) & \leq & H\left(\frac{1+\hat{\edge}(B)}{2}\right) + \left(\frac{1+\edge(B)}{2}-\frac{1+\hat{\edge}(B)}{2}\right) \cdot \frac{1}{2} \cdot \log\left(\frac{1-\hat{\edge}(B)}{1+\hat{\edge}(B)}\right)\nonumber\\
                                    & & = H\left(\frac{1+\hat{\edge}(B)}{2}\right) + \frac{\edge(B)-\hat{\edge}(B)}{4} \cdot \log\left(\frac{1-\hat{\edge}(B)}{1+\hat{\edge}(B)}\right)\nonumber\\
  & \leq & H\left(\frac{1+\hat{\edge}(B)}{2}\right) + \frac{|\edge(B)-\hat{\edge}(B)|}{4} \cdot \log\left(\frac{1+|\hat{\edge}(B)|}{1-|\hat{\edge}(B)|}\right)\nonumber\\
  & & = H\left(\frac{1+\hat{\edge}(B)}{2}\right) + \frac{|\edge(B)-\hat{\edge}(B)|}{4} \cdot \log\left(1+\frac{2|\hat{\edge}(B)|}{1-|\hat{\edge}(B)|}\right)\nonumber\\
  & \leq & H\left(\frac{1+\hat{\edge}(B)}{2}\right) + \frac{|\edge(B)-\hat{\edge}(B)||\hat{\edge}(B)|}{2(1-|\hat{\edge}(B)|)},
\end{eqnarray}
where we have used $\log(1+z)\leq z$ for the last inequality.

\noindent \textbf{Polarity of $\alpha^*$}: as presented in the main body, the state of the art defines the $\alpha$-loss only for $\alpha\geq 0$. The proof of Theorem \ref{thmALPHAGLOBsup}, and more specifically its proof, hints at why alleviating this constraint is important and corresponds to especially difficult cases. We have the general rule $\alpha^* \leq 0$ iff $\edge(B) \leq 0$, which indicates that the twisted posterior tends to be small when the clean posterior tends to be large. Since the Bayes tilted estimate is symmetric if we switch the couple $(\alpha, \posnoise)$ for $(-\alpha, 1-\posnoise)$, $\alpha^* \leq 0$ provokes a change of polarity in the Bayes tilted estimate compared to the twisted posterior. It thus corrects the twisted posterior. We emphasize that such a situation happens for especially damaging twists (in particular, \textit{not} Bayes blunting).

\subsection{Pseudo-Inverse Link} \label{app-PIL}

\begin{figure}[t]
\bignegspace
  \begin{center}
\includegraphics[trim=250bp 750bp 720bp
210bp,clip, width=0.41\textwidth]{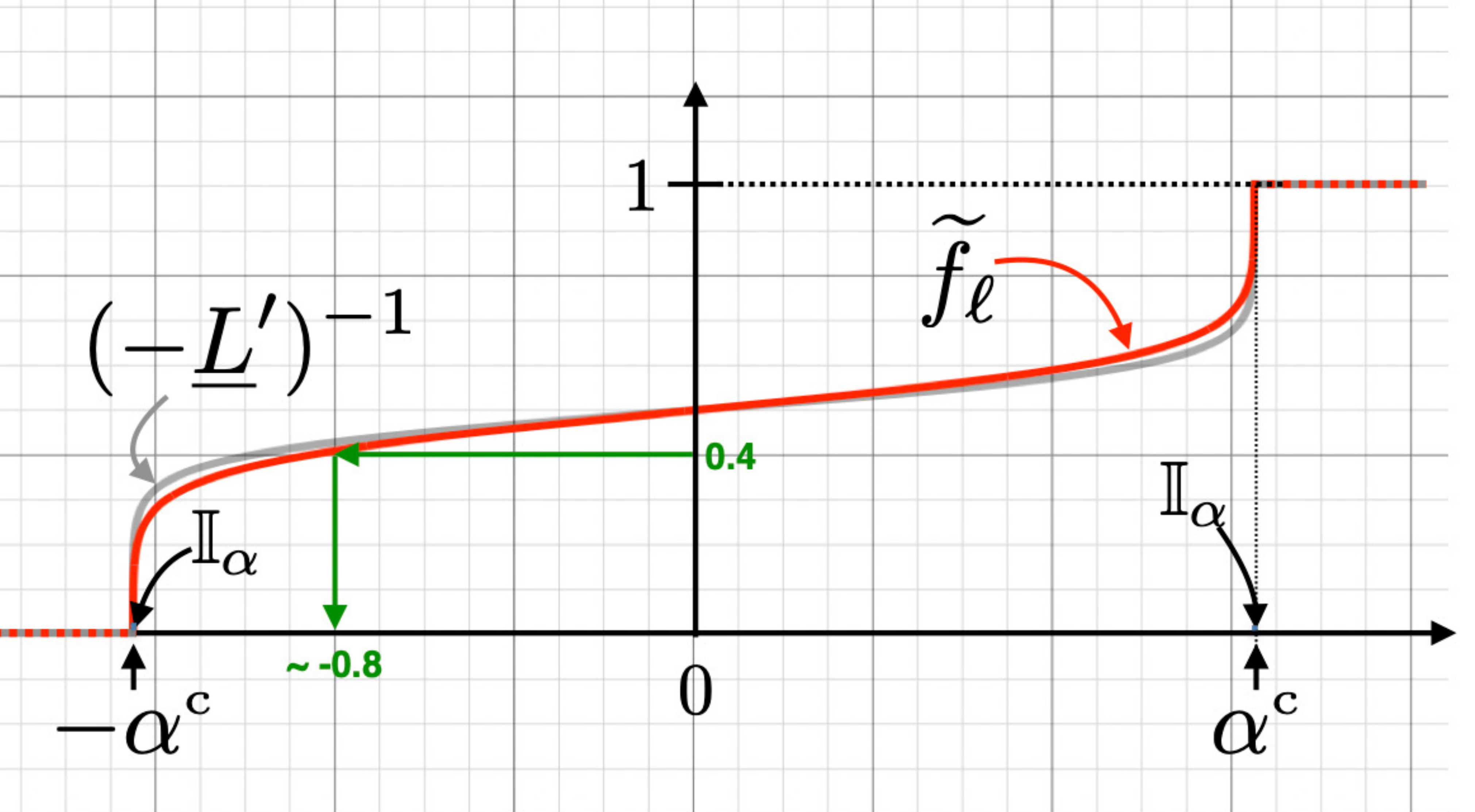}
\end{center}
\caption{\cil~link $\pseudomirrorinv$ vs inverse link $-\poibayesrisk'$ for ($\alpha = 5$)-loss. Notice the quality of the approximation.}
  \label{f-alpha} 
     \vspace{-0.3cm}
\end{figure}

Let $\alpha \in [-\infty, \infty]$, and define the conjugate $\holdera$ such that $1/\holdera + 1/\alpha = 1$, using by extension $\holdera(\infty) = 1, \holdera(1) = \infty$.
If we were to exactly implement a boosting algorithm for the $\alpha$-loss,
we would have to find the \textit{exact} inverse of
\eqref{defweights}, which would require inverting $ -\poibayesrisk'(v)
\defeq \holdera \cdot \pseudob(v)^\holdera - \holdera \cdot
\pseudob(1-v)^\holdera$. Owing to the difficulty to
carry out this step, we {choose a sidestep} that makes
inversion straightforward and can fall in the conditions to apply
Theorem \ref{genBOOST}, thus making \pmboost~a boosting algorithm for 
the $\alpha$-loss of interest. The trick does not just hold for the
$\alpha$-loss, so we describe it for a general loss $\ell$ assuming
for simplicity that $\partialloss{1}(1)=\partialloss{-1}(0) = 0$
and $\pseudob, \partialloss{1}, \partialloss{-1}$ are {invertible} with $\partialloss{1}, \partialloss{-1}$ non-negative, conditions that would hold for many
popular losses (log, square, Matusita, etc.), and the $\alpha$-loss. We then approximate the link 
$-\poibayesrisk'$ by using just one of $\partialloss{-1}$ or
$\partialloss{1}$ depending on their argument, while ensuring functions match
in $0, 1/2, 1$. We name $\pseudomirrorinv$ the \textit{clipped inverse link}, \cil. Letting $a_\loss^- \defeq
\partialloss{1}(0)/(\partialloss{1}(0)-\partialloss{1}(1/2))$ and $a_\loss^+
\defeq
\partialloss{-1}(1)/(\partialloss{-1}(1)-\partialloss{-1}(1/2))$, our link approximation makes use of the following function: $f_\loss(u) \defeq f_\loss^-(u)$ if $u\leq 1/2$ and $f_\loss(u) \defeq f_\loss^+(u)$ otherwise, with the shorthands $f_\loss^-(u) \defeq a_\loss^- \cdot \left(\partialloss{1}(1/2)-\partialloss{1}(\pseudob(u))\right)$, $f_\loss^+ (u) \defeq a_\loss^+ \cdot \left(\partialloss{-1}(\pseudob(u))-\partialloss{-1}(1/2)\right)$. The following Lemma shows, in addition to properties of $f_\loss$, the expression obtained for the clipped inverse link for a general \cpe~loss.
\begin{lemma}\label{lemINVERTLP}
  $f_\loss(u) = -\poibayesrisk'(u), \forall u \in \{0, 1/2, 1\}$; furthermore, the clipped inverse link $\pseudomirrorinv \defeq f^{-1}_\loss$ is: (i) $\pseudomirrorinv(z) = 0$ if $z< -\partialloss{1}(0)$; (ii) $\pseudomirrorinv(z) = \pseudob^{-1} \circ \partialloss{1}^{-1}\left(\frac{\partialloss{1}(1/2)-\partialloss{1}(0)}{\partialloss{1}(0)}\cdot z + \partialloss{1}(1/2)\right)$ if $-\partialloss{1}(0) \leq z< 0$; (iii) $\pseudomirrorinv(z) = \pseudob^{-1} \circ \partialloss{-1}^{-1}\left(\frac{\partialloss{-1}(1)-\partialloss{-1}(1/2)}{\partialloss{-1}(1)}\cdot z + \partialloss{-1}(1/2)\right)$ if $0 \leq z< \partialloss{-1}(1)$; (iv) $\pseudomirrorinv(z) = 1$ if $z\geq \partialloss{-1}(1)$. Furthermore, $\pseudomirrorinv$ is continuous and if \textbf{(S)} and \textbf{(D)} hold, then $\pseudomirrorinv$ is derivable on $\mathbb{R}$ (with the only possible exceptions of $\{-\partialloss{1}(0), \partialloss{-1}(1)\}$).
\end{lemma}
     \bignegspace
The proof is immediate once we remark that $\partialloss{1}(1)=\partialloss{-1}(0) = 0$ bring "properness for the extremes", \textit{i.e.} $0 \in \pseudob(0), 1 \in \pseudob(1)$. 
We now give the expression of the formulas of interest regarding Lemma \ref{lemINVERTLP} for the $\alpha$-loss.
\begin{lemma} \label{lemPILalphaloss}
  We have for the $\alpha$-loss,
  \begin{eqnarray}
    f_\loss (u) & = & \holdera \cdot \left\{
               \begin{array}{rcl}
                 \left(\frac{2\cdot u^\alpha}{u^\alpha+(1-u)^\alpha}\right)^{\frac{1}{\holdera}} - 1 & \mbox{ if } & u\leq 1/2,\\
                 1-\left(\frac{2\cdot (1-u)^\alpha}{u^\alpha+(1-u)^\alpha}\right)^{\frac{1}{\holdera}}  & \mbox{ if } & u\geq 1/2
                 \end{array}
                                                                                                                        \right.,\\
    \pseudoinvlink(z) & = & \left\{
                          \begin{array}{ccl}
                            0 & \mbox{ if } & z\leq -\holdera,\\
                  \frac{\left(\holdera+z\right)^{\frac{\holdera}{\alpha}}}{\left(\holdera+z\right)^{\frac{\holdera}{\alpha}} + \left(2\holdera^\holdera - \left(\holdera+z\right)^{\holdera}\right)^{\frac{1}{\alpha}}} & \mbox{ if } & -\holdera\leq z\leq 0,\\
                            \frac{\left(2\holdera^\holdera - \left(\holdera-z\right)^{\holdera}\right)^{\frac{1}{\alpha}}}{\left(\holdera-z\right)^{\frac{\holdera}{\alpha}} + \left(2\holdera^\holdera - \left(\holdera-z\right)^{\holdera}\right)^{\frac{1}{\alpha}}} & \mbox{ if } & 0 \leq z\leq \holdera,\\
                            1 & \mbox{ if } & z\geq \holdera.
                 \end{array}\right. .
\end{eqnarray}
  \end{lemma}
  
Rewritten, we have that
\begin{align} \label{eq:tildefupdate}
\pseudoinvlink(z) = 
\begin{cases}
0 & z \leq -\frac{\alpha}{\alpha - 1}\\
\frac{\left(\frac{\alpha}{\alpha - 1} + z \right)^{\frac{1}{\alpha - 1}}}{\left(\frac{\alpha}{\alpha - 1} + z \right)^{\frac{1}{\alpha - 1}} + \left(2 \left(\frac{\alpha}{\alpha - 1} \right)^{\frac{\alpha}{\alpha - 1}} - \left(\frac{\alpha}{\alpha - 1} + z \right)^{\frac{\alpha}{\alpha - 1}} \right)^{\frac{1}{\alpha}}} & -\frac{\alpha}{\alpha - 1} \leq z \leq 0 \\
\frac{\left(2\left(\frac{\alpha}{\alpha - 1}\right)^{\frac{\alpha}{\alpha - 1}} - \left(\frac{\alpha}{\alpha - 1} - z \right)^{\frac{\alpha}{\alpha - 1}} \right)^{\frac{1}{\alpha}}}{\left(\frac{\alpha}{\alpha - 1} - z \right)^{\frac{1}{\alpha - 1}} + \left(2\left(\frac{\alpha}{\alpha - 1} \right)^{\frac{\alpha}{\alpha - 1}} - \left(\frac{\alpha}{\alpha - 1} - z \right)^{\frac{\alpha}{\alpha - 1}} \right)^{\frac{1}{\alpha}}} & 0 \leq z \leq \frac{\alpha}{\alpha - 1} \\
1 & z \geq \frac{\alpha}{\alpha - 1}
\end{cases}
\end{align}
  Figure \ref{fig:tildef} plots~\eqref{eq:tildefupdate} for several values of $\alpha$.
\begin{figure}[h] 
    \centerline{\includegraphics[scale=.25]{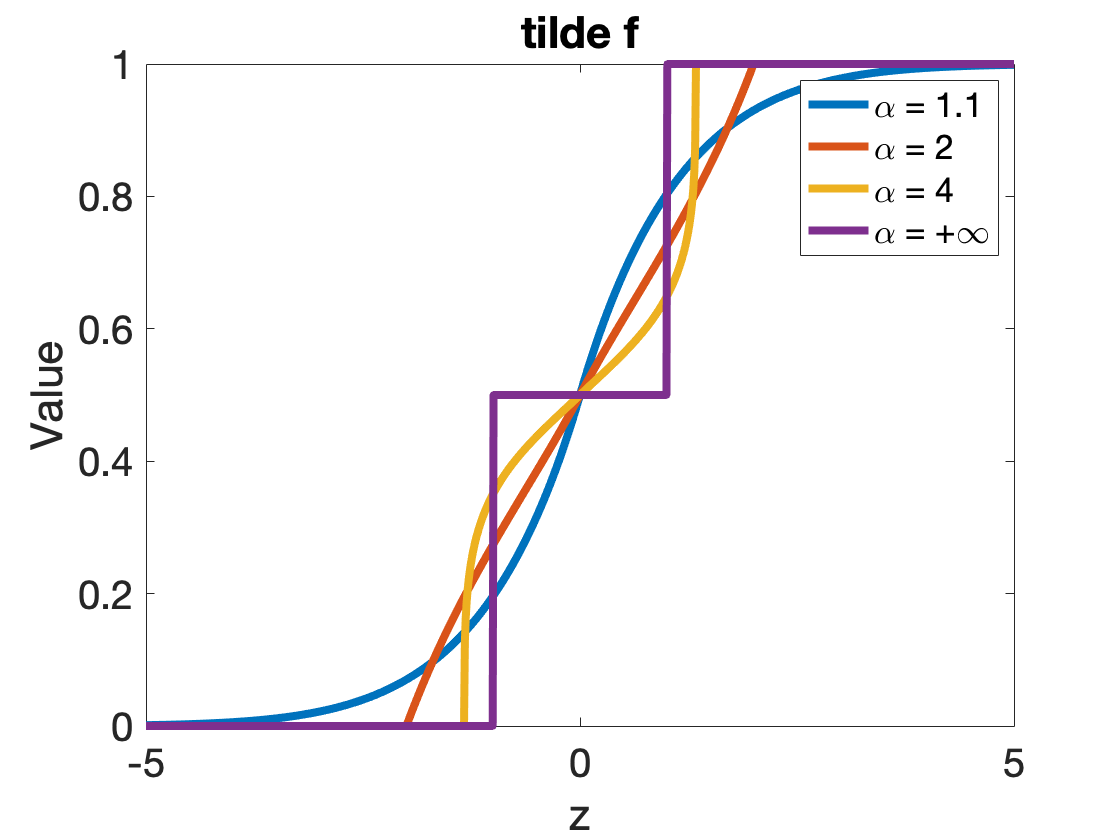}}
    \caption{A plot of $\tilde{f}(z)$ as a function of $\alpha$ as given in~\eqref{eq:tildefupdate}.}
    \label{fig:tildef}
\end{figure}
\begin{remark}\label{remDer}
It could be tempting to think that the clipped inverse link trivially comes from clipping the partial losses themselves such as replacing $\partialloss{1}(u)$ by $0$ if $u\geq 1/2$ and symmetrically for $\partialloss{-1}(u)$. This is not the case as it would lead to $\poibayesrisk$ piecewise constant and therefore $-\poibayesrisk' = 0$ when defined.
  \end{remark}\\
  We turn to a result that authorizes us to use Thm \ref{genBOOST} while virtually not needing \textbf{(O1)} and \textbf{(O2)} for $\alpha$-loss. Denote $\mathbb{I}_\alpha \defeq\pm \holdera\cdot [1-(1/\alpha^4), 1]$ (See Fig. \ref{f-alpha}).
\begin{lemma}\label{edgeALPHA}
Suppose $\alpha \geq 1.2$. For $\pseudomirrorinv$ defined as in Lemma \ref{lemINVERTLP}, $\exists K\geq 0.133$ such that $\alpha$-loss satisfies:
\bignegspace
  \begin{eqnarray}
\forall z \not\in \mathbb{I}_\alpha, |(\pseudomirrorinv - (-\poibayesrisk')^{-1})(z)| & \lesssim & K/\alpha. \label{eqHorDiff}
    \end{eqnarray}
  \end{lemma}
    \bignegspace
  Remark the necessity of a trick as we do not compute $(-\poibayesrisk')^{-1}$ in \eqref{eqHorDiff}. The proof, in Section \ref{proof-lem-edgeALPHA}, bypasses the difficulty by bounding the \textit{horizontal} distance between the \textit{inverses}. The Lemma can be read as: with the exception of an interval vanishing rapidly with $\alpha$, the difference between $\pseudomirrorinv$ (that we can easily compute for the $\alpha$-loss) and $(-\poibayesrisk')^{-1}$ (that we do not compute for the $\alpha$-loss), in order or just pointwise (typically for $\alpha < 10$) is at most $0.14/\alpha$.
  We now show how we can virtually "get rid of" (\textbf{O1}) and (\textbf{O2}) in such a context to apply Theorem \ref{genBOOST}. Consider the following assumptions: (i) no edge falls in $\mathbb{I}_\alpha$, (ii) the weak learner guarantees $\upgamma = 0.14$, (iii) the average weights, $\overline{w}_j \defeq \ve{1}^\top \ve{w}_j/m$, satisfies $\overline{w}_j \geq 0.4$. Looking at Figure \ref{f-alpha}, we see that (i) is virtually not limiting at all; (ii) is a reasonable assumption on \weak; remembering that a weight has the form $w = \pseudomirrorinv(-y H(\ve{x}))$, we see that (iii) requires $H$ to be not "too good", see for example Figure \ref{f-alpha} in which case $w = 0.4$ implies an edge $y H \leq 0.8$. We now observe that given (i), it is trivial to find $a_f$ to satisfy (\textbf{O2}) since we focus only on one $\alpha$-loss. Suppose $\alpha \geq 2.7$, which approaches the average value of the $\alpha$s in our experiments, and finally let $\zeta \defeq 2.5/2.7 \approx 0.926$. Then we get the chain of inequalities:
  \begin{eqnarray}
    \lefteqn{\Delta(F) \underset{\mbox{\tiny{Lem. \ref{edgeALPHA}}}}{\leq} M\cdot \frac{0.14}{\alpha} \underset{\mbox{\tiny{(ii)}}}{=} \upgamma M\cdot \frac{1}{\alpha} \underset{\textbf{(WLA)}}{\leq}  \frac{1}{\alpha} \cdot \tilde{\edge}_j} \nonumber\\
    & & \quad \defeq\frac{1}{\alpha \overline{w}_j} \cdot \edge_j \underset{\mbox{\tiny{(iii)}}}{\leq} \frac{2.5}{\alpha} \cdot \edge_j \leq \frac{2.5}{2.7} \cdot \edge_j \defeq \zeta \cdot \edge_j,\label{chainINEQ}
  \end{eqnarray}
  and so (\textbf{O1}) is implied by the weak learning assumption. 
  To summarise, \pmboost~boosts the convex surrogate of the $\alpha$-loss without either computing it or its derivative, and achieves boosting compliant convergence using only the classical assumptions of boosting, \textbf{(R, WLA)}.
  The proof of Lemma \ref{edgeALPHA} being very conservative, we can expect that the smallest value of $K$ of interest is smaller than the one we use, indicating that \eqref{chainINEQ} should hold for substantially smaller limit values in (ii, iii).

\subsection{Proof of Lemma \ref{edgeALPHA}}\label{proof-lem-edgeALPHA}

Define for short
\begin{eqnarray}
  F(u) & \defeq & \left(\frac{u^\alpha}{u^\alpha+(1-u)^\alpha}\right)^{\holdera} -\left(\frac{(1-u)^\alpha}{u^\alpha+(1-u)^\alpha}\right)^{\holdera}\\
G(u) & \defeq & 1-\left(\frac{2\cdot (1-u)^\alpha}{u^\alpha+(1-u)^\alpha}\right)^{\frac{1}{\holdera}},
\end{eqnarray}
that we study for $u\geq 1/2$ (the bound also holds by construction for $u<1/2$). Define the following functions:
\begin{eqnarray}
  g(u) & \defeq & 1 - (2u) ^{\frac{1}{\holdera}},\\
  h(z) & \defeq & \frac{1}{1+z},\\
  i_{\alpha}(u) & \defeq & u^{\alpha},
\end{eqnarray}
and $u_\alpha \defeq h \circ i_{-\alpha} \circ h^{-1}(u)$, $f(u) \defeq i_{\holdera}(1-u) - i_{\holdera}(u)$. We remark that $g$ is convex if $\alpha \geq 1$ while $f$ is concave. Both derivatives match in $1/2$ if
\begin{eqnarray}
(\holdera)^2 2^{1-\holdera} & = & 1,
\end{eqnarray}
whose roots are $\holdera <6$. It means if $\alpha \geq 6/5 = 1.2$, then $(g-f)' \geq 0$, and so if we measure
\begin{eqnarray*}
k^* & \defeq & \arg\sup_k \sup_{x, x' : g(x) = f(x') = k} |x-x'|,
  \end{eqnarray*}
  then $k^*$ is obtained for $x=1$, for which $g(x) = 1 - 2^{\frac{1}{\holdera}} = k^*$. We then need to lowerbound $x'$ such that $f(x') = 1 - 2^{\frac{1}{\holdera}}$, which amounts to finding $x^*$ such that $f(x^*) \geq 1 - 2^{\frac{1}{\holdera}}$, since $f$ is strictly decreasing. Fix
  \begin{eqnarray}
    x^* & \defeq & 1-\frac{K}{\alpha},
  \end{eqnarray}
  A series expansion reveals that for $x=x^*$ and $K = \log 2$,
  \begin{eqnarray}
f(x^*) & = & g(x^*) + O\left(\frac{1}{\alpha^2}\right),
  \end{eqnarray}
  and we thus get that there exists $K \geq \log 2$ such that
  \begin{eqnarray}
\sup_k \sup_{x, x' : g(x) = f(x') = k} |x-x'| & \leq & \frac{K}{\alpha},
  \end{eqnarray}
  or similarly for any ordinate value, the difference between the abscissae giving the value for $f$ and $g$ are distant by at most $K/\alpha$. The exact value of the constant is not so much important than the dependence in $1/\alpha$: we now plug this in the $u_\alpha$s notation and ask the following question: suppose $f(u_\alpha) = g(v_\alpha) = k$. Since $|u_\alpha - v_\alpha| \leq K/\alpha$, what is the maximum difference $|u-v|$ as a function of $\alpha$ ? We observe
\begin{eqnarray}
\frac{\partial}{ \partial u} u_\alpha & = & -\frac{\alpha (u(1-u))^{\alpha-1}}{(u^\alpha + (1-u)^\alpha)^2},\\
\frac{\partial^2}{ \partial u^2} u_\alpha & = & \alpha \cdot \frac{(u(1-u))^{\alpha-2}((\alpha-2u+1)u^\alpha - (\alpha+2u-1)(1-u)^\alpha)}{(u^\alpha + (1-u)^\alpha)^3},
\end{eqnarray}
which shows the convexity of $u_\alpha$ as long as
\begin{eqnarray}
\left(\frac{u}{1-u}\right)^\alpha & \geq & \frac{\alpha + 2u -1}{\alpha -2u + 1},
\end{eqnarray}
a sufficient condition for which -- given the RHS increases with $u$ -- is
\begin{eqnarray}
  u & \geq & \frac{\left(\frac{4}{\alpha-1}\right)^{\frac{1}{\alpha}}}{1+\left(\frac{4}{\alpha-1}\right)^{\frac{1}{\alpha}}}.
\end{eqnarray}
Since $u\geq 1/2$, we note the constraint quickly vanishes. In particular, if $\alpha \geq 5$, the RHS is $\leq 1/2$, so $u_\alpha$ is strictly convex. Otherwise, scrutinising the maximal values of the derivative for $\alpha \geq 1$ reveals that if we suppose $v\leq \delta$ for some $\delta$, then $|u-v|$ is maximal for $v= \delta$. So, suppose $v_\alpha = \epsilon$ and we solve for $u_\alpha = K/\alpha + \varepsilon$, which yields
\begin{eqnarray}
  u & = & \frac{\left(1-\frac{K}{\alpha}-\varepsilon\right)^{\frac{1}{\alpha}}}{\left(\frac{K}{\alpha}+\varepsilon\right)^{\frac{1}{\alpha}}+\left(1-\frac{K}{\alpha}-\varepsilon\right)^{\frac{1}{\alpha}}}\\
  & = & \frac{\left((1-\varepsilon)\alpha - K\right)^{\frac{1}{\alpha}}}{\left(K+\varepsilon\alpha\right)^{\frac{1}{\alpha}}+\left((1-\varepsilon)\alpha - K\right)^{\frac{1}{\alpha}}},
\end{eqnarray}
while the $v$ producing the largest $|u-v|$ is:
\begin{eqnarray}
v & = & \frac{(1-\varepsilon)^{\frac{1}{\alpha}}}{\varepsilon^{\frac{1}{\alpha}}+(1-\varepsilon)^{\frac{1}{\alpha}}},
\end{eqnarray}
so
\begin{eqnarray}
|v-u| = (v-u)(\varepsilon) & = & \frac{(1-\varepsilon)^{\frac{1}{\alpha}}}{\varepsilon^{\frac{1}{\alpha}}+(1-\varepsilon)^{\frac{1}{\alpha}}} - \frac{\left((1-\varepsilon)\alpha - K\right)^{\frac{1}{\alpha}}}{\left(K+\varepsilon\alpha\right)^{\frac{1}{\alpha}}+\left((1-\varepsilon)\alpha - K\right)^{\frac{1}{\alpha}}}.
\end{eqnarray}
If we fix
\begin{eqnarray}
\varepsilon & = & \frac{1}{\alpha^4},
\end{eqnarray}
then we get after separate series are computed in $\alpha \rightarrow +\infty$,
\begin{eqnarray}
  |v-u| = (v-u)(\varepsilon) & = &  \frac{\log(1+\log K)}{4\alpha} + O\left(\frac{1}{\alpha^2}\right)\\
  & \lesssim & \frac{0.133}{\alpha}.
\end{eqnarray}
The "forbidden interval" for $v$ is then
\begin{eqnarray}
\left[\frac{(\alpha^4-1)^{\frac{1}{\alpha}}}{1+(\alpha^4-1)^{\frac{1}{\alpha}}}, 1\right] & \approx & \left[\frac{1}{2} + \frac{\log \alpha}{\alpha}, 1\right];
\end{eqnarray}
what is more interesting for us is the corresponding forbidden images for $g (v_\alpha)$, which are thus
\begin{eqnarray}
  g &\not\in& \holdera \cdot \left[1-\frac{1}{\alpha^4}, 1\right] \defeq \mathbb{I}_\alpha,
\end{eqnarray}
where we use shorthand $z\cdot [a, b] \defeq [az, bz]$. This, we note, translates directly into observable edges since $g$ is the function we invert. Summarizing, we have shown that if (i) $\alpha \geq 1.2$ then for any $u,v$ such that $F(u) = G(v) \not\in \mathbb{I}_\alpha$, then $|u-v| \lesssim 0.133 / \alpha$. It suffices to remark that $\mathbb{I}_\alpha$ represents the set of forbidden weights to get the statement of the Lemma.

\subsection{Proof of Theorem \ref{genBOOST}}\label{app-proof-genBOOST}
\begin{remark} 
\pmboost~and its convergence analysis bring a side contribution of ours: it is impossible to exactly encode in standard machine types the inverse link of losses like the log loss, so the implementation of classical boosting algorithms \cite{fGFA,fhtAL} can only rely on approximations of the inverse link function.
  Our results yield convergence guarantees for the \textit{implementations} of such algorithms, and \textbf{(O1)} can be interpreted and checked in the context of machine encoding. 
Two additional remarks hold regarding convergence rate: first, the $1/\upgamma^2$ dependence meets the general optimum for boosting \cite{aghmBS}; second, the $1/\varepsilon^2$ dependence parallels classical training convergence of convex optimization \cite{tjnoPE} (and references therein). There is however a major difference with such work: \pmboost~requires \textit{no} function oracles for $F$ (function values, (sub)gradients, etc.).
This ``\textit{sideways}'' fork to minimizing $F$ pays (only) a $1/(1-\zeta)^2$ factor in convergence. 
\end{remark}

\newcommand{\lowerb}{\triangle}

We proceed in two steps, assuming \textbf{(WLA)} holds for \weak~and \textbf{(R)} holds for the weak classifiers.\\

\noindent \textbf{In Step 1}, we show that for any loss defined by $F$ twice differentiable, convex and non-increasing, for any $z^* \in \mathbb{R}$, as long as $F$ satisfies assumptions \textbf{(O1)} and \textbf{(O2)} for $T$ iterations such that
\begin{eqnarray}
 \sum_{t=0}^{T} \tilde{w}^2_t & \geq & \frac{2\surlossprime^* (\surlossprime(0) - \surlossprime(z^*))}{\upgamma^2 (1-\zeta)^2(1-\pi^2)},
\end{eqnarray}
we have the guarantee on the risk defined by $F$:
\begin{eqnarray}
\expect_{i\sim \mathcal{S}}\left[ \surloss (y_i H_{T}(\ve{x}_i))\right] & \leq & F(z^*).
  \end{eqnarray}

Let $F$ be any twice differentiable, convex and non-increasing function. We wish to find a lowerbound $\lowerb$ on the decrease of the expected loss computed using $F$:
\begin{eqnarray}
\expect_{i\sim \mathcal{S}}\left[ \surloss (y_i H_{t}(\ve{x}_i))\right] - \expect_{i\sim \mathcal{S}}\left[ \surloss (y_i H_{t+1}(\ve{x}_i))\right] & \geq & \lowerb,
\end{eqnarray}
where with a slight abuse of notation we let $H_{t}$ denote the learned real-valued classifier at iteration $t$.
We make use of a similar 
proof technique as in \cite[Theorem 7]{nwLO}. Suppose
\begin{eqnarray}
  H_{t+1} & = & H_t + \beta_j \cdot h_j,
\end{eqnarray}
index $j$ being returned by \weak~at iteration $t$. For any such index $j$, any $g : \mathbb{R} \rightarrow \mathbb{R}_+$ and any $H \in \mathbb{R}^{\mathcal{X}}$, let
\begin{eqnarray}
\edge(j,g,H) & \defeq & \expect_{i\sim \mathcal{S}} \left[y_i h_j(\ve{x}_i) \cdot g(y_i H(\ve{x}_i))\right],
  \end{eqnarray}
denote the expected edge of $h_j$ on weights defined by the couple $(g, H)$. 
Furthermore, let 
  \begin{eqnarray}
\Delta (g_1, g_2) & \defeq & \left| \edge(j,g_1,H_t)-\edge(j,g_2,H_t) \right|,
  \end{eqnarray}
  denote the discrepancy between two expected edges defined by $g_{1}, g_{2}$, respectively. 

There are two quantities we consider. First, let
  \begin{eqnarray}
    X & \defeq & \expect_{i\sim \mathcal{S}} \left[(y_i H_t (\ve{x}_i) - y_i H_{t+1} (\ve{x}_i)) \surlossprime'(y_i H_t (\ve{x}_i))\right]\\
    & = & \beta_j \cdot \expect_{i\sim \mathcal{S}} \left[ y_i h_j(\ve{x}_i) \cdot -\surlossprime'(y_i H_t (\ve{x}_i))\right]\\
    & = & \beta_j \cdot \edge(j, -\surlossprime', H_t) \\
    & \geq & \beta_j \cdot \expect_{i\sim \mathcal{S}} \left[ y_i h_j(\ve{x}_i) \cdot \pseudoinvlink(-y_i H_{t}(\ve{x}_i))\right] - \beta_j \cdot \Delta (-\surlossprime', \sympseudoinvlink)\\
      &  = & a_{f}\edge^2(j, \sympseudoinvlink,H_t) - a_{f}\edge (j, \sympseudoinvlink,H_t) \cdot \Delta(-\surlossprime', \sympseudoinvlink)\\
    & \geq & a_{f} (1-\zeta) \edge^2(j, \sympseudoinvlink,H_t),\label{eqassumptionD}
  \end{eqnarray}
  where $\sympseudoinvlink(z) \defeq \sympseudoinvlink(-z)$ and
finally \eqref{eqassumptionD} makes use of assumption \textbf{(O1)}. The second quantity we define is:
  \begin{eqnarray}
Y(\mathcal{Z}) & \defeq & \expect_{i\sim \mathcal{S}}\left[ (y_i
                                             H_{t}(\ve{x}_i) - y_i
                                             H_{t+1}(\ve{x}_i))^2
     \surloss ''(z_i)\right] \label{defYY},
\end{eqnarray}
where $\mathcal{Z} \defeq \{z_1, z_2, ..., z_m\} \subset
\mathbb{R}^m$. Using assumption \textbf{(R)} and letting $\surlossprime^*$ be any real such that $\surlossprime^* \geq \sup \surlossprime'' (z)$, we obtain:
  \begin{eqnarray}
    Y(\mathcal{Z}) & \leq &\surlossprime^* \cdot \expect_{i\sim \mathcal{S}}\left[ (y_i
                                             H_{t}(\ve{x}_i) - y_i
                            H_{t+1}(\ve{x}_i))^2\right]\nonumber\\
    &  = & \surlossprime^* \cdot \beta_j^2 \cdot \expect_{i\sim \mathcal{S}}\left[ (y_i h_j(\ve{x}_i))^2\right]\nonumber\\
            & \leq &                \surlossprime^*\cdot \beta_j^2 \cdot  M^2\nonumber\\
    &    = & \surlossprime^* a^2   M^2 \cdot \edge^2(j, \sympseudoinvlink,H_t).
  \end{eqnarray}
A second order Taylor expansion on $\surloss$ brings that there exists $\mathcal{Z} \defeq \{z_1, z_2, ..., z_m\} \subset
\mathbb{R}^m$ such that:
\begin{eqnarray}
  \expect_{i\sim \mathcal{S}}\left[ \surloss (y_i H_{t}(\ve{x}_i))\right] & = & \expect_{i\sim \mathcal{S}}\left[ \surloss (y_i H_{t+1}(\ve{x}_i))\right] + \expect_{i\sim \mathcal{S}} \left[(y_i H_t (\ve{x}_i) - y_i H_{t+1} (\ve{x}_i)) \surlossprime'(y_i H_t (\ve{x}_i))\right] \nonumber\\
  & & + \expect_{i\sim \mathcal{S}}\left[ (y_i
                                             H_{t}(\ve{x}_i) - y_i
                                             H_{t+1}(\ve{x}_i))^2
     \cdot \frac{\surloss ''(z_i)}{2}\right] ,
 \end{eqnarray} 
  So,
  \begin{eqnarray}
    \expect_{i\sim \mathcal{S}}\left[\surlossprime(y_i H_t(\ve{x}_i))\right]-\expect_{i\sim \mathcal{S}}\left[\surlossprime(y_i H_{t+1}(\ve{x}_i))\right] & = & X - \frac{Y(\mathcal{Z})}{2} \nonumber\\
                                                                                                                                                    & \geq & \underbrace{\left(1 - \zeta - \frac{\surlossprime^* a M^2}{2}\right) a}_{\defeq Z(a)} \cdot \edge^2(j, \sympseudoinvlink,H_t).
  \end{eqnarray}
  Suppose we fix $\pi \in [0,1]$ and choose any
  \begin{eqnarray}
a & \in & \frac{1-\zeta}{\surlossprime^* M^2} \cdot \left[1-\pi, 1+\pi\right].\label{consta}
  \end{eqnarray}
  We can check that
  \begin{eqnarray}
Z(a) & \geq & \frac{(1-\zeta)^2(1-\pi^2)}{2\surlossprime^* M^2},
  \end{eqnarray}
  and so
  \begin{eqnarray}
\expect_{i\sim \mathcal{S}}\left[\surlossprime(y_i H_t(\ve{x}_i))\right]-\expect_{i\sim \mathcal{S}}\left[\surlossprime(y_i H_{t+1}(\ve{x}_i))\right] & \geq & \frac{(1-\zeta)^2(1-\pi^2)}{2\surlossprime^* M^2}\cdot \edge^2(j, \sympseudoinvlink,H_t) .
  \end{eqnarray}
  So, taking into account that for the first classifier, we have $\expect_{i\sim \mathcal{S}}\left[\surlossprime(y_i H_0(\ve{x}_i))\right] = \surlossprime(0)$, if we take any $z^* \in \mathbb{R}$ and we boost for a number of iterations $T$ satisfying (we use notation $\edge_t$ as a summary for $\edge^2(j, \sympseudoinvlink,H_t)$ with respect to \pmboost):
  \begin{eqnarray}
\sum_{t=1}^{T} \edge_t^2 & \geq & \frac{2\surlossprime^* M^2 (\surlossprime(0) - \surlossprime(z^*))}{(1-\zeta)^2(1-\pi^2)},\label{beta1}
  \end{eqnarray}
  then $\expect_{i\sim \mathcal{S}}\left[\surlossprime(y_i H_{T}(\ve{x}_i))\right] \leq \surlossprime(z^*)$. We now assume \textbf{(WLA)} holds, the LHS of \eqref{beta1} is $\geq T \upgamma^2$. Given that we choose $a = a_f$ in \pmboost, we need to make sure \eqref{consta} is satisfied for the loss defined by $\surloss$, which translates to 
\begin{eqnarray}
\surlossprime^* & \in & \frac{1-\zeta}{a_f M^2} \cdot \left[1-\pi, 1+\pi\right], \label{constF}
  \end{eqnarray}
  and defines assumption \textbf{(O2)}.  To complete Step 1, we normalize the edge. Letting $\tilde{w}_i \defeq w_i / \sum_k w_k$, $\tilde{w}_t \defeq \ve{1}^\top \ve{w}_t /m$ and
  \begin{eqnarray}
\tilde{\edge}_t & \defeq & \frac{\edge_t}{\tilde{w}_t} \in [-M, M],
  \end{eqnarray}
which is then properly normalized and such that \eqref{beta1} becomes equivalently:
  \begin{eqnarray}
\sum_{t=0}^{T} \tilde{w}^2_t \tilde{\edge}_t^2 & \geq & \frac{2\surlossprime^* M^2 (\surlossprime(0) - \surlossprime(z^*))}{(1-\zeta)^2(1-\pi^2)},\label{beta2}
  \end{eqnarray}
  and so under the (weak learning) assumption on $\tilde{\edge}_t$ that $|\tilde{\edge}_t | \geq \upgamma \cdot M$, a sufficient condition for \eqref{beta2} is then
  \begin{eqnarray}
    \sum_{t=0}^{T} \tilde{w}^2_t & \geq & \frac{2\surlossprime^* (\surlossprime(0) - \surlossprime(z^*))}{\upgamma^2 (1-\zeta)^2(1-\pi^2)},\label{beta3}
  \end{eqnarray}
  completing step 1 of the proof.\\

  \noindent \textbf{In Step 2}, we show a result on the distribution of edges, \textit{i.e.} margins. 
  \eqref{beta3} contains all the intuition about how the rest of the proof unfolds, as we have two major steps: in step 2.1,  we translate the guarantee of \eqref{beta3} on margins, and in step 2.2, we translate the ``margin'' based \eqref{beta3} in a readable guarantee in the boosting framework (we somehow ``get rid'' of the $\tilde{w}^2_t$ in the LHS of \eqref{beta3}).\\

  \noindent \textbf{Step 2.1}. Let $\mathcal{Z} \defeq \{z_1, z_2, ..., z_m\} \subset \mathbb{R}$ a set of reals. Since $\surlossprime$ is non-increasing, we have $\forall u\in [0,1], \forall \theta \geq 0$,
  \begin{eqnarray}
    \pr_{i}[z_i \leq \theta] > u \Rightarrow  \expect_i [\surlossprime(z_i)] & > & (1-u) \inf_z \surlossprime(z)  + u \surlossprime(\theta)\nonumber\\
    & & \defeq (1-u) \surlossprimemin + u \surlossprime(\theta),\label{prtoexpect}
  \end{eqnarray}
  so if we pick $z^*$ in \eqref{beta3} such that
  \begin{eqnarray}
\surlossprime(z^*) & \defeq &  (1-u) \surlossprimemin + u \surlossprime(\theta),\label{valuezstar}
  \end{eqnarray}
  then \eqref{beta3} implies $\expect_{i\sim \mathcal{S}}\left[\surlossprime(y_i H_{T}(\ve{x}_i))\right] \leq (1-u) \surlossprimemin + u \surlossprime(\theta)$ and so by the contraposition of \eqref{prtoexpect} yields:
  \begin{eqnarray}
\pr_{i\sim \mathcal{S}}\left[y_i H_{T}(\ve{x}_i) \leq \theta\right] & \leq & u, \label{margineq}
  \end{eqnarray}
  which yields our margin based guarantee.\\

  \noindent \textbf{Step 2.2}. At this point, the key (in)equalities are \eqref{beta3} (for boosting) and \eqref{margineq} (for margins). Fix $\kappa > 0$. We have two cases:
  \begin{itemize}
  \item Case 1: $\tilde{w}_t$ never gets too small, say $\tilde{w}_t \geq \kappa, \forall t\geq 0$. In this case, granted the weak learning assumption holds on $\tilde{\edge}_t$, \eqref{beta3} yields a direct lowerbound on iteration number $T$ to get $\pr_{i\sim \mathcal{S}}\left[y_i H_{T}(\ve{x}_i) \leq \theta\right] \leq u$;
    \item  Case 2: $\tilde{w}_t \leq \kappa$ at some iteration $t$. Since the smaller it is, the better classified are the examples, if we pick $\kappa$ small enough, then we can get $\pr_{i\sim \mathcal{S}}\left[y_i H_{T}(\ve{x}_i) \leq \theta\right] \leq u$ ``straight''.
    \end{itemize}
    This suggests our use of the notion of ``denseness'' for weights \cite{bcsEN}.
    \begin{definition}
The weights at iteration $t$ is called $\kappa$-dense iff $\tilde{w}_t \geq \kappa$.
\end{definition}
We now have the following Lemma.
\begin{lemma}\label{lemKd}
  For any $t\geq 0, \theta \in \mathbb{R}, \kappa>0$, if weights produced in Step 2.1 of \pmboost~fail to be $\kappa$-dense, then
  \begin{eqnarray}
\pr_{i\sim \mathcal{S}}\left[y_i H_{T}(\ve{x}_i) \leq \theta\right] & \leq & \frac{\kappa}{\pseudoinvlink(-\theta)}.\label{bsuppr}
    \end{eqnarray}
  \end{lemma}
  \begin{proof}
    Let $\mathcal{Z} \defeq \{z_1, z_2, ..., z_m\} \subset \mathbb{R}$ a set of reals. Since $\pseudoinvlink$ is non-decreasing, we have $\forall \theta \in \mathbb{R}$,
    \begin{eqnarray}
      \expect_i[\pseudoinvlink(z_i)] & \geq & \pr_i[z_i < - \theta]\cdot \inf_z \pseudoinvlink(z) + \pr_i[z_i \geq - \theta]\cdot \pseudoinvlink(-\theta)\nonumber\\
      & &  = \pr_i[z_i \geq - \theta]\cdot \pseudoinvlink(-\theta)
    \end{eqnarray}
    since by assumption $\inf \pseudoinvlink = 0$. Pick $z_i \defeq -y_i H_{T}(\ve{x}_i)$. We get that if $\pr_{i\sim \mathcal{S}} [-y_i H_{T}(\ve{x}_i) \geq -\theta ] = \pr_{i\sim \mathcal{S}} [y_i H_{T}(\ve{x}_i) \leq \theta ] \geq \xi$, then $\tilde{w}_t \defeq \expect_{i\sim \mathcal{S}}[\pseudoinvlink(-y_i H_{T}(\ve{x}_i))] \geq \xi \cdot \pseudoinvlink(-\theta)$. If we fix
    \begin{eqnarray}
\xi & = & \frac{\kappa}{\pseudoinvlink(-\theta)},
      \end{eqnarray}
then $\tilde{w}_t < \kappa$ implies \eqref{bsuppr}, which ends the proof of Lemma \ref{lemKd}.
  \end{proof}
  From Lemma \ref{lemKd}, we let $\kappa \defeq \xi_* \cdot \pseudoinvlink(-\theta)$ and $u \defeq \xi_*$ in \eqref{margineq}. If at any iteration, $H_{T}$ fails to be $\kappa$-dense, then $\pr_{i\sim \mathcal{S}}\left[y_i H_{\ve{\beta}}(\ve{x}_i) \leq \theta\right] \leq \xi_*$ and classifier $H_{\ve{\beta}}$ satisfies the conditions of Theorem \ref{genBOOST} (this is our Case 2 above).

  Otherwise, suppose it is always $\kappa$-dense (this is our Case 1 above). We then have at any iteration $T$ $\sum_{t<T} \tilde{w}_t^2 \geq T \xi_*^2 \cdot \pseudoinvlink^2(-\theta)$ and so a sufficient condition to get \eqref{beta3} is then $T \geq \frac{2\surlossprime^* (\surlossprime(0) - \surlossprime(z^*))}{\xi_*^2 \pseudoinvlink^2(-\theta)\upgamma^2 (1-\zeta)^2(1-\pi^2)}$,
where we recall $z^*$ is chosen so that $\surlossprime(z^*)  = (1-\xi_*) \surlossprimemin + \xi_* \surlossprime(\theta)$. This ends the proof of Theorem \ref{genBOOST} (with the change of notation $\xi_* \leftrightarrow \varepsilon$).

\section{Additional Experimental Results} \label{sec:additionalexperiments}

In this section, we provide additional experimental results and discussion to accompany Section~\ref{sec:experiments} in the main text.

\subsection{General Details}


Most of the experiments were performed over the course of a month on a $2015$ MacBook Pro with a $2.2$ GHz Quad-Core Intel Core $i7$ processor and $16$GB of memory.
The Adaptive $\alpha$ experiments were perfomred on a computing cluster and each required about $30$ minutes of compute time.
Code can be found in \textit{PILBoostExperiments.py}, \textit{AdaptiveAlphaMenon.py}, \textit{AdaptiveAlphaALL.py}.
Averaged experiments employed $10$-fold cross validation, and when twisters were present, randomization occurred over the twisted samples as well.
All algorithms across all experiments ran for $1000$ iterations.

\subsection{Discussion of $a_{f}$ and $\alpha$}

In general, we found that for most experiments, $0.1 \leq a_{f} \leq 15$.
From the theory, we know that if $a_{f}$ is too small, boosting needs to occur for a very long time, and if $a_{f}$ is too large, almost no loss fits to O2 (equivalently, O2 fails for us).
We also generally found that \pmboost~was not particularly sensitive to the choice of $a_{f}$ as long as it was in the ``right ballpark'', hence our use of integer or rational values of $a_{f}$ for all experiments. 
When there is twist present, we found that $\alpha > 1$ performed
best, where $\alpha^{*}$ increased as the amount of twist increased
(both observations are conistent with our theory, see for example Lemma 3.4). 
Regarding the relationship between $a_{f}$ and $\alpha$, this appeared to depend on the dataset and depth of the decision trees.

\subsection{Random Class Noise Twister} \label{sec:additionalclassnoise}

\begin{table}[H]
\begin{center}
\begin{tabular}{ccccc}
                        &               &               &           &                \\
\multirow{1}{*}{Dataset} &  \multirow{1}{*}{Algorithm}             &                \multicolumn{3}{c}{Random Class Noise Twister}                               \\ \cline{3-5} 
                         &                     & $p=$ $0$  & $0.15$ & $0.3$                  \\ \hhline{=====}
                         & AdaBoost            & $0.966 \pm 0.015$             & $0.905 \pm 0.027$ & $0.856 \pm 0.033$     \\ \cline{2-5} 
                         & us ($\alpha = 1.1$) & $0.944 \pm 0.029$             & $0.912 \pm 0.013$ & $0.861 \pm 0.042$                            \\ \cline{2-5} 
                         & us ($\alpha = 2.0$) & $0.956 \pm 0.018$             & $0.938 \pm 0.017$ & $0.905 \pm 0.039$                          \\ \cline{2-5} 
\multirow{-3}{*}{cancer}    & us ($\alpha = 4.0$) & $0.957 \pm 0.014$          & $0.917 \pm 0.012$ & $0.922 \pm 0.032$                             \\ \cline{2-5} 
                         & XGBoost             & $0.971 \pm 0.012$             & $0.861 \pm 0.033$ & $0.733 \pm 0.031$                              \\ 
\end{tabular}
\caption{cancer feature random class noise. Accuracies reported for each algorithm and level of twister. Depth one trees. For $\alpha = 1.1$, $a_{f} = 7$, for $\alpha = 2$, $a_{f} = 2$, and for $\alpha = 4$, $a_{f} = 1$.}
\label{table:cancerclassnoise}
\end{center}
\end{table}

\begin{figure}[h]
\begin{center}
\begin{tabular}{ccc}
\hspace{-0.3cm} \includegraphics[trim=20bp 15bp 30bp 15bp,clip,width=0.32\linewidth]{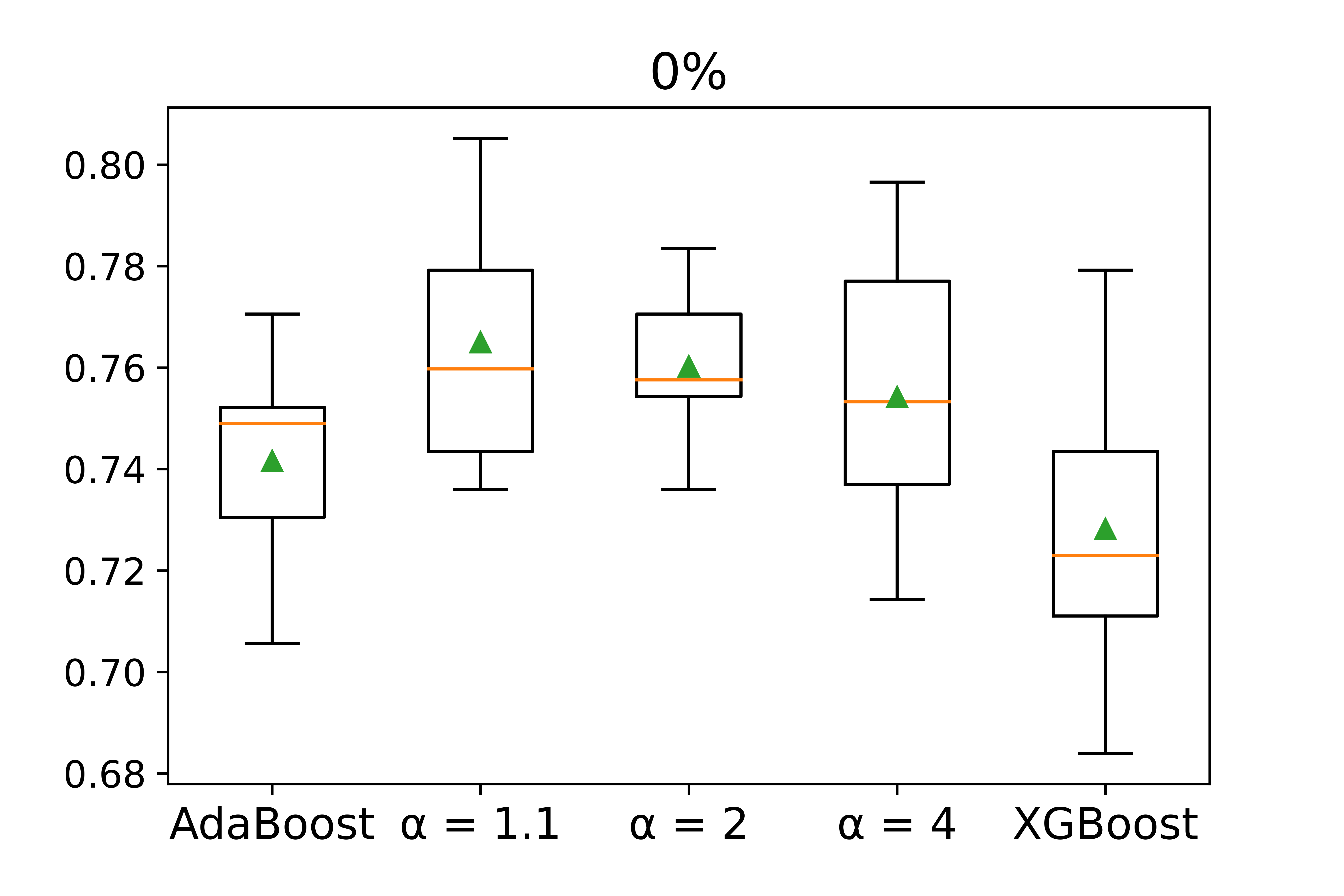}
  \hspace{-0.3cm} & \hspace{-0.2cm} \includegraphics[trim=20bp 15bp 30bp 15bp,clip,width=0.32\linewidth]{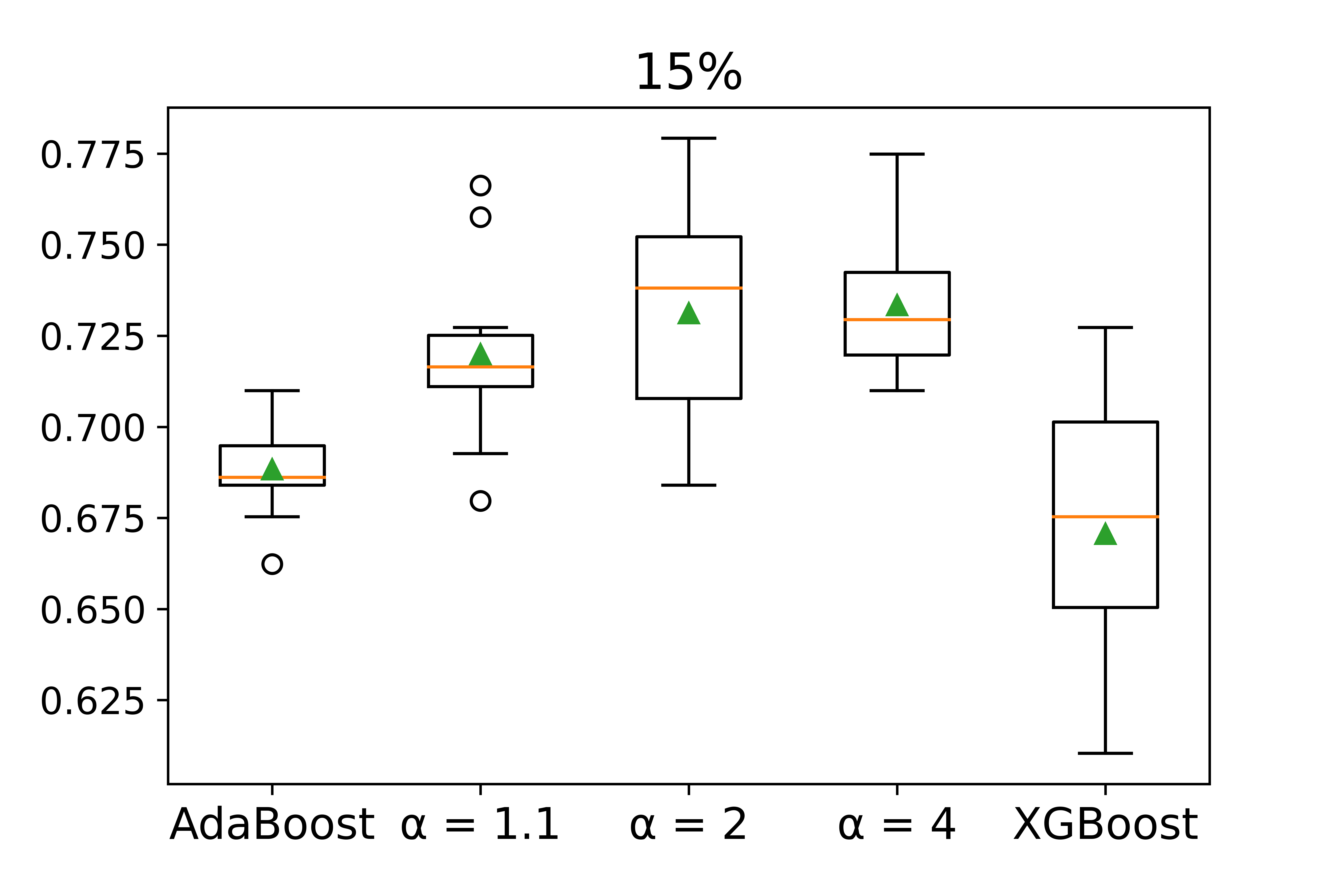}
  & \hspace{-0.3cm} \includegraphics[trim=20bp 15bp 30bp 15bp,clip,width=0.32\linewidth]{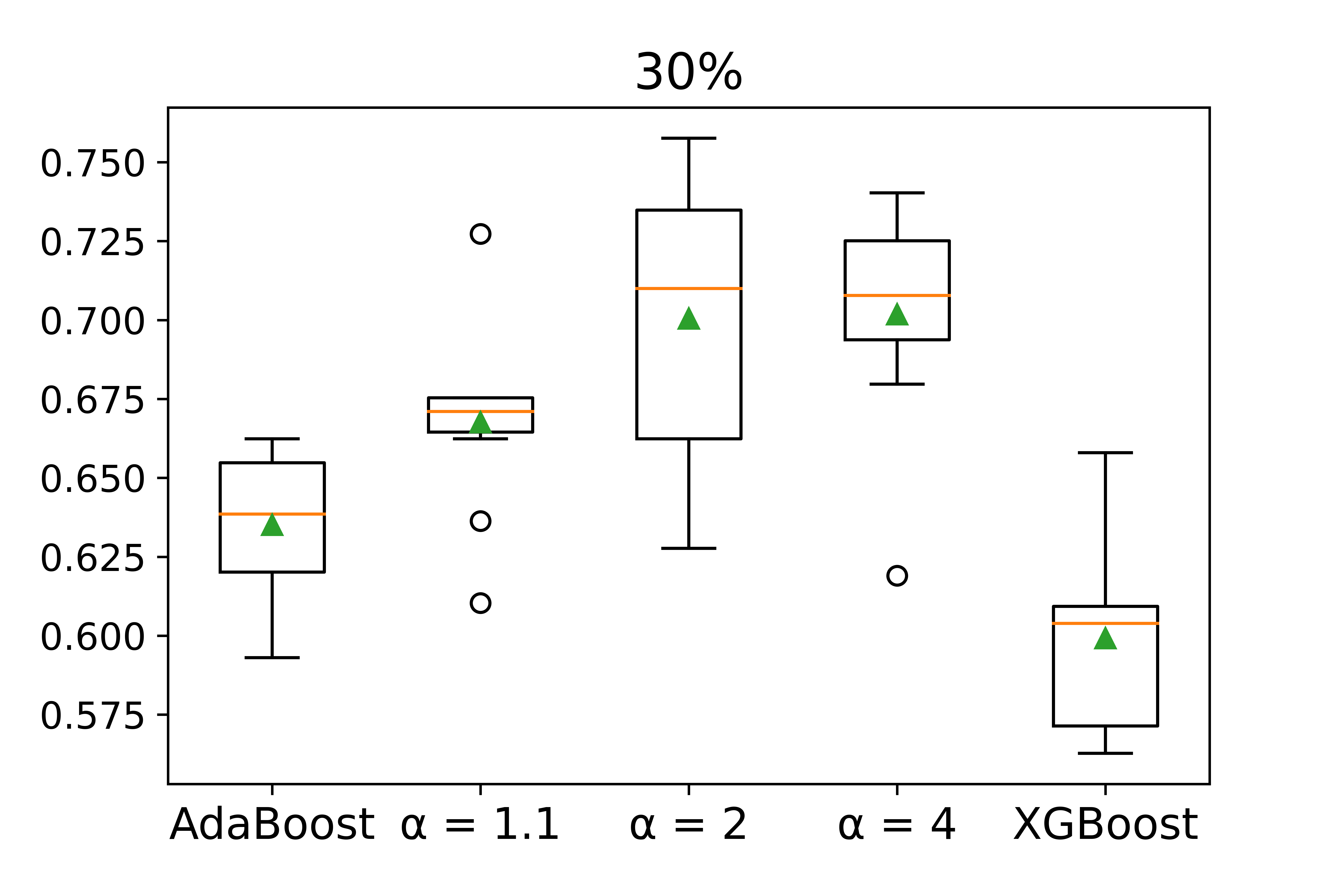}
\end{tabular}
\end{center}
\vspace{-0.2cm}
\caption{Random class noise twister on the diabetes dataset. Depth 3 trees. $a_{f} = 0.1$ for all $\alpha$.}
  \label{fig:classnoisediabetes}
\vspace{-0.2cm}
\end{figure}

\begin{table}[H]
\begin{center}
\begin{tabular}{ccccc}
                        &               &               &           &                \\
\multirow{1}{*}{Dataset} &  \multirow{1}{*}{Algorithm}             &                \multicolumn{3}{c}{Random Class Noise Twister}                               \\ \cline{3-5} 
                         &                     & $p=$ $0$  & $0.15$ & $0.3$                  \\ \hhline{=====}
                         & AdaBoost            & $1.000 \pm 0.000$      & $0.949 \pm 0.016$ & $0.830 \pm 0.043$     \\ \cline{2-5} 
                         & us ($\alpha = 1.1$) & $1.000 \pm 0.000$      & $0.981 \pm 0.013$ & $0.886 \pm 0.033$                            \\ \cline{2-5} 
                         & us ($\alpha = 2.0$) & $1.000 \pm 0.000$      & $0.992 \pm 0.009$ & $0.900 \pm 0.027$                          \\ \cline{2-5} 
\multirow{-3}{*}{xd6}    & us ($\alpha = 4.0$) & $1.000 \pm 0.000$      & $0.999 \pm 0.003$ & $0.927 \pm 0.023$                             \\ \cline{2-5} 
                         & XGBoost             & $1.000 \pm 0.000$      & $0.912 \pm 0.016$ & $0.776 \pm 0.041$                              \\ 
\end{tabular}
\caption{xd6 random class noise. Accuracies reported for each algorithm and level of twister. Depth three trees. $a_{f} = 8$ for all $\alpha$. Note that for $0\%$ noise $\alpha = 4$ used $a_{f} = 0.1$.}
\label{table:xd6classnoise}
\end{center}
\end{table}

\begin{table}[H]
\begin{center}
\begin{tabular}{cccccc}
                        &               &               &           &                   & \\
\multirow{1}{*}{Dataset} &  \multirow{1}{*}{Algorithm}             &                \multicolumn{4}{c}{Random Class Noise Twister}                               \\ \cline{3-6} 
                         &                     & $p=$ $0$  & $0.10$ & $0.20$                 & $0.30$   \\ \hhline{======}
                         & AdaBoost            & $0.902 \pm 0.002$             & $0.900 \pm 0.004$ & $0.898 \pm 0.005$ & $0.894 \pm 0.004$        \\ \cline{2-6} 
                         & us ($\alpha = 1.1$) & $0.901 \pm 0.005$             & $0.899 \pm 0.003$ & $0.897 \pm 0.004$                     & $0.890 \pm 0.004$         \\ \cline{2-6} 
                         & us ($\alpha = 2.0$) & $0.901 \pm 0.004$             & $0.895 \pm 0.004$ & $0.895 \pm 0.003$                      & $0.894 \pm 0.004$           \\ \cline{2-6} 
\multirow{-3}{*}{Online Shopping}  & us ($\alpha = 4.0$) & $0.898 \pm 0.003$   & $0.873 \pm 0.009$ & $0.892 \pm 0.005$                       & $0.889 \pm 0.005$          \\ \cline{2-6} 
                         & XGBoost             & $0.893 \pm 0.005$             & $0.874 \pm 0.002$ & $0.842 \pm 0.006$                     & $0.782 \pm 0.008$           \\ 
\end{tabular}
\caption{Accuracies reported for each algorithm and level of twister. Random training sample selected with probability $p$. Then, for selected training sample, boolean feature flipped with probability $p$ for each feature, independently. Depth three trees. For $\alpha = 1.1$, $a_{f} = 7$, for $\alpha = 2$, $a_{f} = 8$, and for $\alpha = 4$, $a_{f} = 15$. }
\label{table:onlineshoppingclassnoise}
\end{center}
\end{table}


\subsection{Insider Twister} \label{sec:additionalinsidertwister}

\begin{figure}[h]
\begin{center}
\begin{tabular}{ccc}
\hspace{-0.3cm} \includegraphics[trim=0bp 0bp 30bp 15bp,clip,width=0.48\linewidth]{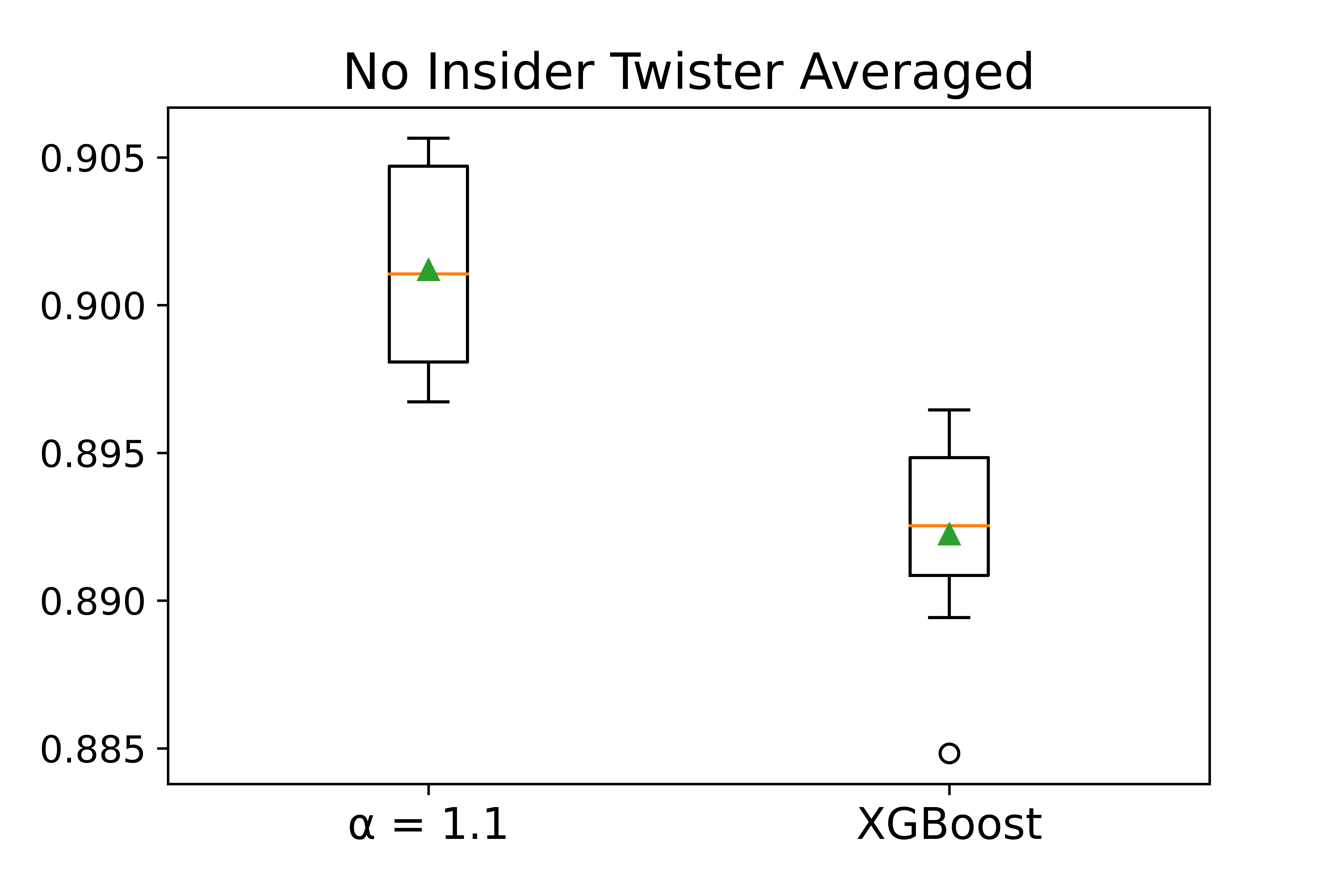}
  \hspace{-0.3cm} & \hspace{-0.2cm} \includegraphics[trim=0bp 0bp 30bp 15bp,clip,width=0.48\linewidth]{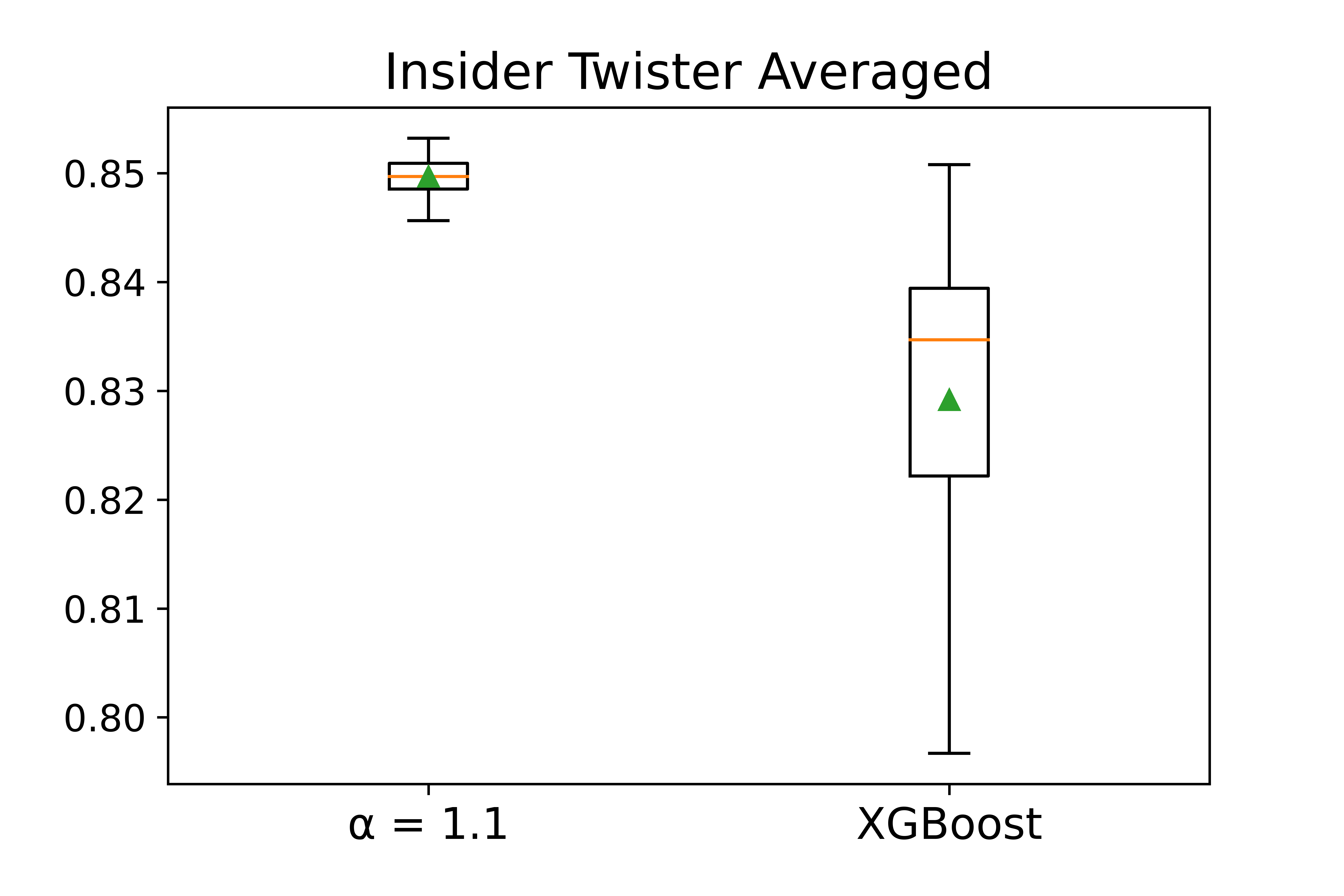}
\end{tabular}
\end{center}
\vspace{-0.2cm}
\caption{Box and whisker visualization of scores associated with Figure~\ref{fig:onlineshoppersinsidertwisterfeatures}. For all insider twister results, we fixed $a_{f} = 7$. Under no twister, $\alpha = 1.1$, has accuracy $0.901 \pm 0.003$, and XGBoost has accuracy $0.892 \pm 0.003$.
Under the insider twister, $\alpha = 1.1$, has accuracy $0.850 \pm 0.002$, and XGBoost has accuracy $0.829 \pm 0.016$; under the Welch t-test, the results have a $p$-value of $0.004$.}
  \label{fig:averagedresultsinsidertwister}
\vspace{-0.2cm}
\end{figure}

\subsection{Discussion of XGBoost} \label{sec:parametersxgboost}

\begin{table}[H]
\begin{center}
\begin{tabular}{cccccc}
                        &               &               &           &                   & \\
&  \multirow{1}{*}{Algorithm}             &                \multicolumn{4}{c}{Average Compute Times}   \\ \cline{3-6} 
                         &                     & cancer  & xd6  & diabetes   & shoppers   \\ \hhline{======}
                         & AdaBoost            & $1.41$  & $0.75$  & $1.11$ & $13.68$        \\ \cline{2-6} 
                         & us ($\alpha = 1.1$) & $2.19$  & $2.01$ & $2.19$  & $30.88$         \\ \cline{2-6} 
                         & us ($\alpha = 2.0$) & $1.11$  & $0.79$  & $2.09$ & $21.85$           \\ \cline{2-6} 
                         & us ($\alpha = 4.0$) & $0.96$  & $1.35$ & $1.82$  & $13.01$          \\ \cline{2-6} 
                         & XGBoost             & $0.29$  & $0.28$ & $0.46$  & $3.16$           \\ 
\end{tabular}
\caption{Average compute times per run ($10$ runs) in seconds across the datasets. Note that the values of $a_{f}$ are chosen identically to choices in Section~\ref{sec:additionalclassnoise}.}
\label{table:computetimes}
\end{center}
\end{table}

XGBoost is a very fast, very well engineered boosting algorithm. 
It employs many different hyperparameters and customizations.
In order to report the fairest comparison between AdaBoost, \pmboost~, and XGBoost, we opted to keep as many hyperparameters fixed (and similar, e.g., depth of decision trees) as possible. 
That being said, it appears that XGBoost inherently uses pruning, so the algorithm pruned while the other two did not. 
Further details regarding three other important points related to XGBoost:
\begin{enumerate}
    \item Please refer to Table~\ref{table:computetimes} for averaged compute times for the three different algorithms. In general, XGBoost had the far faster computation time among the three. However, note that \pmboost~was not particularly engineered for speed. Indeed, we estimate that the computation of $\tilde{f}$ accounts for $40-50\%$ of the total computation time, which we believe can be improved. Thus, we leave the further computational optimization of \pmboost~for future work.
    \item For details regarding regularization, refer to Figure~\ref{fig:insidertwisterregularizedfeatures}, where we report a comparison of regularized XGBoost and \pmboost~such that the training data suffers from the insider twister. We find that regularization improves the ability of XGBoost to combat the twister, but it is not as effective as \pmboost.
    \item For details regarding early stopping, refer to Figure~\ref{fig:insidertwisterearlystoppingfeatures}, where we report a comparison of early-stopped XGBoost (on un-twisted validation data, i.e., cheating) and \pmboost~such that the training data suffers from the insider twister. We find that even early-stopping does not improve XGBoost's ability to combat the insider twister as effectively as \pmboost.
\end{enumerate}


Early stopping - on an untwisted hold-out set contradicts our experiment. With early stopping enabled on a twisted hold-out set, XGBoost generally did not early stop.


\begin{figure}[h]
\begin{center}
\begin{tabular}{ccc}
\hspace{-0.3cm} \includegraphics[trim=0bp 0bp 30bp 15bp,clip,width=0.48\linewidth]{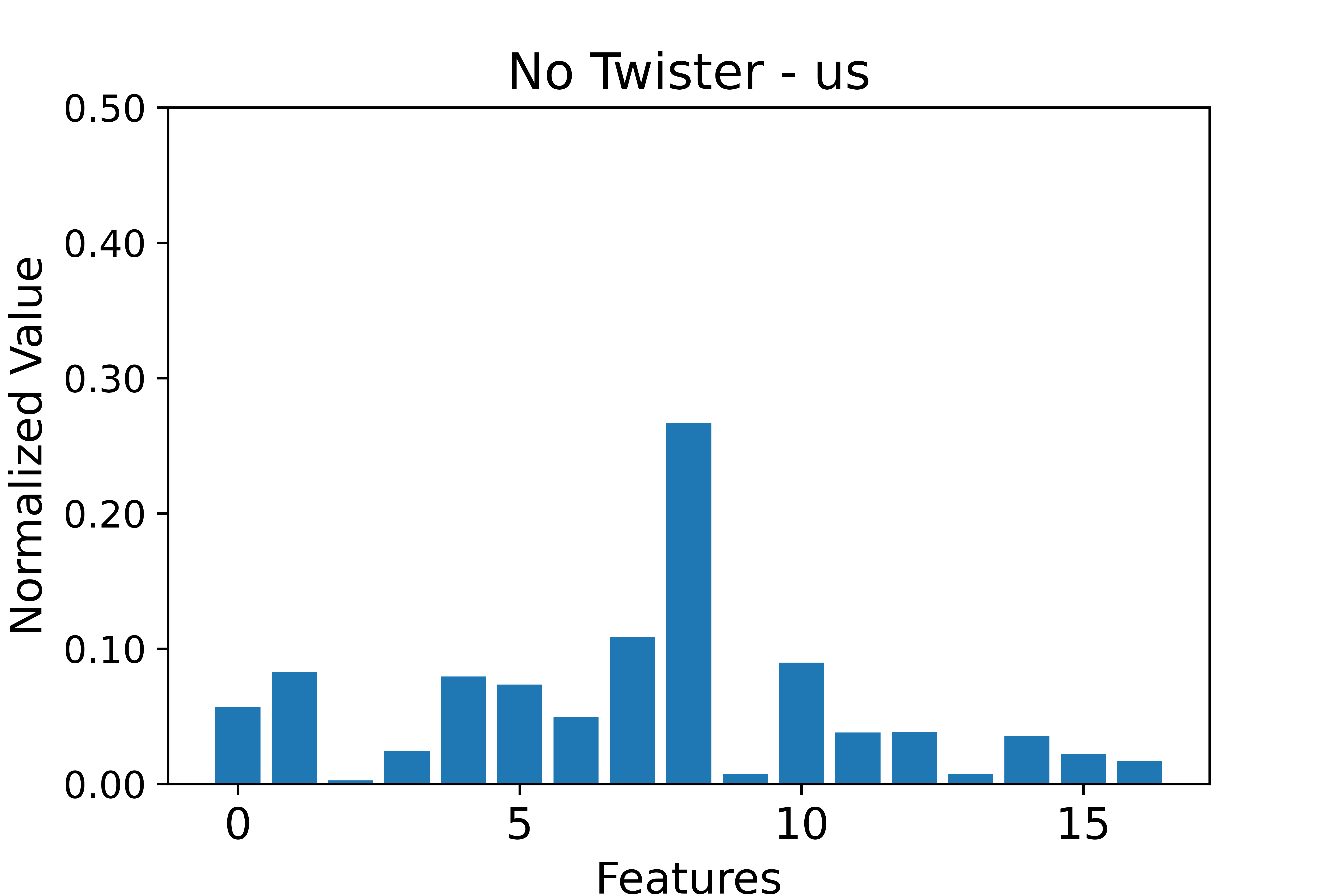}
  \hspace{-0.3cm} & \hspace{-0.2cm} \includegraphics[trim=0bp 0bp 30bp 15bp,clip,width=0.48\linewidth]{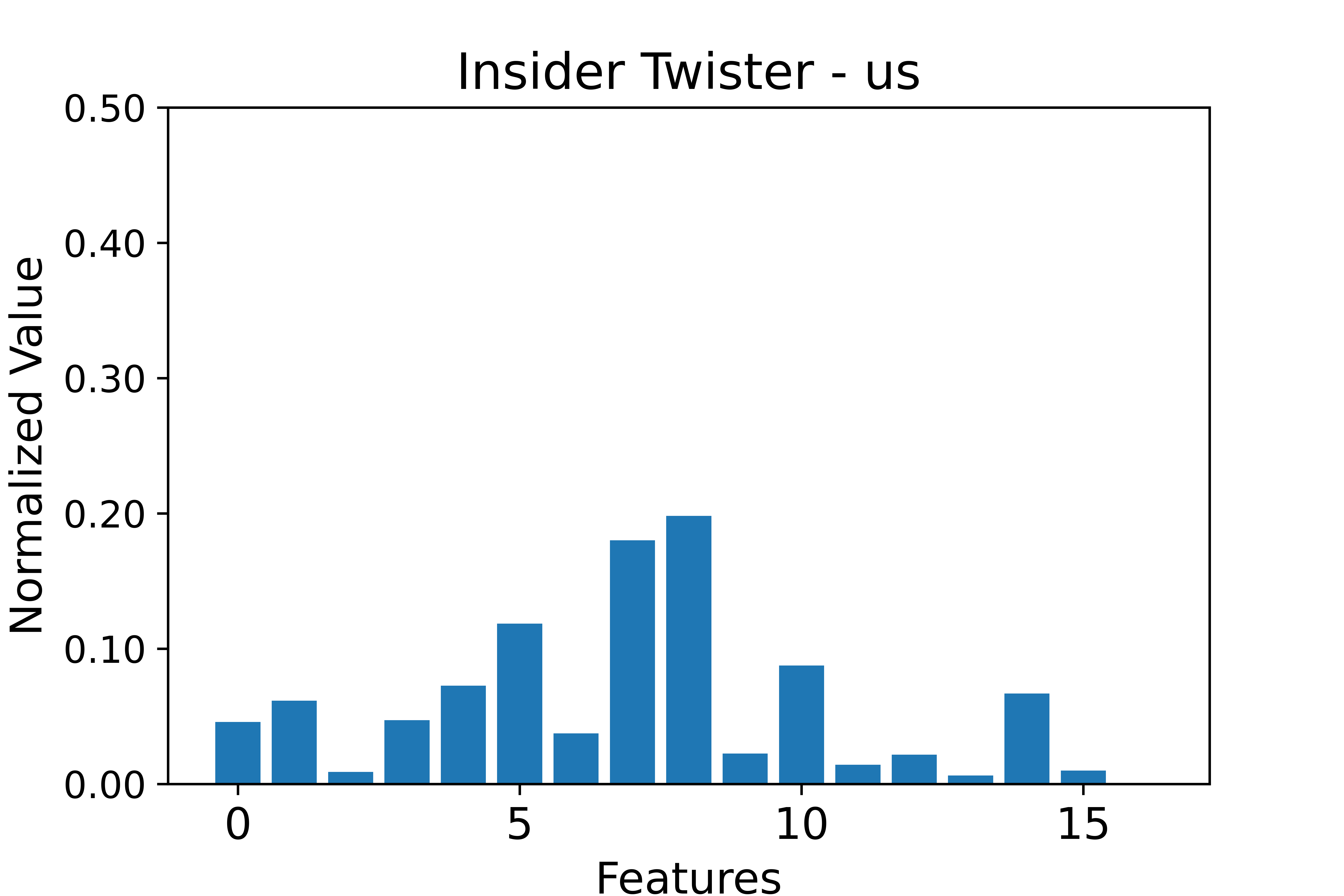}
  \\
\hline
  \hline
  \\
 \hspace{-0.3cm} \includegraphics[trim=0bp 0bp 30bp 15bp,clip,width=0.48\linewidth]{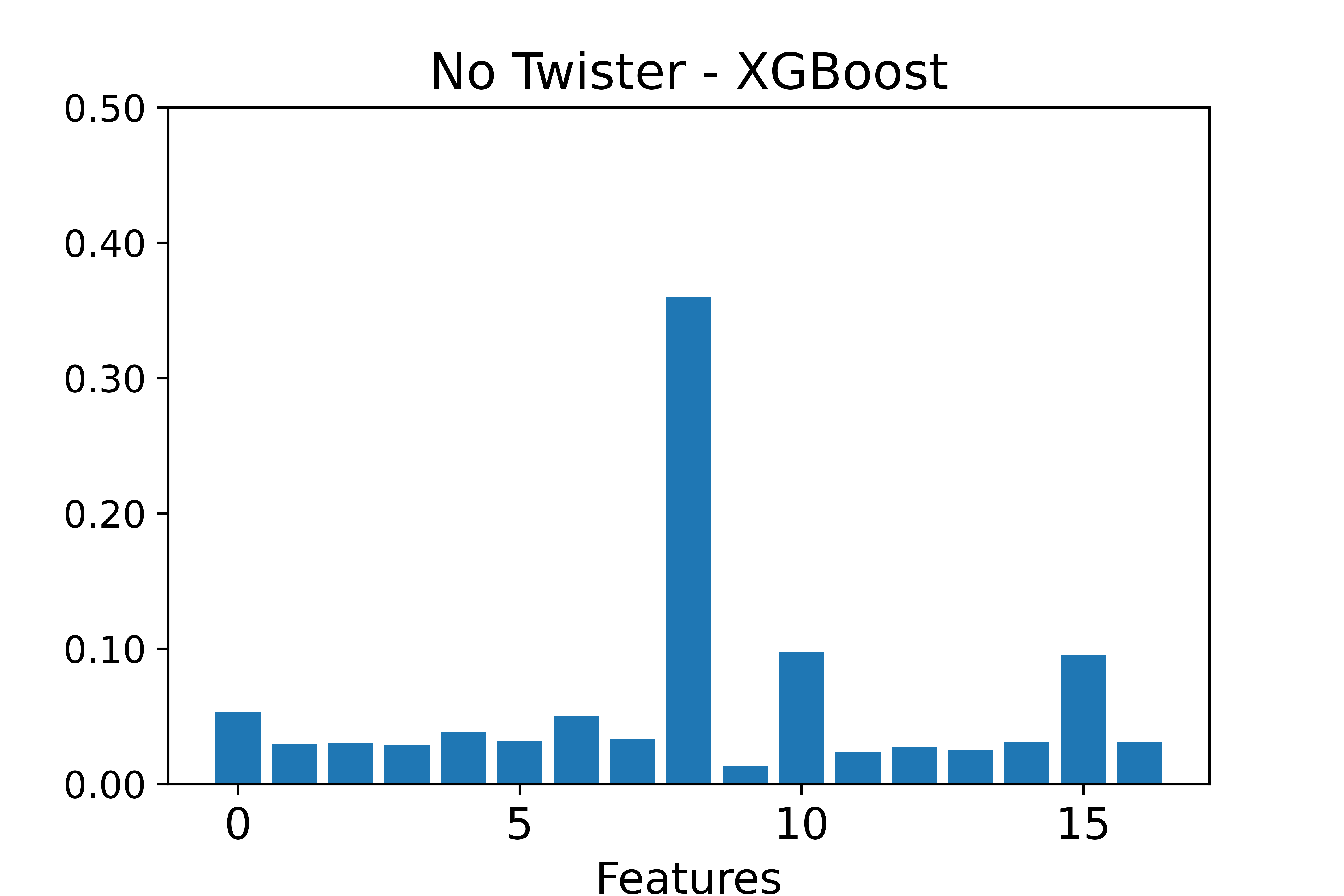}
  \hspace{-0.3cm} & \hspace{-0.2cm} \includegraphics[trim=0bp 0bp 30bp 15bp,clip,width=0.48\linewidth]{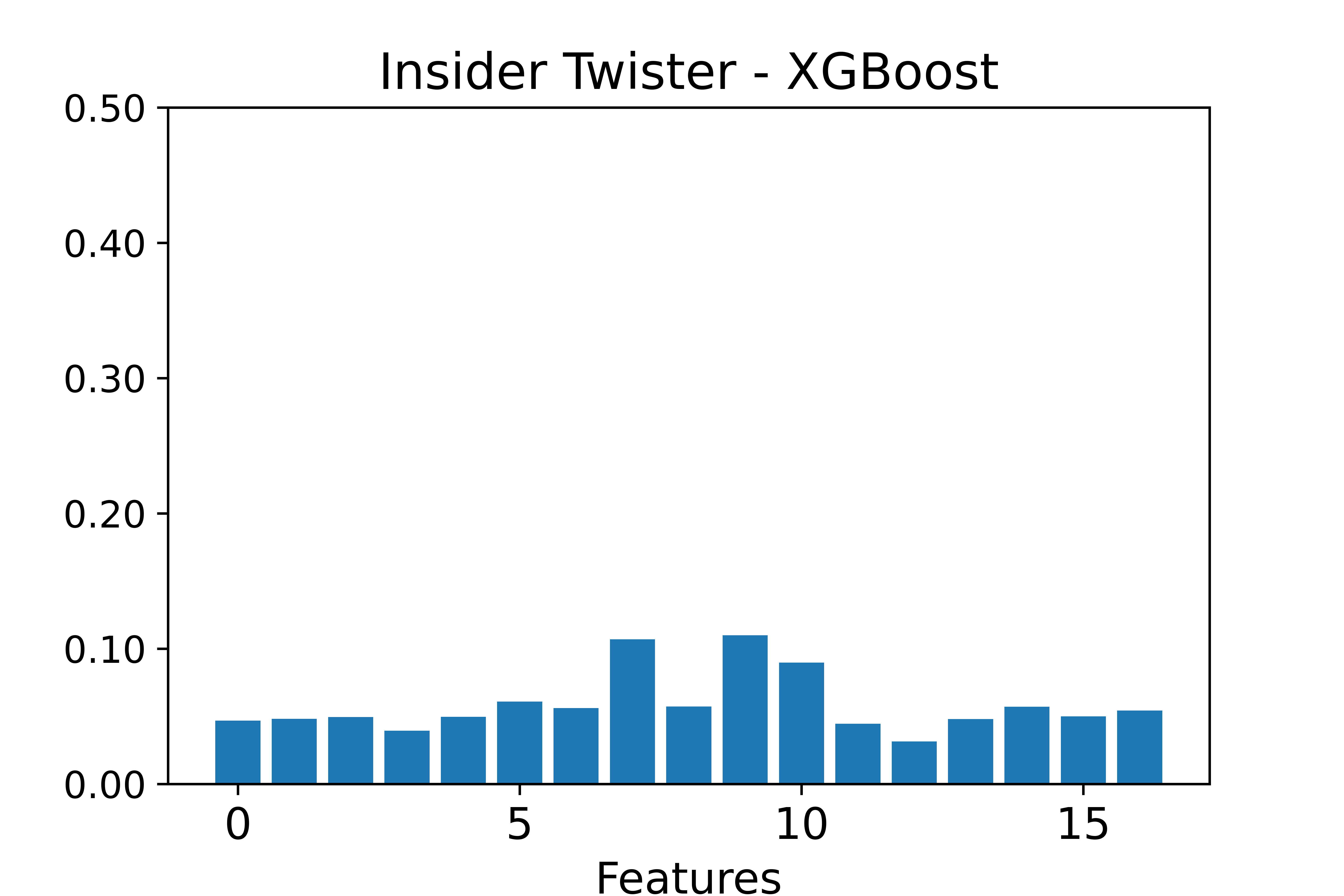}
\end{tabular}
\end{center}
\vspace{-0.2cm}
\caption{With regularization (where $\lambda = 20$), we still observe that the feature importance of XGBoost is perturbed. Note that \pmboost~is not regularized.}
  \label{fig:insidertwisterregularizedfeatures}
\vspace{-0.2cm}
\end{figure}

\begin{figure}[t]
\begin{center}
\begin{tabular}{ccc}
\hspace{-0.3cm} \includegraphics[trim=0bp 0bp 0bp 0bp,clip,width=0.48\linewidth]{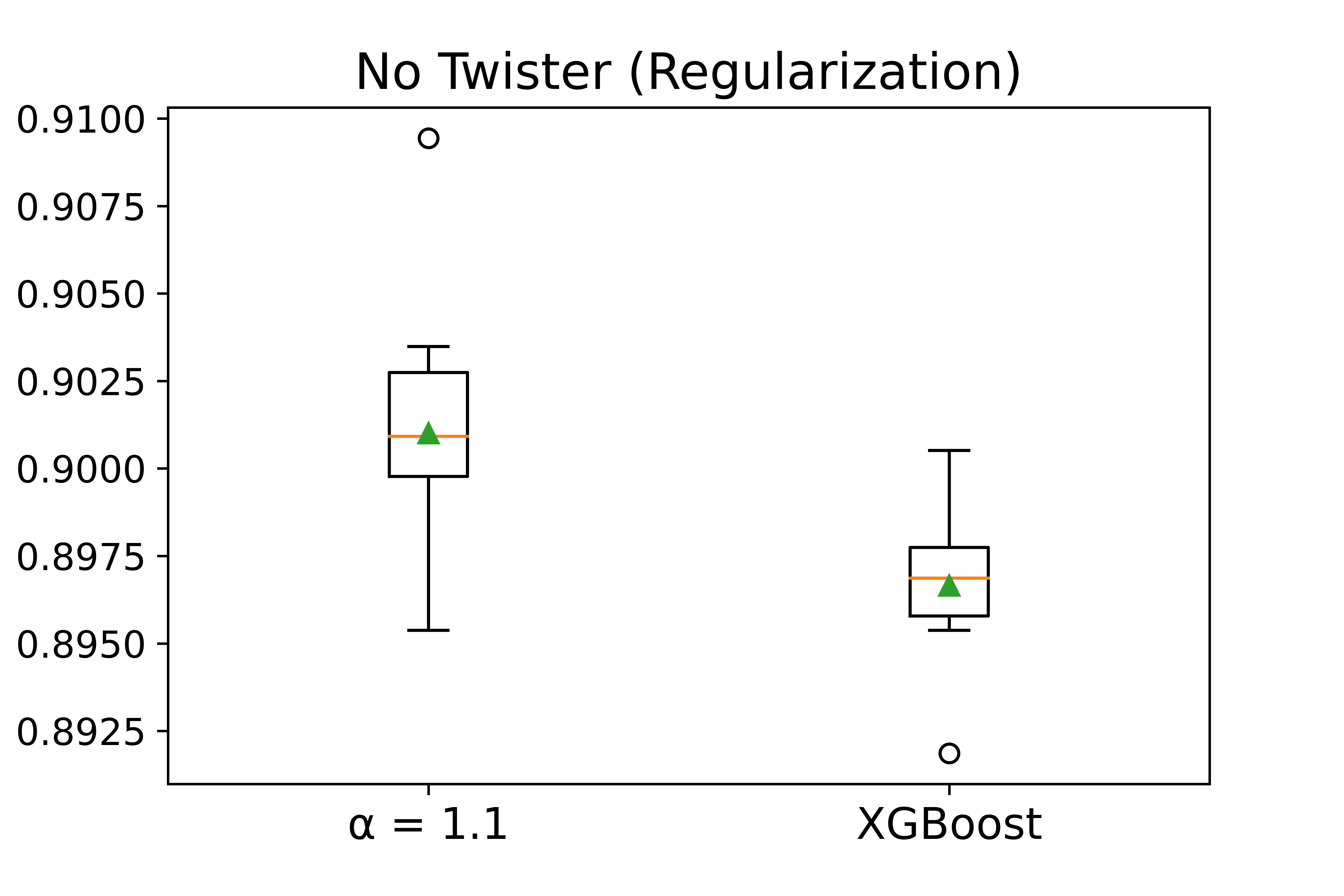}
  \hspace{-0.3cm} & \hspace{-0.2cm} \includegraphics[trim=0bp 0bp 0bp 0bp,clip,width=0.48\linewidth]{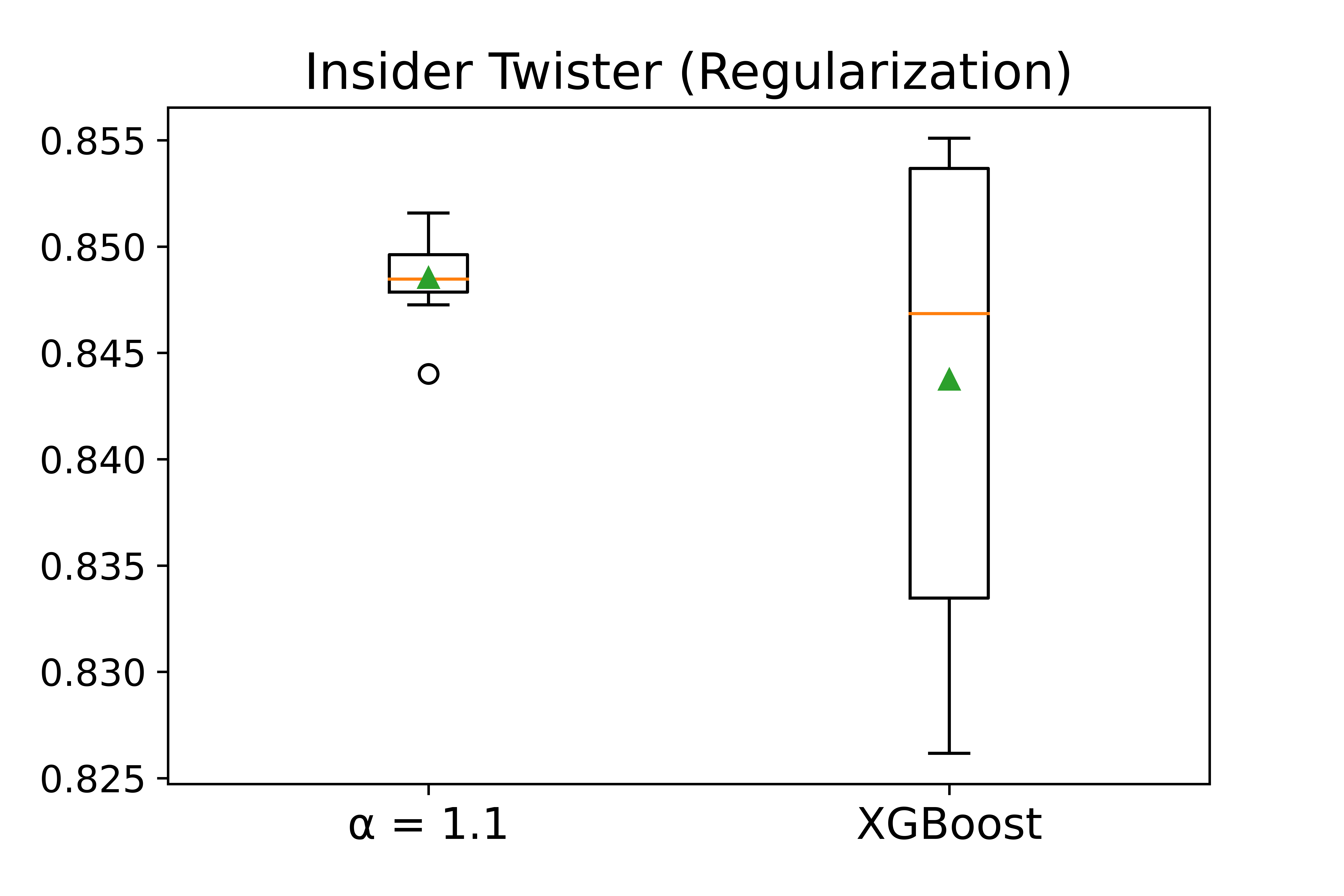}
\end{tabular}
\end{center}
\vspace{-0.2cm}
\caption{Scores associated with Figure~\ref{fig:insidertwisterregularizedfeatures}.}
  \label{fig:averagedearlystoppinginsidertwister}
\vspace{-0.2cm}
\end{figure}


\begin{figure}[h]
\begin{center}
\begin{tabular}{ccc}
\hspace{-0.3cm} \includegraphics[trim=0bp 0bp 30bp 15bp,clip,width=0.48\linewidth]{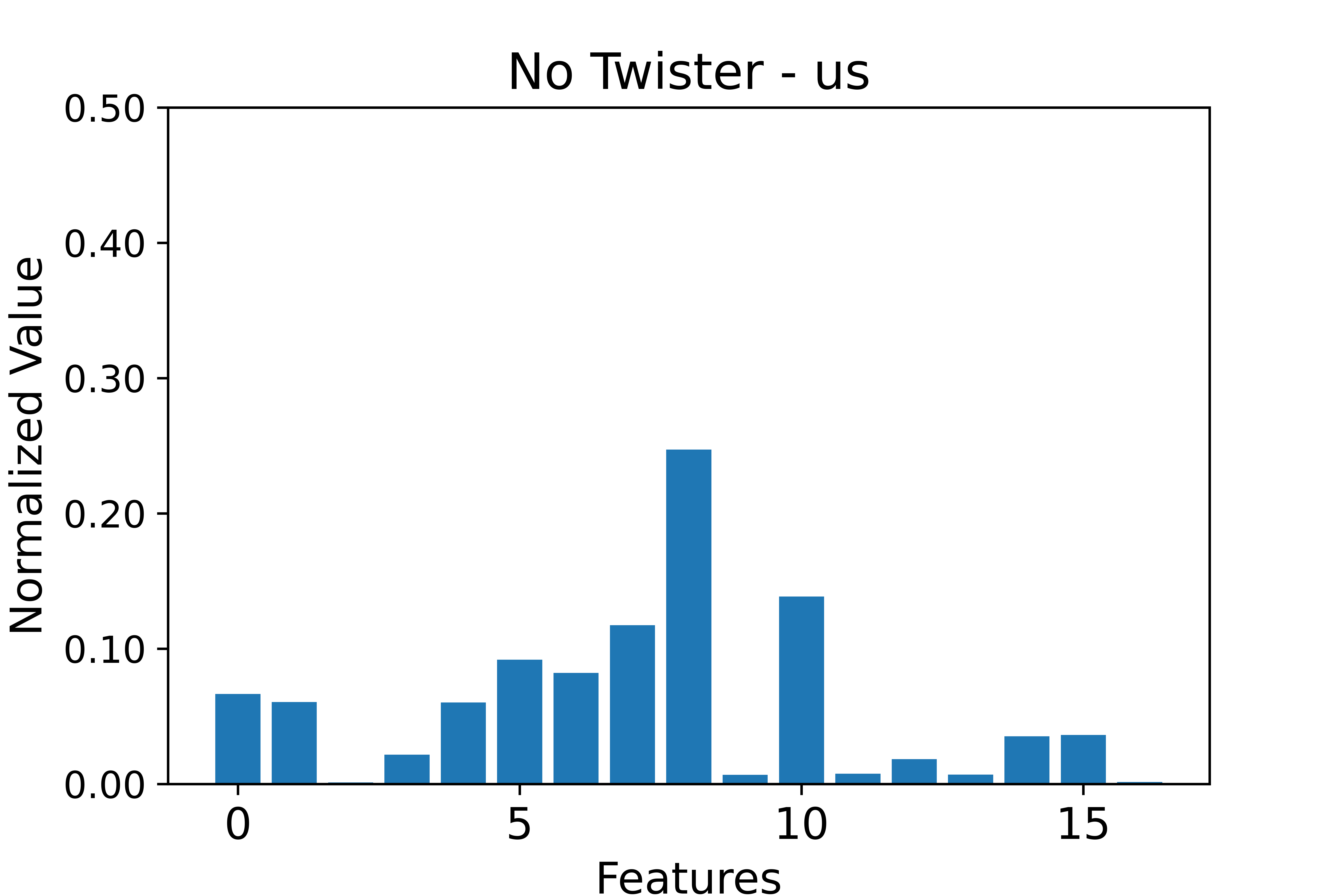}
  \hspace{-0.3cm} & \hspace{-0.2cm} \includegraphics[trim=0bp 0bp 30bp 15bp,clip,width=0.48\linewidth]{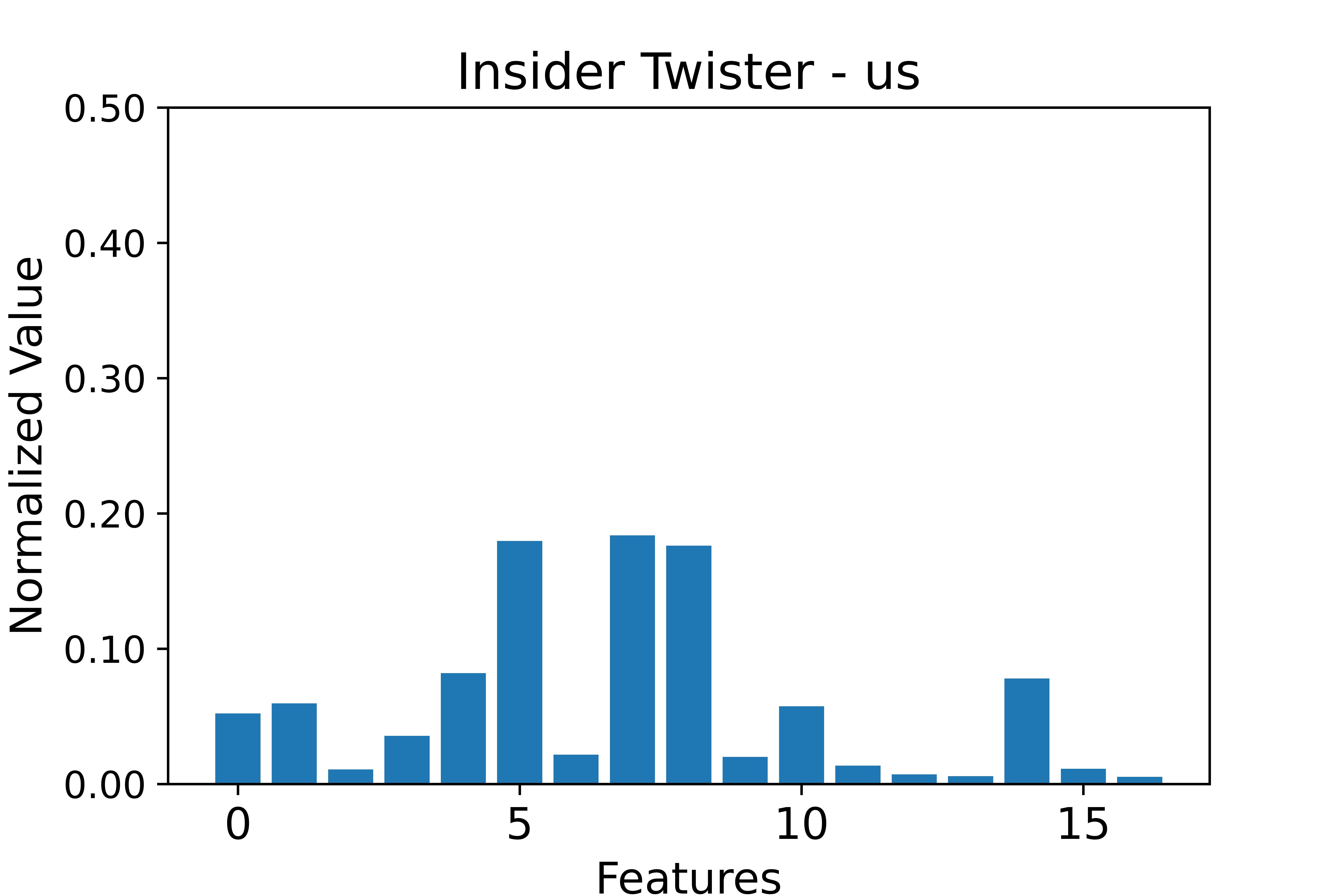}
  \\
\hline
  \hline
  \\
 \hspace{-0.3cm} \includegraphics[trim=0bp 0bp 30bp 15bp,clip,width=0.48\linewidth]{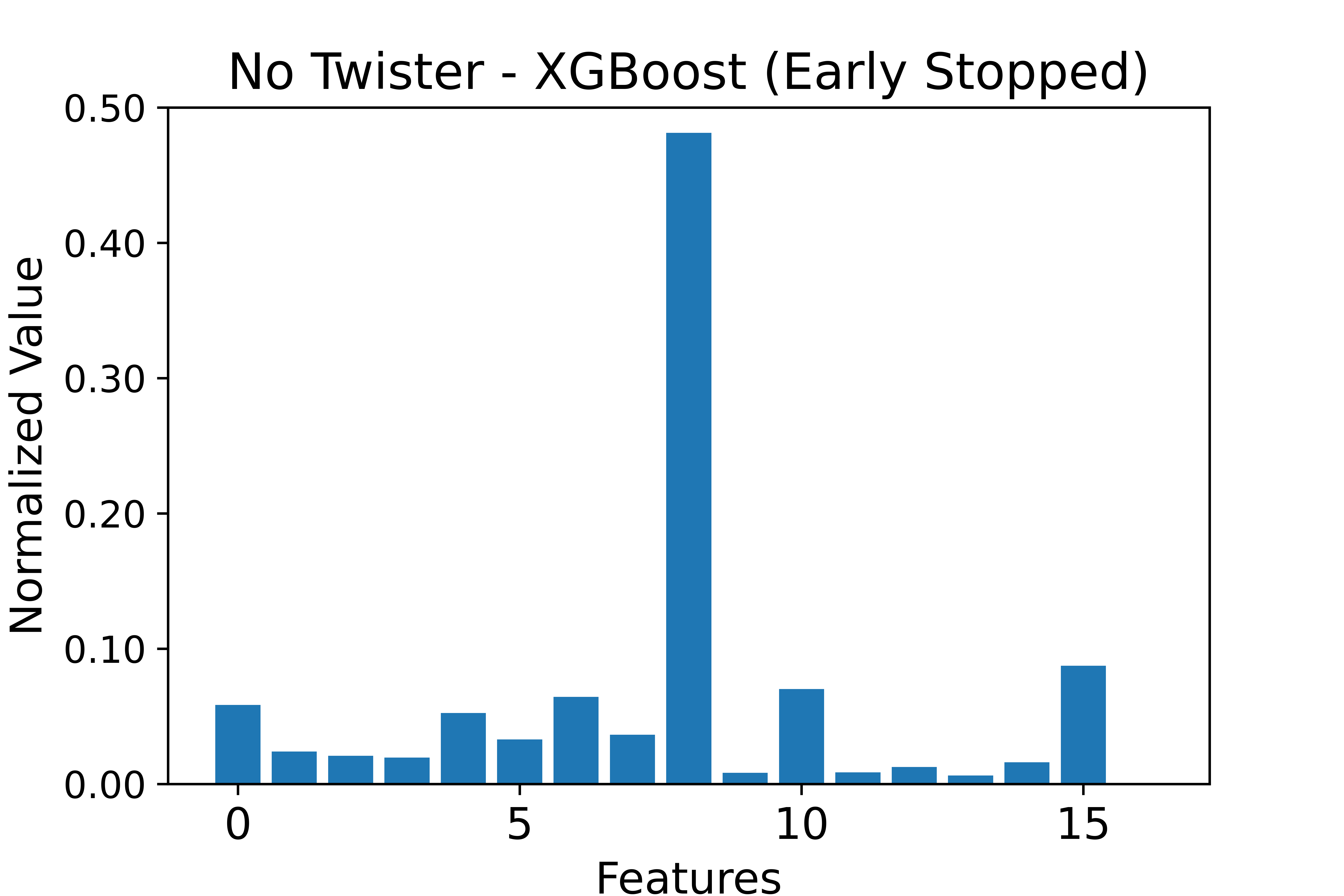}
  \hspace{-0.3cm} & \hspace{-0.2cm} \includegraphics[trim=0bp 0bp 30bp 15bp,clip,width=0.48\linewidth]{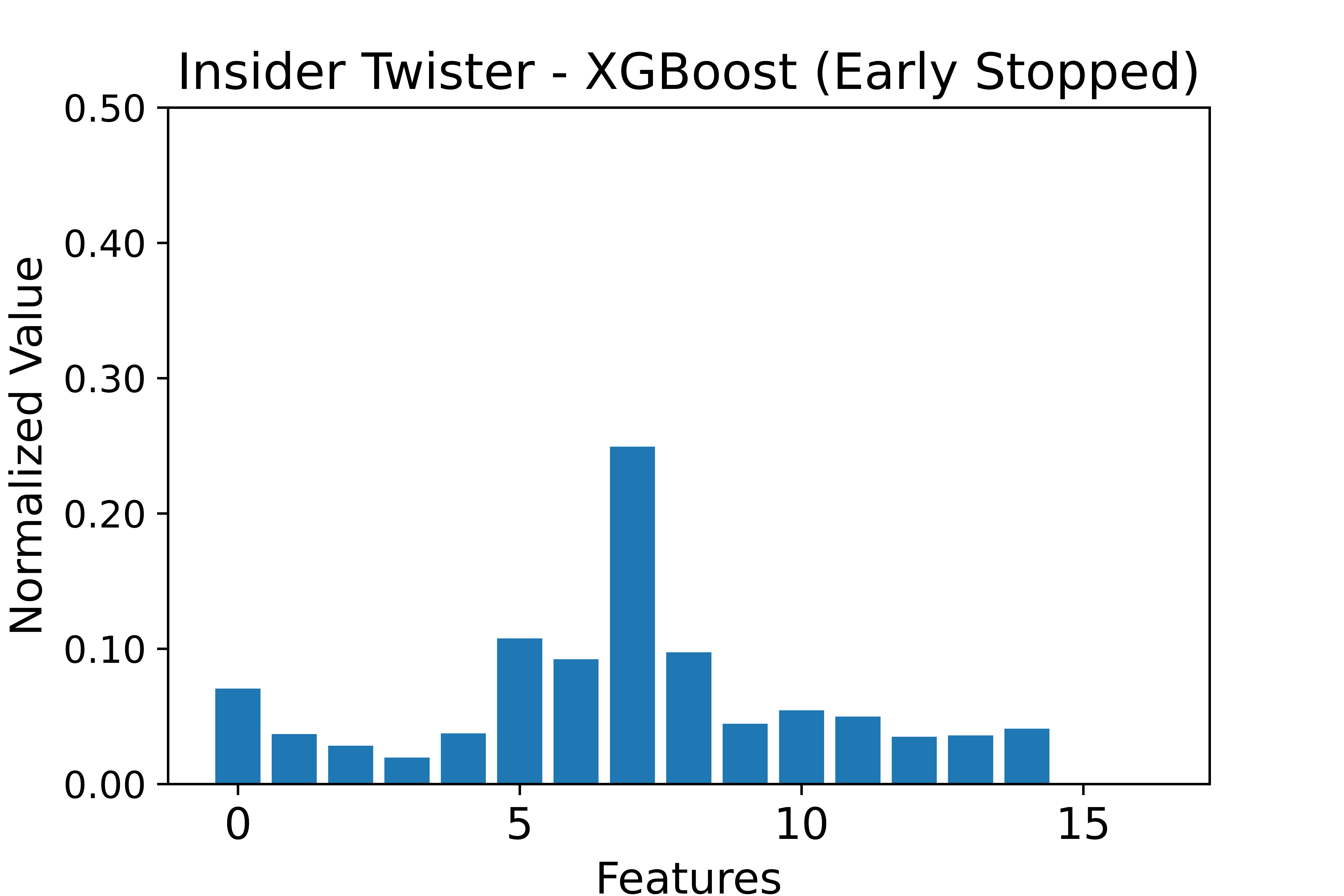}
\end{tabular}
\end{center}
\vspace{-0.2cm}
\caption{With early stopping (where XGBoost has access to clean validation data - cheating scenario), we still observe that the feature importance of XGBoost is perturbed. Note that \pmboost~is not early stopped.}
  \label{fig:insidertwisterearlystoppingfeatures}
\vspace{-0.2cm}
\end{figure}


\begin{figure}[t]
\begin{center}
\begin{tabular}{ccc}
\hspace{-0.3cm} \includegraphics[trim=0bp 0bp 0bp 0bp,clip,width=0.48\linewidth]{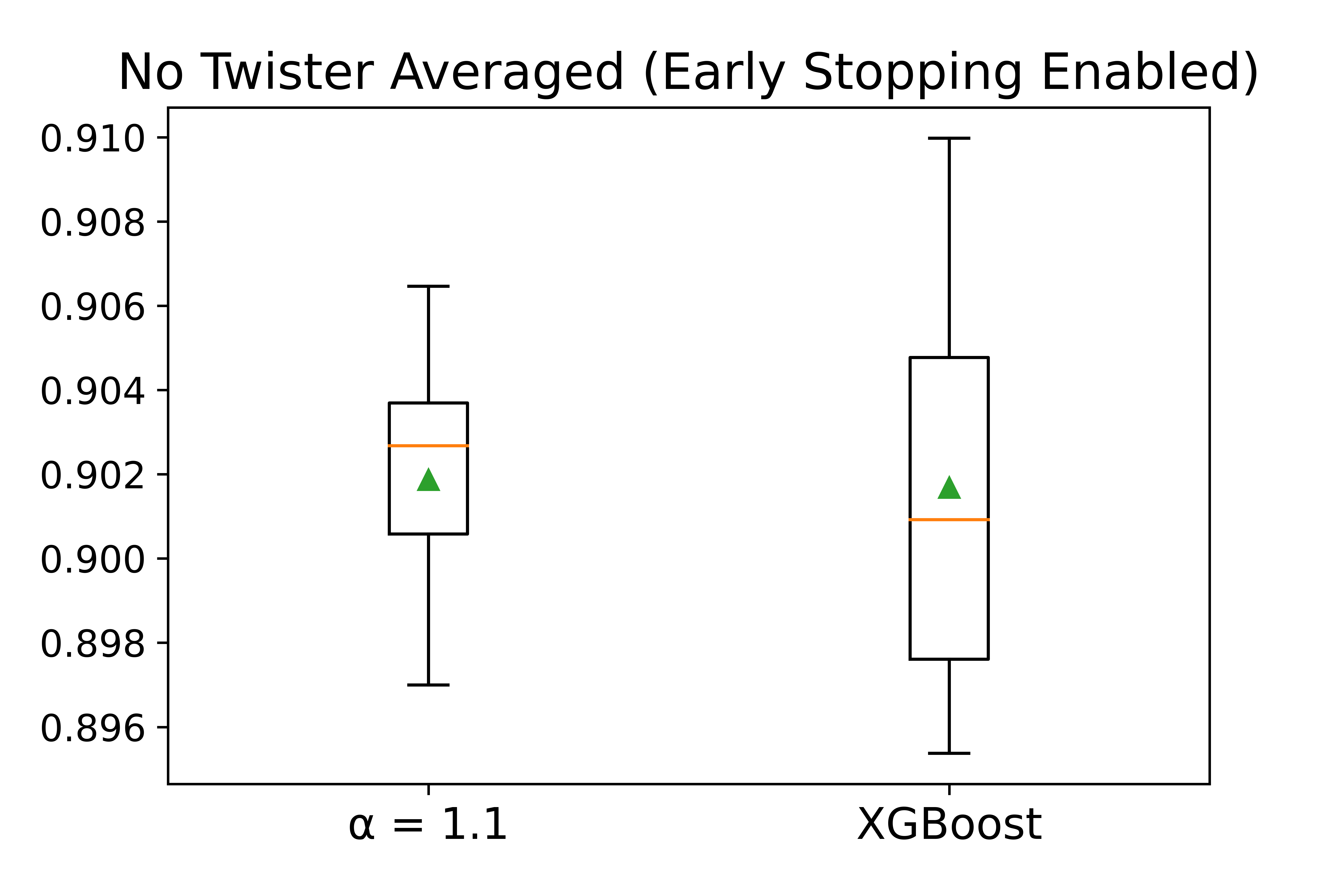}
  \hspace{-0.3cm} & \hspace{-0.2cm} \includegraphics[trim=0bp 0bp 0bp 0bp,clip,width=0.48\linewidth]{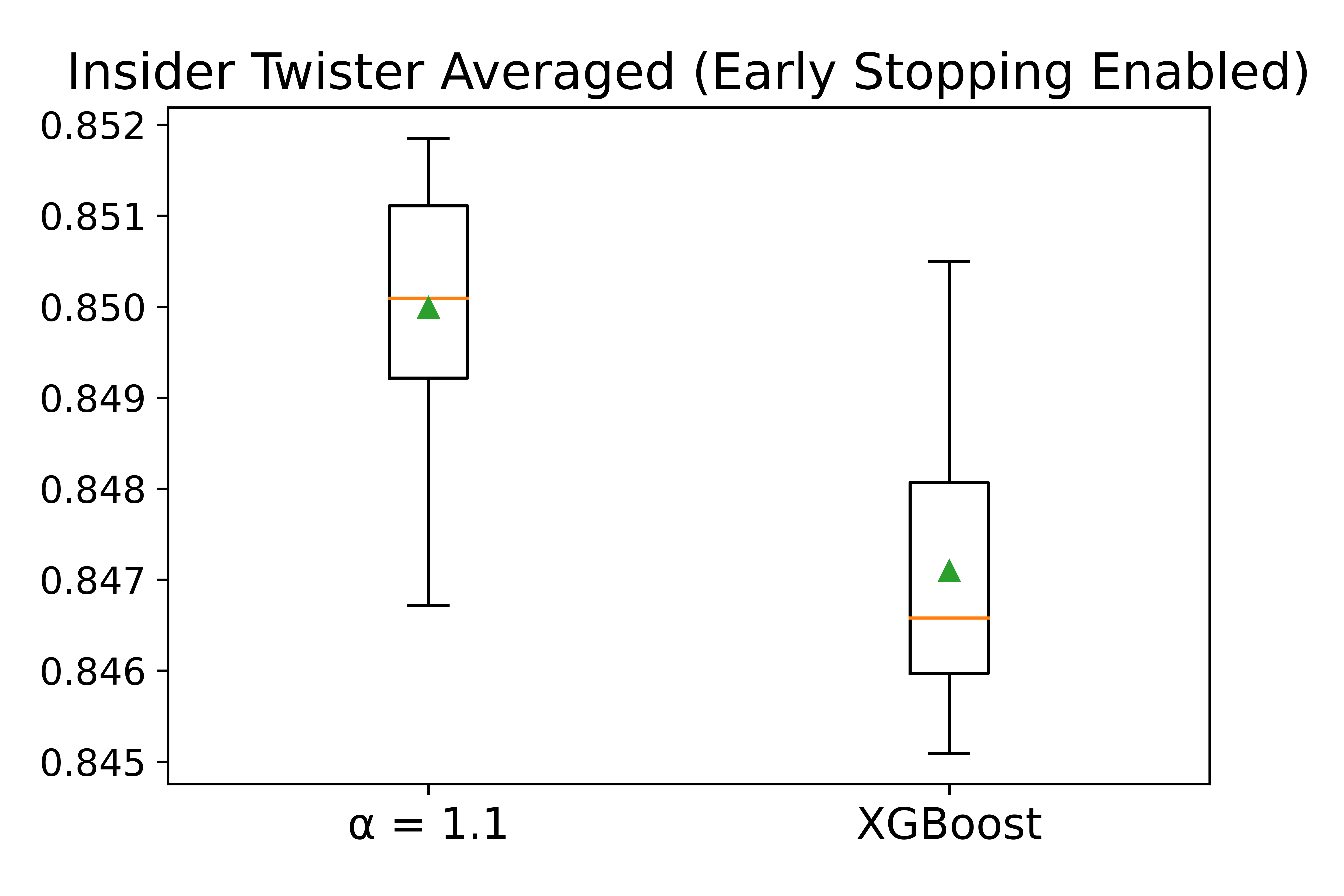}
\end{tabular}
\end{center}
\vspace{-0.2cm}
\caption{Scores associated with Figure~\ref{fig:insidertwisterearlystoppingfeatures}.}
  \label{fig:averagedearlystoppinginsidertwister2}
\vspace{-0.2cm}
\end{figure}

\subsection{Adaptive $\alpha$ Experiment} \label{sec:adaptivealphaexperiments}
\begin{figure}[t]
\bignegspace
  \begin{center}
\includegraphics[scale=.75]{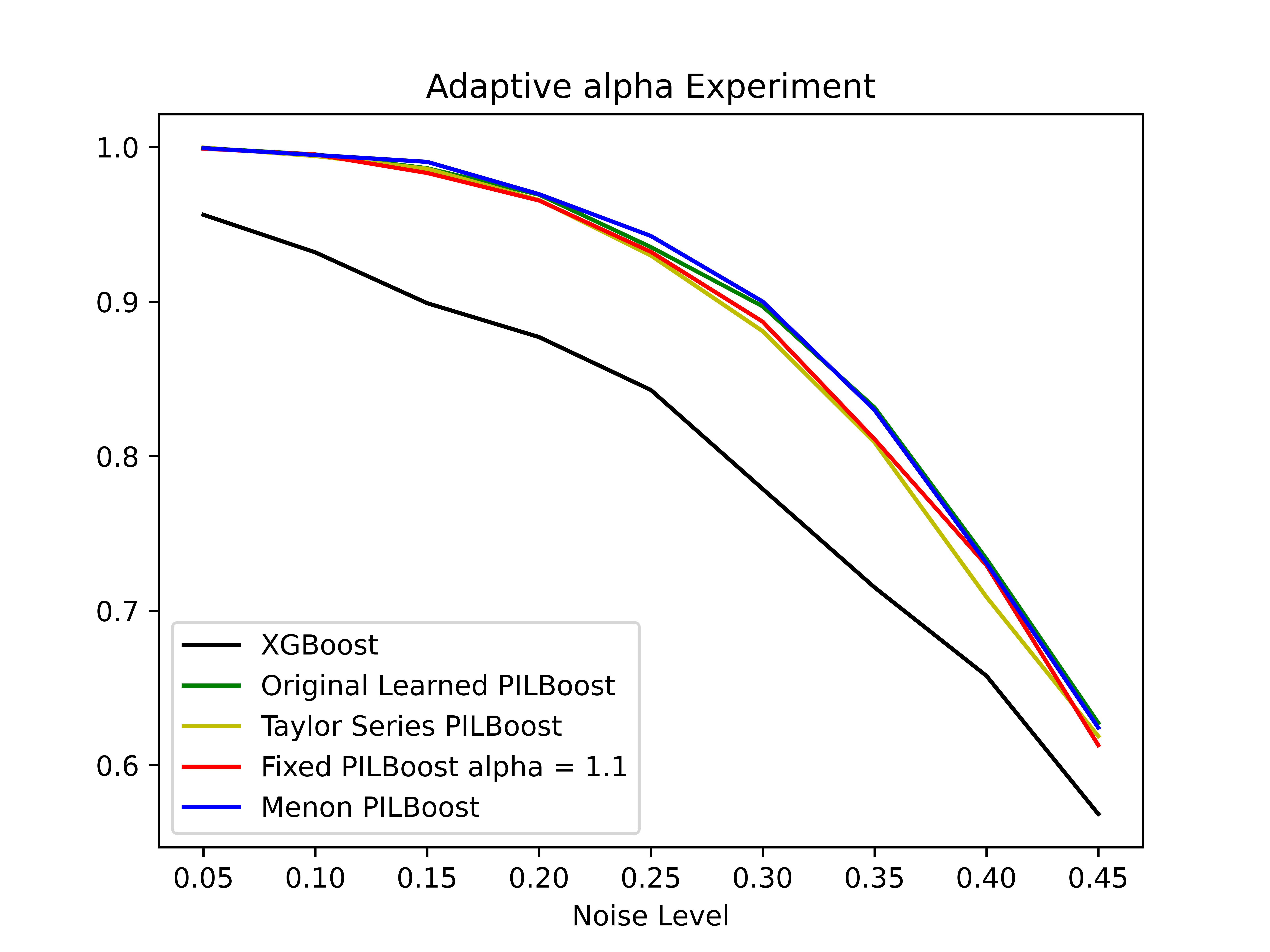}
\end{center}
\caption{Extended version of Figure~\ref{fig:ConfusionMatrix1} with two additional adaptive $\alpha$ methods. Original Learned~\pmboost~estimates its choice of $\alpha$ by using XGBoost as an oracle. 
That is, the method trains XGBoost on the noisy data, then computes its confusion matrix on a \textit{clean} validation set. 
From the confusion matrix, the label flip probability $p$ is estimated using $p = \text{avg}\left(\frac{\text{FP}}{\text{TP} + \text{FP}}, \frac{\text{FN}}{\text{FN} + \text{TN}} \right)$. 
Next, we estimate $\posclean$ and $\posnoise$ with $\posclean = \frac{\text{FN}+\text{TP}}{\text{FP}+\text{TP} + \text{FN} + \text{TN}}$ and $\posnoise = \frac{\text{FP}+\text{TP}}{\text{FP}+\text{TP} + \text{FN} + \text{TN}}$, respectively.
Lastly similar to Menon~\pmboost, using the estimates of $p$, $\posclean$, and $\posnoise$, we apply the formula in Lemma~\ref{lemPointwise}\textbf{(c)} and the SLN formula given just before Lemma~\ref{lemSymLabFlip} to obtain an estimate for $\alpha^{*}$.
Taylor Series~\pmboost~is identical to Original Learned~\pmboost~except at the last step, where a Taylor series approximation of the formula in Lemma~\ref{lemPointwise}\textbf{(c)} is used instead.
We find that Menon's Method also outperforms both of these methods on the xd6 dataset, except for when Original Learned~\pmboost~slightly outperforms Menon's Method in the very high noise regime.
Even stronger, note that both Original Learned~\pmboost~and Taylor Series~\pmboost~assume more information than ``Menon's Method'', which only uses the noisy training data, not a clean validation set.
}
  \label{fig:ConfusionMatrix2} 
     \vspace{-0.3cm}
\end{figure}



\begin{figure}[t]
\begin{center}
\begin{tabular}{ccc}
\hspace{-0.3cm} \includegraphics[trim=0bp 0bp 0bp 0bp,clip,width=0.48\linewidth]{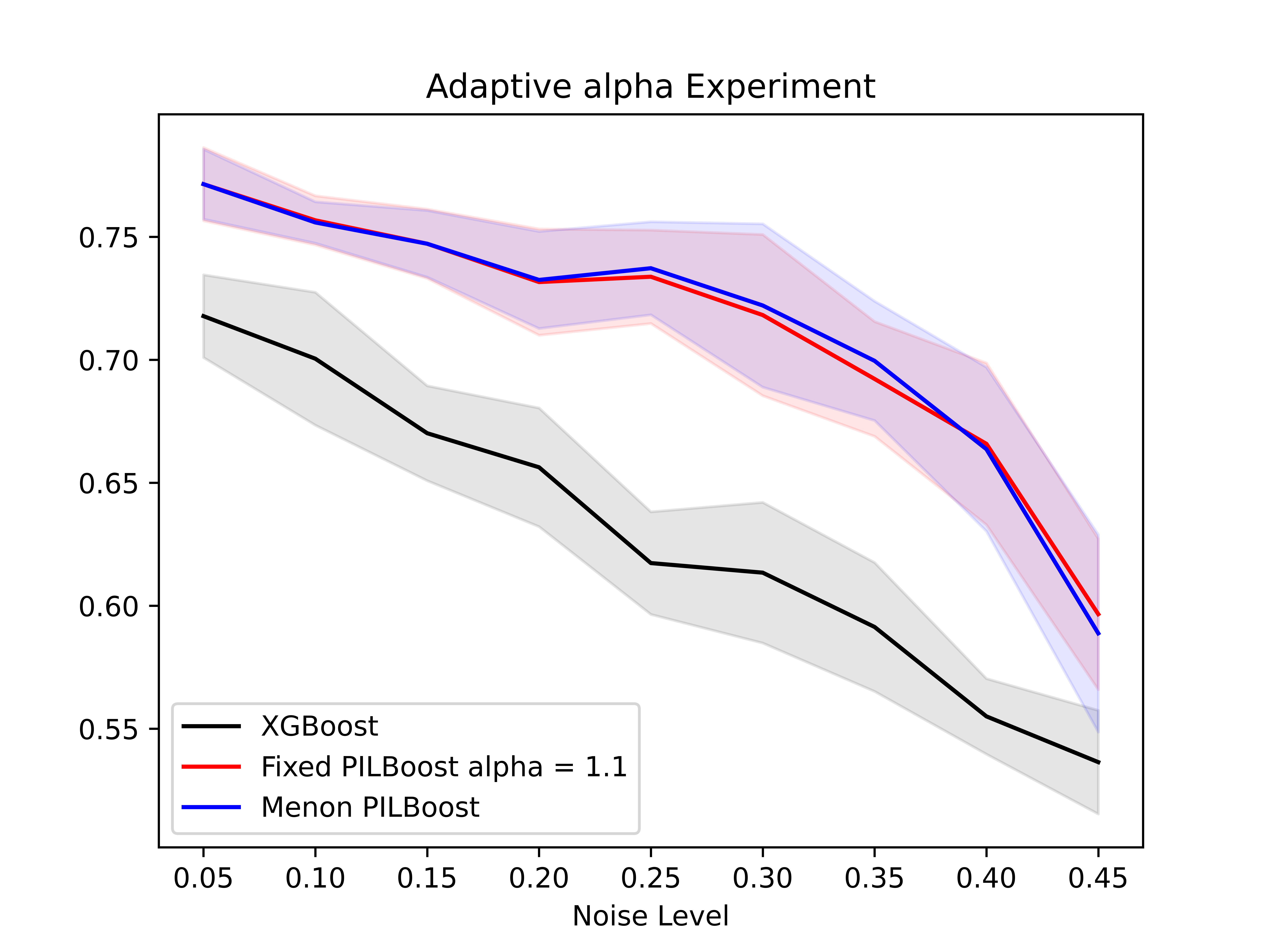}
  \hspace{-0.3cm} & \hspace{-0.2cm} \includegraphics[trim=0bp 0bp 0bp 0bp,clip,width=0.48\linewidth]{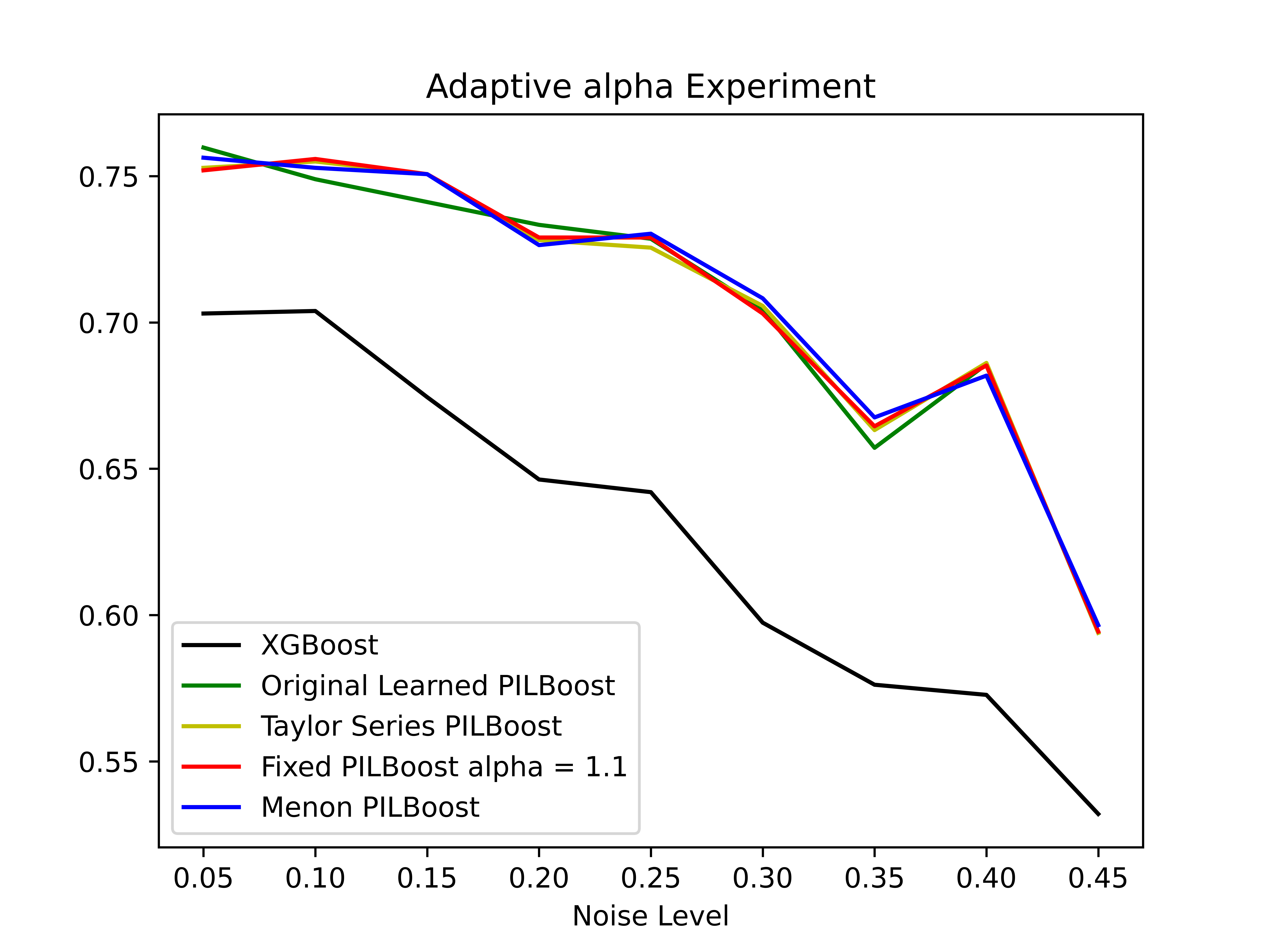}
\end{tabular}
\end{center}
\vspace{-0.2cm}
\caption{Companion Figure to Figures~\ref{fig:ConfusionMatrix1} and~\ref{fig:ConfusionMatrix2} on the \textit{diabetes} dataset.}
  \label{fig:ConfusionMatrix3}
\vspace{-0.2cm}
\end{figure}



\begin{figure}[t]
\begin{center}
\begin{tabular}{ccc}
\hspace{-0.3cm} \includegraphics[trim=0bp 0bp 0bp 0bp,clip,width=0.48\linewidth]{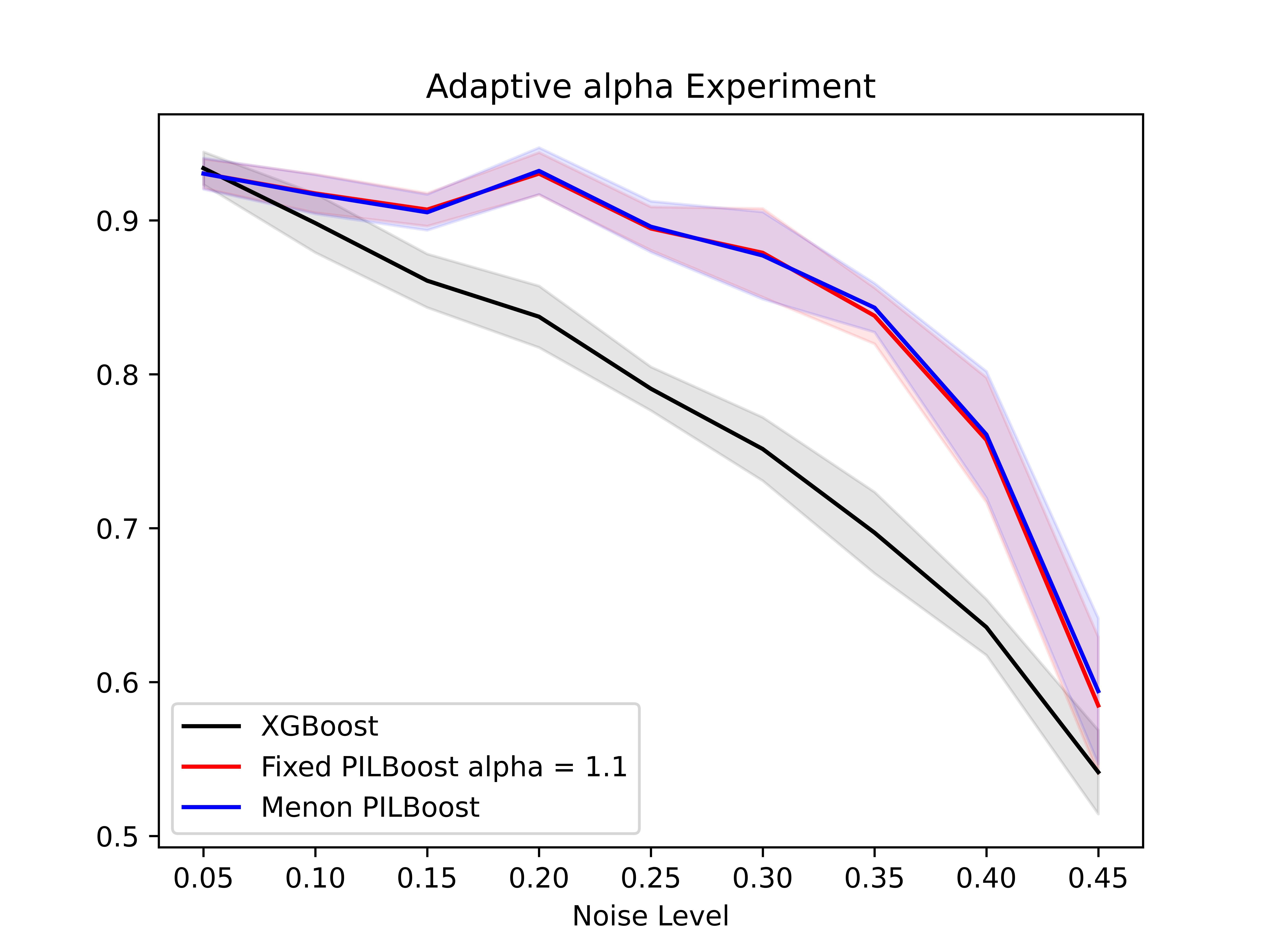}
  \hspace{-0.3cm} & \hspace{-0.2cm} \includegraphics[trim=0bp 0bp 0bp 0bp,clip,width=0.48\linewidth]{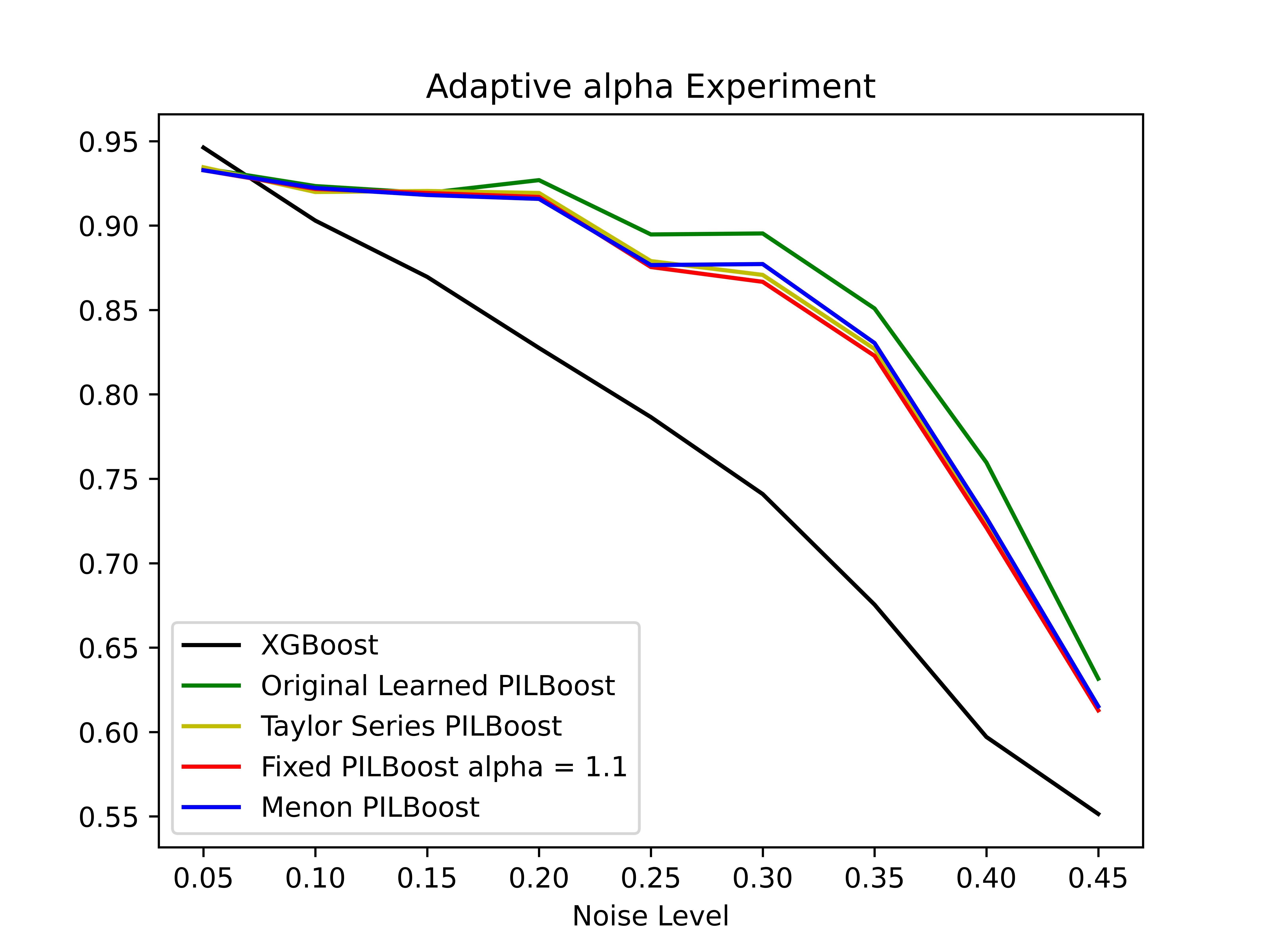}
\end{tabular}
\end{center}
\vspace{-0.2cm}
\caption{Companion Figure to Figures~\ref{fig:ConfusionMatrix1} and~\ref{fig:ConfusionMatrix2} on the \textit{cancer} dataset.}
  \label{fig:ConfusionMatrix4}
\vspace{-0.2cm}
\end{figure}

\end{appendices}


\end{document}